\documentclass{article}

\usepackage{times}
\usepackage{graphicx} 
\usepackage{subfigure}

\usepackage{natbib}

\usepackage{algorithm}
\usepackage{algorithmic}
\usepackage[algo2e,linesnumbered,lined,boxed,vlined]{algorithm2e}

\usepackage{amsfonts}
\usepackage{amssymb}

\usepackage{hyperref}

\usepackage[accepted]{icml2016}

\usepackage{comment}
\usepackage{amsthm}
\usepackage{mathtools}
\usepackage{tabularx}
\usepackage{multirow}

\newtheorem{theorem}{Theorem}
\newtheorem{lemma}{Lemma}

\newtheorem{corollary}{Corollary}
\newtheorem{definition}{Definition}

\newtheorem{remark}{Remark}

\def\b{\ensuremath\boldsymbol}

\usepackage{lastpage}
\usepackage{fancyhdr}
\usepackage{url}
\icmltitlerunning{Generative Adversarial Networks and Adversarial Autoencoders: Tutorial and Survey}

\begin{document}

\AddToShipoutPictureBG*{%
  \AtPageUpperLeft{%
    \setlength\unitlength{1in}%
    \hspace*{\dimexpr0.5\paperwidth\relax}
    \makebox(0,-0.75)[c]{\normalsize {\color{black} To appear as a part of an upcoming textbook on dimensionality reduction and manifold learning.}}
    }}

\twocolumn[
\icmltitle{Generative Adversarial Networks and Adversarial Autoencoders: \\Tutorial and Survey}

\icmlauthor{Benyamin Ghojogh}{bghojogh@uwaterloo.ca}
\icmladdress{Department of Electrical and Computer Engineering, 
\\Machine Learning Laboratory, University of Waterloo, Waterloo, ON, Canada}
\icmlauthor{Ali Ghodsi}{ali.ghodsi@uwaterloo.ca}
\icmladdress{Department of Statistics and Actuarial Science \& David R. Cheriton School of Computer Science, 
\\Data Analytics Laboratory, University of Waterloo, Waterloo, ON, Canada}
\icmlauthor{Fakhri Karray}{karray@uwaterloo.ca}
\icmladdress{Department of Electrical and Computer Engineering, 
\\Centre for Pattern Analysis and Machine Intelligence, University of Waterloo, Waterloo, ON, Canada}
\icmlauthor{Mark Crowley}{mcrowley@uwaterloo.ca}
\icmladdress{Department of Electrical and Computer Engineering, 
\\Machine Learning Laboratory, University of Waterloo, Waterloo, ON, Canada}

\icmlkeywords{Tutorial}

\vskip 0.3in
]

\begin{abstract}
This is a tutorial and survey paper on Generative Adversarial Network (GAN), adversarial autoencoders, and their variants. We start with explaining adversarial learning and the vanilla GAN. Then, we explain the conditional GAN and DCGAN. The mode collapse problem is introduced and various methods, including minibatch GAN, unrolled GAN, BourGAN, mixture GAN, D2GAN, and Wasserstein GAN, are introduced for resolving this problem. Then, maximum likelihood estimation in GAN are explained along with f-GAN, adversarial variational Bayes, and Bayesian GAN. Then, we cover feature matching in GAN, InfoGAN, GRAN, LSGAN, energy-based GAN, CatGAN, MMD GAN, LapGAN, progressive GAN, triple GAN, LAG, GMAN, AdaGAN, CoGAN, inverse GAN, BiGAN, ALI, SAGAN, Few-shot GAN, SinGAN, and interpolation and evaluation of GAN. Then, we introduce some applications of GAN such as image-to-image translation (including PatchGAN, CycleGAN, DeepFaceDrawing, simulated GAN, interactive GAN), text-to-image translation (including StackGAN), and mixing image characteristics (including FineGAN and MixNMatch). Finally, we explain the autoencoders based on adversarial learning including adversarial autoencoder, PixelGAN, and implicit autoencoder. 
\end{abstract}

\section{Introduction}

Suppose we have a generative model which takes a random noise as input and generates a data point. We want the generated data point to be of good quality; hence, we should somehow judge its quality. One way to judge it is to observe the generated sample and assess its quality visually. In this case, the judge is a human. However, we cannot take derivative of human's judgment for optimization. Generative Adversarial Network (GAN), proposed in \cite{goodfellow2014generative}, has the same idea but it can take derivative of the judgment. For that, it uses a classifier as the judge rather than a human. Hence, we have a generator generating a sample and a binary classifier (or discriminator) to classify the generated sample as a real or generated sample. 
This classifier can be a pre-trained network which is already trained by some real and generated (fake) data points. However, GAN puts a step ahead and lets the classifier be trained simultaneously with training the generator. This is the core idea of adversarial learning with the classifier, also called the discriminator, and the generator compete each other; hence, they make each other stronger gradually by this competition \cite{goodfellow2020generative}.  

It is noteworthy that the term ``adversarial" is used in two main streams of research in machine learning and they should not be confused. These two research areas are:
\begin{itemize}
\item Adversarial attack, also called learning with adversarial examples or adversarial machine learning. This line of research inspects some examples which can be changed slightly but wisely to fool a trained learning model. For example, perturbation of some specific pixels in the input image  may change the decision of learning model. The reason for this can be analyzed theoretically. Some example works in this area are \cite{huang2011adversarial,moosavi2016deepfool,kurakin2017adversarial,kurakin2017adversarial2,madry2018towards}.
\item Adversarial learning for generation. This line of research is categorized as generative models \cite{ng2002discriminative} and/or methods based on that. GAN is in this line of research. This paper focuses on this research area.
\end{itemize}

Another good tutorial on GAN is \cite{goodfellow2016nips} but it does not cover most recent methods in adversarial learning. Also, an honorary introduction of GAN, by several main contributors of GAN, is \cite{goodfellow2020generative}. 
Some other existing surveys on GAN are \cite{wang2017generative,creswell2018generative,gonog2019review,hong2019generative,pan2019recent}.
This paper is a tutorial and survey on GAN and its variants. 

\section*{Required Background for the Reader}

This paper assumes that the reader has general knowledge of calculus, probability, linear algebra, and basics of optimization. 

\section{Generative Adversarial Network (GAN)}


\subsection{Adversarial Learning: The Adversarial Game}

The original GAN, also called the vanilla GAN, was proposed in \cite{goodfellow2014generative}.
Consider a $d$-dimensional dataset with $n$ data points, i.e., $\{\b{x}_i \in \mathbb{R}^d\}_{i=1}^n$.
In GAN, we have a generator $G$ which takes a $p$-dimensional random noise $\b{z} \in \mathbb{R}^p$ as input and outputs a $d$-dimensional generated point $\b{x} \in \mathbb{R}^d$. Hence, it is the mapping $G: \b{z} \rightarrow \b{x}$ where:
\begin{align}\label{equation_GAN_G}
G(\b{z}) = \b{x}.
\end{align}
The random noise can be seen as a latent factor on which the generated data point is conditioned. 
The probabilistic graphical model of generator is a variable $\b{x}$ conditioned on a latent variable $\b{z}$ (see {\citep[Fig. 13]{goodfellow2016nips}} for its visualization). 

Let the distribution of random noise be denoted by $\b{z} \sim p_z(\b{z})$.
We want the generated $\widehat{\b{x}}$ to be very similar to some original (or real) data point $\b{x}$ in the dataset. We need a module to judge the quality of the generated point to see how similar it is to the real point. 
This module can be a human but we cannot take derivative of human's judgment for optimization!
A good candidate for the judge is a classifier, also called the discriminator. The discriminator (also called the critic), denoted by $D: \b{x} \rightarrow [0,1]$, is a binary classifier which classifies the generated point as a real or generated point: 
\begin{align}\label{equation_GAN_D}
D(\b{x}) := 
\left\{
    \begin{array}{ll}
        1 & \mbox{if } \b{x} \text{ is real}, \\
        0 & \mbox{if } \b{x} \text{ is generated (fake)}.
    \end{array}
\right.
\end{align}
The perfect discriminator outputs one for real points and zero for generated points. 
The discriminator's output is in the range $[0,1]$ where the output for real data is closer to one and the output for fake data is closer to zero. 
If the generated point is very good and closely similar to a real data point, the classifier may make a mistake and outputs a value close to one for it. Therefore, if the classifier makes a mistake for the generated point, the generator has done a good job in generating a data point. 

The discriminator can be pre-trained but we can make the problem more sophisticated. Let us train the discriminator simultaneously while we are training the generator. This makes the discriminator $D$ and the generator $G$ stronger gradually while they compete each other. On one hand, the generator tries to generate realistic points to fool the discriminator and make it a hard time to distinguish the generated point from a real point. On the other hand, the discriminator tries to discriminate the fake (i.e., generated) point from a real point. When one of them gets stronger in training, the other one tries to become stronger to be able to compete. Therefore, there is an adversarial game between the generator and the discriminator. This game is zero-sum because whatever one of them loses, the other wins. 

\subsection{Optimization and Loss Function}\label{section_GAN_optimization_and_loss}

\begin{figure}[!t]
\centering
\includegraphics[width=3.2in]{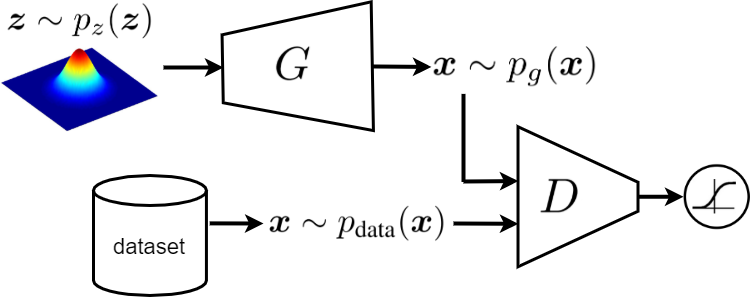}
\caption{The structure of GAN.}
\label{figure_GAN}
\end{figure}

We denote the probability distributions of dataset and noise by $p_\text{data}(\b{x})$ and $p_z(\b{z})$, respectively. 
The structure of GAN is depicted in Fig. \ref{figure_GAN}. As the figure shows, the discriminator is trained by real points from dataset as well as generated points from the generator. 
The discriminator and generator are trained simultaneously.
The optimization loss function for both the discriminator and generator is:
\begin{equation}\label{equation_GAN_loss}
\begin{aligned}
\min_G \max_D\,\,\,\, &V(D,G) := \mathbb{E}_{\b{x} \sim p_\text{data}(\b{x})}\Big[\log\!\big(D(\b{x})\big)\Big] \\
&+ \mathbb{E}_{\b{z} \sim p_z(\b{z})}\Big[\log\!\Big(1 - D\big(G(\b{z})\big)\Big)\Big],
\end{aligned}
\end{equation}
where $\mathbb{E}[.]$ denotes the expectation operator and the loss function $V(D,G)$ is also called the value function of the game. 
In practice, we can use the Monte Carlo approximation \cite{ghojogh2020sampling} of expectation where the expectations are replaced with averages over the mini-batch. 
This loss function is in the form of a cross-entropy loss. 

The first term in Eq. (\ref{equation_GAN_loss}) is expectation over the real data. This term is only used for the discriminator while it is a constant for the generator. According to Eq. (\ref{equation_GAN_D}), $D(\b{x})$ outputs one (the larger label) for the real data; therefore, the discriminator maximizes this term because it assigns the larger label to the real data. 

The second term in Eq. (\ref{equation_GAN_loss}) is expectation over noise. It inputs the noise $\b{z}$ to the generator to have $G(\b{z})$. The output of generator, which is the generated point, is fed as input to the discriminator (see Fig. \ref{figure_GAN}) to have $D\big(G(\b{z})\big)$. The discriminator wants to minimize $D\big(G(\b{z})\big)$ because the smaller label is assigned to the generated data, according to Eq. (\ref{equation_GAN_D}). In other words, the discriminator wants to maximize $1- D\big(G(\b{z})\big)$. As logarithm is a monotonic function, we can say that the discriminator wants to maximize $\mathbb{E}_{\b{z} \sim p_z(\b{z})}[\log(1 - D(G(\b{z})))]$ which is the second term in Eq. (\ref{equation_GAN_loss}). As opposed to the discriminator, the generator minimizes $\mathbb{E}_{\b{z} \sim p_z(\b{z})}[\log(1 - D(G(\b{z})))]$ which is the second term in Eq. (\ref{equation_GAN_loss}). This is because the generator wants to fool the discriminator to label the generated data as real data. 

The Eq. (\ref{equation_GAN_loss}) is a minimax optimization problem \cite{du2013minimax} and can be solved using alternating optimization \cite{ghojogh2021kkt} where we optimize over $D$ and over $G$ iteratively until convergence (i.e., Nash equilibrium). The original GAN \cite{goodfellow2014generative} uses a step of stochastic gradient descent \cite{ghojogh2021kkt} for updates of each variable in the alternating optimization. If we denote the loss function in Eq. (\ref{equation_GAN_loss}) by $V(D,G)$, the alternating optimization is done as:
\begin{align}
& D^{(k+1)} := D^{(k)} + \eta^{(k)} \frac{\partial }{\partial D} \Big(V(D,G^{(k)})\Big), \label{equation_GAN_alternating_opt_D} \\
& G^{(k+1)} := G^{(k)} - \eta^{(k)} \frac{\partial }{\partial G} \Big(V(D^{(k+1)},G)\Big), \label{equation_GAN_alternating_opt_G}
\end{align}
where $k$ is the index of iteration and $\eta^{(k)}$ is the learning rate at iteration $k$. 
Throughout this paper, derivatives w.r.t. $D$ and $G$ mean the derivatives w.r.t. the parameters (weights) of $D$ and $G$ networks, respectively. 
Eqs. (\ref{equation_GAN_alternating_opt_D}) and (\ref{equation_GAN_alternating_opt_G}) are one step of gradient ascent and gradient descent, respectively. 
Note that the gradients here are the average of gradients in the mini-batch. Every mini-batch includes both real and generated data. 
The paper \cite{goodfellow2014generative} suggests that Eq. (\ref{equation_GAN_alternating_opt_D}) can be performed for several times before performing Eq. (\ref{equation_GAN_alternating_opt_G}); however, the experiments of that paper perform Eq. (\ref{equation_GAN_alternating_opt_D}) for only one time before performing Eq. (\ref{equation_GAN_alternating_opt_G}). 
Also note that another way to solve the optimization problem in GAN is simultaneous optimization \cite{mescheder2017numerics} in which Eqs. (\ref{equation_GAN_alternating_opt_D}) and (\ref{equation_GAN_alternating_opt_G}) are performed at the same time and not one after the other.

\begin{remark}[Minimax versus maximin in GAN {\citep[Section 5]{goodfellow2016nips}}]
We saw in Eq. (\ref{equation_GAN_loss}) that the optimization of GAN is a minimax problem:
\begin{align}\label{equation_GAN_minimax}
\min_G \max_D\,\,\, V(D,G). 
\end{align}
By changing the order of optimization, one can see GAN as a maximin problem \cite{goodfellow2016nips}:
\begin{align}\label{equation_GAN_maximin}
\max_D \min_G\,\,\, V(D,G). 
\end{align}
In fact, under some conditions, Eqs. (\ref{equation_GAN_minimax}) and (\ref{equation_GAN_maximin}) are equivalent \cite{du2013minimax}. 
\end{remark}

\subsection{Network Structure of GAN}\label{section_GAN_structure}

In practice, the discriminator and generator are two (deep) neural networks. 
The structure of GAN is depicted in Fig. \ref{figure_GAN}.
The first layer of discriminator network is $d$-dimensional and its last layer is one dimensional with scalar output. 
In the original GAN, maxout activation function \cite{goodfellow2013maxout} is used for all layers except the last layer which has the sigmoid activation function to output a probability to model Eq. (\ref{equation_GAN_D}). The closer the output of $D$ to one, the more probable its input is to be real.

The generator network has a $p$-dimensional input layer for noise and a $d$-dimensional output layer for generating data. 
In the generator, a combination of ReLU \cite{nair2010rectified} and sigmoid activation functions are used. 
The space of noise as the input to the generator is called the latent space or the latent factor. 
Each of the Eqs. (\ref{equation_GAN_alternating_opt_D}) and (\ref{equation_GAN_alternating_opt_G}) are performed using backpropagation in the neural networks. 

\subsection{Optimal Solution of GAN}

\begin{theorem}[{\citep[Proposition 1]{goodfellow2014generative}}]\label{theorem_GAN_optimal_discriminator}
For a fixed generator $G$, the optimal discriminator is:
\begin{align}\label{equation_GAN_D_optimum}
D^*(\b{x}) = \frac{p_\text{data}(\b{x})}{p_\text{data}(\b{x}) + p_g(\b{x})},
\end{align}
where $p_\text{data}(\b{x})$ is the probability distribution of the real dataset evaluated at point $\b{x}$ and $p_g(\b{x})$ is the probability distribution of output of generator evaluated at point $\b{x}$.
\end{theorem}
\begin{proof}
According to the definition of expectation, the loss function in Eq. (\ref{equation_GAN_loss}) can be stated as:
\begin{align*} 
V(D,G) = &\int_{\b{x}} p_\text{data}(\b{x}) \log(D(\b{x})) d\b{x} \\
&+ \int_{\b{z}} p_z(\b{z}) \log(1 - D(G(\b{z}))) d\b{z}.
\end{align*}
According to Eq. (\ref{equation_GAN_G}), we have:
\begin{align*}
& G(\b{z}) = \b{x} \implies \b{z} = G^{-1}(\b{x}) \implies d\b{z} = (G^{-1})'(\b{x}) d\b{x},
\end{align*}
where $(G^{-1})'(\b{x})$ is the derivative of $(G^{-1})(\b{x})$ with respect to (w.r.t.) $\b{x}$. Hence:
\begin{align*} 
&V(D,G) = \int_{\b{x}} p_\text{data}(\b{x}) \log(D(\b{x})) d\b{x} \\
&+ \int_{\b{z}} p_z(G^{-1}(\b{x})) \log(1 - D(\b{x})) (G^{-1})'(\b{x}) d\b{x}.
\end{align*}
The relation of distributions of input and output of generator is:
\begin{align}\label{equation_pg_pz_relation}
p_g(\b{x}) = p_z(\b{z}) \times G^{-1}(\b{x}) = p_z(G^{-1}(\b{x}))\, G^{-1}(\b{x}),
\end{align}
where $G^{-1}(\b{x})$ is the Jacobian of distribution at point $\b{x}$. Hence:
\begin{align} 
&V(D,G) = \int_{\b{x}} p_\text{data}(\b{x}) \log(D(\b{x})) d\b{x} \nonumber \\
&+ \int_{\b{z}} p_g(\b{x}) \log(1 - D(\b{x})) d\b{x} \nonumber \\
&= \int_{\b{x}} \Big( p_\text{data}(\b{x}) \log(D(\b{x})) + p_g(\b{x}) \log(1 - D(\b{x})) \Big) d\b{x}. \label{equation_GAN_loss_inTermsOf_x}
\end{align}
For optimization in Eq. (\ref{equation_GAN_loss}), taking derivative w.r.t. $D(\b{x})$ gives:
\begin{align*}
&\frac{\partial V(D,G)}{\partial D(\b{x})} \\
&\overset{(a)}{=} \frac{\partial }{\partial D(\b{x})} \Big( p_\text{data}(\b{x}) \log(D(\b{x})) + p_g(\b{x}) \log(1 - D(\b{x})) \Big) \\
&= \frac{p_\text{data}(\b{x})}{D(\b{x})} - \frac{p_g(\b{x})}{1 - D(\b{x})} \\
&= \frac{p_\text{data}(\b{x}) (1 - D(\b{x})) - p_g(\b{x}) D(\b{x})}{D(\b{x}) (1 - D(\b{x}))} \overset{\text{set}}{=} 0 \\
&\implies p_\text{data}(\b{x}) - p_\text{data}(\b{x}) D(\b{x}) - p_g(\b{x}) D(\b{x}) = 0 \\
&\implies D(\b{x}) = \frac{p_\text{data}(\b{x})}{p_\text{data}(\b{x}) + p_g(\b{x})},
\end{align*}
where $(a)$ is because taking derivative w.r.t. $D(\b{x})$ considers a specific $\b{x}$ and hence it removes the integral (summation). Q.E.D.
\end{proof}

\begin{theorem}[{\citep[Theorem 1]{goodfellow2014generative}}]
The optimal solution of GAN is when the distribution of generated data becomes equal to the distribution of data:
\begin{align}\label{equation_GAN_pg_equals_pdata}
p_{g^*}(\b{x}) = p_\text{data}(\b{x}).
\end{align}
\end{theorem}
\begin{proof}
Putting the optimum $D^*(\b{x})$, i.e. Eq. (\ref{equation_GAN_D_optimum}), in Eq. (\ref{equation_GAN_loss_inTermsOf_x}) gives:
\begin{align*}
&V(D^*, G) \\
&= \int_{\b{x}} \Big( p_\text{data}(\b{x}) \log(D^*(\b{x})) + p_g(\b{x}) \log(1 - D^*(\b{x})) \Big) d\b{x} \\
&\overset{(\ref{equation_GAN_D_optimum})}{=} \int_{\b{x}} \Big[ p_\text{data}(\b{x}) \log\Big(\frac{p_\text{data}(\b{x})}{p_\text{data}(\b{x}) + p_g(\b{x})}\Big) \\
&~~~~~~~~~~~~~~ + p_g(\b{x}) \log\Big(\frac{p_g(\b{x})}{p_\text{data}(\b{x}) + p_g(\b{x})}\Big) \Big] d\b{x} \\
&= \int_{\b{x}} \Big[ p_\text{data}(\b{x}) \log\Big(\frac{p_\text{data}(\b{x})}{2 \times \frac{p_\text{data}(\b{x}) + p_g(\b{x})}{2}}\Big) \\
&~~~~~~~~~~~~~~ + p_g(\b{x}) \log\Big(\frac{p_g(\b{x})}{2 \times \frac{p_\text{data}(\b{x}) + p_g(\b{x})}{2}}\Big) \Big] d\b{x} \\
&= \int_{\b{x}} \Big[ p_\text{data}(\b{x}) \log\Big(\frac{p_\text{data}(\b{x})}{\frac{p_\text{data}(\b{x}) + p_g(\b{x})}{2}}\Big) \\
&~~~ + p_g(\b{x}) \log\Big(\frac{p_g(\b{x})}{\frac{p_\text{data}(\b{x}) + p_g(\b{x})}{2}}\Big) \Big] d\b{x} + \log(\frac{1}{2}) + \log(\frac{1}{2}) \\
&= \int_{\b{x}} \Big[ p_\text{data}(\b{x}) \log\Big(\frac{p_\text{data}(\b{x})}{\frac{p_\text{data}(\b{x}) + p_g(\b{x})}{2}}\Big) \\
&~~~~~~~~~~~~ + p_g(\b{x}) \log\Big(\frac{p_g(\b{x})}{\frac{p_\text{data}(\b{x}) + p_g(\b{x})}{2}}\Big) \Big] d\b{x} - \log(4) 
\end{align*}
\begin{align}
&\overset{(a)}{=} \text{KL}\Big(p_\text{data}(\b{x}) \Big\| \frac{p_\text{data}(\b{x}) + p_g(\b{x})}{2}\Big) \nonumber \\
&~~~~~~~~~~~ + \text{KL}\Big(p_g(\b{x}) \Big\| \frac{p_\text{data}(\b{x}) + p_g(\b{x})}{2}\Big) - \log(4), \label{equation_GAN_loss_KL}
\end{align}
where $(a)$ is because of the definition of KL divergence. 
The Jensen-Shannon Divergence (JSD) is defined as \cite{nielsen2010family}:
\begin{equation}\label{equation_GAN_JSD}
\begin{aligned}
&\text{JSD}(P \| Q) := \frac{1}{2} \text{KL}(P \| \frac{1}{2} (P + Q)) \\
&~~~~~~~~~~~~~~~~~~~~~+ \frac{1}{2} \text{KL}(Q \| \frac{1}{2} (P + Q)),
\end{aligned}
\end{equation}
where $P$ and $Q$ denote the probability densities. In contrast to KL divergence, the JSD is symmetric. The obtained $V(D^*, G)$ can be restated as:
\begin{align}\label{equation_GAN_loss_JSD}
V(D^*, G) = 2\, \text{JSD}\big(p_\text{data}(\b{x})\, \|\, p_g(\b{x})\big) - \log(4), 
\end{align}
According to Eq. (\ref{equation_GAN_loss}), the generator minimizes $V(D^*, G)$. As the JSD is non-negative, the above loss function is minimized if we have:
\begin{align*}
& \text{JSD}\big(p_\text{data}(\b{x})\, \|\, p_{g^*}(\b{x})\big) = 0 \implies p_\text{data}(\b{x}) = p_{g^*}(\b{x}).
\end{align*}
Q.E.D.
\end{proof}

\begin{corollary}[{\citep[Theorem 1]{goodfellow2014generative}}]
From Eqs. (\ref{equation_GAN_pg_equals_pdata}) and (\ref{equation_GAN_loss_JSD}), we conclude that the optimal loss function in GAN is:
\begin{align}
V(D^*, G^*) = -\log(4).
\end{align}
\end{corollary}

It is noteworthy that one can generalize Eq. (\ref{equation_GAN_JSD}) in GAN to \cite{huszar2015not}:
\begin{equation}\label{equation_GAN_JSD_interpolate}
\begin{aligned}
&\text{JSD}_\pi(P \| Q) := \pi\, \text{KL}\big(P \| \pi P + (1-\pi) Q\big) \\
&~~~~~~~~~~~~~~~~~~~~~+ (1-\pi)\, \text{KL}\big(Q \| \pi P + (1-\pi) Q\big),
\end{aligned}
\end{equation}
with $\pi \in (0,1)$. Its special case is Eq. (\ref{equation_GAN_JSD}) with $\pi = 0.5$.

\begin{corollary}
From Eqs. (\ref{equation_GAN_D_optimum}) and (\ref{equation_GAN_pg_equals_pdata}), we conclude that at convergence (i.e., Nash equilibrium), the discriminator cannot distinguish between generated and real data:
\begin{equation}
\begin{aligned}
& D^*(\b{x}) = 0.5, \quad \forall \b{x} \sim p_\text{data}(\b{x}), \\
& D^*(\b{x}) = 0.5, \quad \b{x} = G^*(\b{z}), \forall \b{z} \sim p_z(\b{z}).
\end{aligned}
\end{equation}
\end{corollary}

\begin{lemma}[Label smoothing in GAN {\citep[Section 3.4]{salimans2016improved}}]
It is shown that replacing labels $0$ and $1$, respectively, with smoother values $0.1$ and $0.9$ \cite{szegedy2016rethinking} can improve neural network against adversarial attacks \cite{hazan2017adversarial}. If we smooth the labels of discriminator $D$ for real and generated data to be $\alpha$ and $\beta$, respectively, the optimal discriminator becomes \cite{salimans2016improved}:
\begin{align}\label{equation_GAN_D_optimum_smooth_labels}
D^*(\b{x}) = \frac{\alpha p_\text{data}(\b{x}) + \beta p_g(\b{x})}{p_\text{data}(\b{x}) + p_g(\b{x})},
\end{align}
which generalizes Eq. (\ref{equation_GAN_D_optimum}). 
The presence of $p_g(\b{x})$ causes a problem because, for an $\b{x}$ with small $p_\text{data}(\b{x})$ and large $p_g(\b{x})$, the point does not change generator well enough to get close to the real data. Hence, it is recommended to set $\beta=0$ to have one-sided label smoothing. In this case, the optimal discriminator is:
\begin{align}\label{equation_GAN_D_optimum_smooth_labels_2}
D^*(\b{x}) = \frac{\alpha p_\text{data}(\b{x})}{p_\text{data}(\b{x}) + p_g(\b{x})}.
\end{align}
\end{lemma}
\begin{proof}[Proof (sketch)]
Using $\alpha$ and $\beta$ in the proof of Theorem \ref{theorem_GAN_optimal_discriminator} results in Eq. (\ref{equation_GAN_D_optimum_smooth_labels}).
\end{proof}

\subsection{Convergence and Equilibrium Analysis of GAN}

\begin{theorem}[{\citep[Proposition 2]{goodfellow2014generative}}]
If the discriminator and generator have enough capacity and, at every iteration of the alternating optimization, the discriminator is allowed to reach its optimum value as in Eq. (\ref{equation_GAN_D_optimum}), and $p_g(\b{x})$ is updated to minimize $V(D^*, G)$ stated in Eq. (\ref{equation_GAN_loss_JSD}), $p_g(\b{x})$ converges to $p_\text{data}(\b{x})$ as stated in Eq. (\ref{equation_GAN_pg_equals_pdata}). 
\end{theorem}
\begin{proof}
The KL divergences in Eq. (\ref{equation_GAN_loss_KL}) are convex functions w.r.t. $p_g(\b{x})$. Hence, with sufficiently small updates of $p_g(\b{x})$, it converges to $p_\text{data}(\b{x})$. Note that Eq. (\ref{equation_GAN_loss_KL}), which we used here, holds if Eq. (\ref{equation_GAN_D_optimum}) holds, i.e., the discriminator is allowed to reach its optimum value. Q.E.D.
\end{proof}

The GAN loss, i.e. Eq. (\ref{equation_GAN_loss}), can be restated as \cite{nagarajan2017gradient}:
\begin{equation}\label{equation_GAN_loss_2}
\begin{aligned}
\min_G \max_D\,\,\,\, &V(D,G) := \mathbb{E}_{\b{x} \sim p_\text{data}(\b{x})}\Big[f\big(D(\b{x})\big)\Big] \\
&+ \mathbb{E}_{\b{z} \sim p_z(\b{z})}\Big[f\!\Big(\!\!- \!\!D\big(G(\b{z})\big)\Big)\Big],
\end{aligned}
\end{equation}
where $f$ is the negative logistic function, i.e., $f(x) := -\log(1 + \exp(-x))$. In fact, the function $f(.)$ can be any concave function. This formulation is slightly different from the original GAN in the sense that, here, the discriminator $D$ outputs a real-valued scalar (without any activation function) while the discriminator of Eq. (\ref{equation_GAN_loss}) outputs values in the range $(0,1)$ after a sigmoid activation function. 
If $D$ outputs $0.5$ and $0$, it means that it is completely confused in Eqs. (\ref{equation_GAN_loss}) and (\ref{equation_GAN_loss_2}), respectively. 
The Eq. (\ref{equation_GAN_loss_2}) is a concave-concave loss function in most of the domain of discriminator {\citep[Proposition 3.1]{nagarajan2017gradient}}. 

\begin{theorem}[{\citep[Theorem 3.1]{nagarajan2017gradient}}]
After satisfying several reasonable assumptions (see \cite{nagarajan2017gradient} for details), a GAN with loss function of Eq. (\ref{equation_GAN_loss_2}) is locally exponentially stable. 
\end{theorem}


\begin{lemma}[Nash equilibrium in GAN \cite{farnia2020gans}]
Nash equilibrium is the state where no player can improve its gain by choosing a different strategy. At the Nash equilibrium of GAN, we have:
\begin{align}
V(D, G^*) \leq V(D^*, G^*) \leq V(D^*, G),
\end{align}
which is obvious because we are minimizing and maximizing $V(G, D)$ by the generator and discriminator, respectively, in Eq. (\ref{equation_GAN_loss}). 
\end{lemma}
Empirical experiments have shown that GAN may not reach its Nash equilibrium in practice \cite{farnia2020gans}. 
Regularization can help convergence of GAN to the Nash equilibrium \cite{mescheder2018training}. 
It is shown in \cite{mescheder2018training} that A effective regularization for GAN is noise injection \cite{ghojogh2019theory} in which independent Gaussian noise is added to the training data points. 

\begin{definition}[Proximal equilibrium \cite{farnia2020gans}]
We can use the proximal operator \cite{ghojogh2021kkt} in the loss function of GAN:
\begin{equation*}
\begin{aligned}
\min_G \max_D &\Big( V_\text{prox}(D,G) := \max_{\widetilde{D}} (V(\widetilde{D}, G) - \lambda \|\widetilde{D} - D\|_2^2) \Big),
\end{aligned}
\end{equation*}
where $V(D,G)$ is defined in Eq. (\ref{equation_GAN_loss}) and $\lambda>0$ is the regularization parameter. 
The equilibrium of the game having this loss function is called the proximal equilibrium. 
\end{definition}

\begin{theorem}[Convergence of GAN based on the Jacobian \cite{mescheder2017numerics}]
Let the updated solution of GAN optimization at every iteration be obtained by some operator $F(D,G)$, such as a step of gradient descent. The convergence of GAN can be explained based on the Jacobian of $F(D,G)$ with respect to $D$ and $G$. If the absolute values of some eigenvalues of the Jacobian are larger than one, GAN will not converge to the Nash equilibrium. If all eigenvalues have absolute values less than one, GAN will converge to the Nash equilibrium with a linear rate $\mathcal{O}(|\lambda_{\max}|^k)$ where $\lambda_{\max}$ is the eigenvalue with largest absolute eigenvalue and $k$ is the iteration index. If all eigenvalues have unit absolute value, GAN may or may not converge. 
\end{theorem}

The readers can refer to \cite{farnia2018convex} for duality in GAN, which is not explained here for brevity. 
Moreover, some papers have specifically combined GAN with game theory. Interested readers can refer to \cite{oliehoek2017gangs,arora2017generalization,unterthiner2018coulomb,tembine2019deep}.


\subsection{Conditional GAN}\label{section_conditional_GAN}

As was explained before, in the original GAN, we randomly draw noise $\b{z}$ from a prior distribution $p_z(\b{z})$ and feed it to the generator. The generator outputs a point $\b{x}$ from the noise $\b{z}$. Assume that the dataset with which GAN is trained has $c$ number of classes. The original GAN generates points from any class and we do not have control to generate a point from a specific class. Although, note that after GAN is trained, the latent space for $\b{z}$ is meaningful in the sense that every part of the latent space results in generation of some specific points from some class. However, the user cannot choose specifically what class to generate points from. 

Conditional GAN \cite{mirza2014conditional}, also called the conditional adversarial network, gives the user the opportunity to choose the class of generation of points. For the dataset $\{\b{x}_i \in \mathbb{R}^d\}_{i=1}^n$, let the one-hot encoded class labels be $\{\b{y}_i \in \mathbb{R}^c\}_{i=1}^n$. 
In conditional GAN, we use the following loss function instead of Eq. (\ref{equation_GAN_loss}):
\begin{equation}\label{equation_conditional_GAN_loss}
\begin{aligned}
\min_G \max_D\,\,\,\, &V_C(D,G) := \mathbb{E}_{\b{x} \sim p_\text{data}(\b{x})}\Big[\log\!\big(D(\b{x} | \b{y})\big)\Big] \\
&+ \mathbb{E}_{\b{z} \sim p_z(\b{z})}\Big[\log\!\Big(1 - D\big(G(\b{z} | \b{y})\big)\Big)\Big],
\end{aligned}
\end{equation}
where the discriminator and generator are both conditioned on the labels. 
In practice, for implementing the loss function (\ref{equation_conditional_GAN_loss}), we concatenate the one-hot encoded label $\b{y}$ to the point $\b{x}$ for the input to discriminator. We also concatenate the one-hot encoded label $\b{y}$ to the noise $\b{z}$ for the input to generator. For these, the input layers of discriminator and generator are enlarged to accept the concatenated inputs. 
In the test phase, the user choose the desired class label and the generator generates a new point from that class. 

\subsection{Deep Convolutional GAN (DCGAN)}\label{section_DCGAN}

Deep Convolutional GAN (DCGAN), proposed in \cite{radford2016unsupervised}, made GAN deeper and generated higher resolution images than GAN. It also showed that it is very hard to train a GAN. 

\subsubsection{Network Structure}

In DCGAN, we use an all-convolutional network \cite{springenberg2015striving} which replaces the pooling functions with strided convolutions. In this way, the network learns its own spatial downsampling. This network is used for both generator and discriminator. 
In DCGAN, we also have only convolutional layers in the input layer of generator and output layer of discriminator, without any fully-connected layer. This elimination of fully connected layers is inspired by \cite{mordvintsev2015inceptionism}.

We also apply batch normalization \cite{ioffe2015batch} to all layers except the last layer of generative and the first layer of discriminator. This is because batch normalization makes the mean of input to each neuron zero and its variance one; however, we should learn the mean and variance of data in the first layer of discriminator and the mean and variance of data should be reproduced by the last layer of generator. Batch normalization reduces the problem of mode collapse, which will be introduced later in Section \ref{section_mode_collapse} with the price of causing some fluctuations and instability \cite{radford2016unsupervised}. 

\begin{remark}[Virtual batch normalization {\citep[Section 3.5]{salimans2016improved}}]
Batch normalization has a problem; it makes the effect of every input $\b{x}$ on network dependent on other inputs in the mini-batch. To not have this problem, Virtual Batch Normalization (VBN) fixes the mini-batches initially once before start of training; these mini-batches are called the reference batches. Every reference batch is normalized by only its own statistics (i.e., mean and covariance). vbn has been found to be effective in training of generator \cite{salimans2016improved}. 
\end{remark}

\begin{figure*}[!t]
\centering
\includegraphics[width=5in]{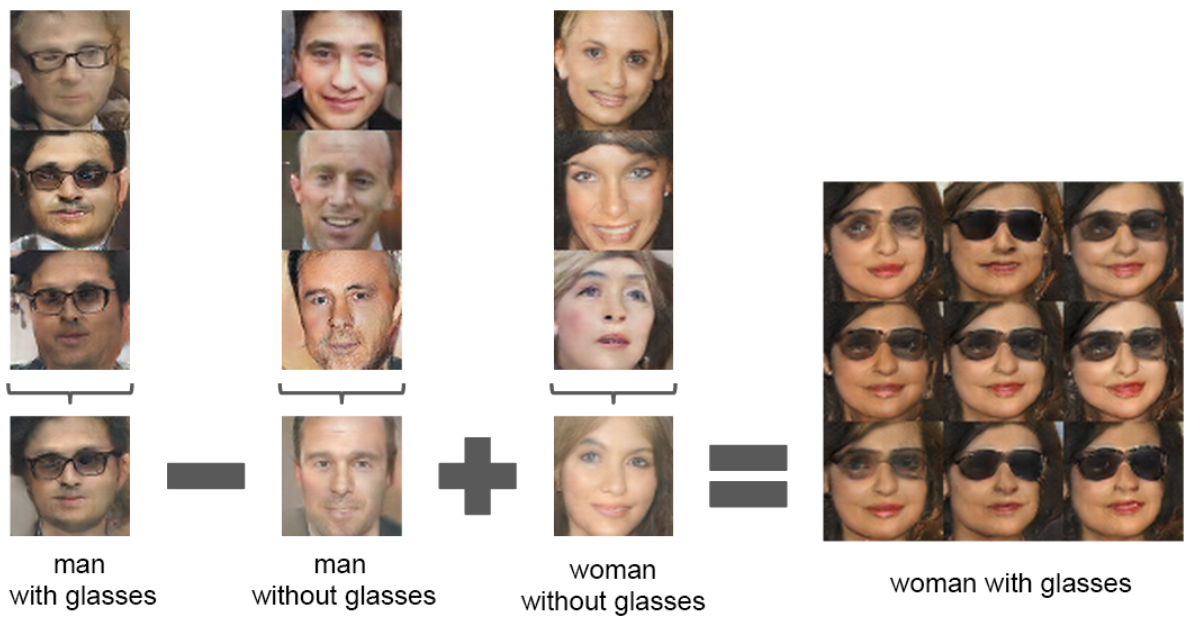}
\caption{Vector arithmetic in the latent space on images by DCGAN. Credt of image is for \cite{radford2016unsupervised}.}
\label{figure_vector_arithmetic}
\end{figure*}

In DCGAN, the last layer of generator has the hyperbolic tangent activation function and its other layers have the ReLU activation function \cite{nair2010rectified}. 
As in GAN, the one-to-last layer of discriminator is flattened and connected to one neuron with the sigmoid activation function. In contrast to GAN, which uses the maxout activation function \cite{goodfellow2013maxout} for discriminator layers (see Section \ref{section_GAN_structure}), DCGAN uses the leaky rectified action function for discriminator. 

\subsubsection{Vector Arithmetic in Latent Space}\label{section_vector_arithmetic_DCGAN}

DCGAN showed that we can generate images from a specific domain if we train GAN on that domain. for example, bedroom images were generated by DCGAN after being trained on a dataset of bedroom images. 
DCGAN also showed that the learned latent space is meaningful and we can do vector arithmetic in the latent space. Vector arithmetic in the latent space was previously used for showing the ability of Word2Vec \cite{mikolov2013distributed}. 
DCGAN made it possible to do vector arithmetic in the latent space for images. 
An example of vector arithmetic by DCGAN is shown in Fig. \ref{figure_vector_arithmetic}. 
In the latent space, the latent variables corresponding to man with glasses, man without glasses, and woman without glasses are used. Using one latent point for each of these does not work very well. An average of three latent vectors for each has worked properly in practice.  
As Fig. \ref{figure_vector_arithmetic} shows, vector arithmetic works because ``man with glasses" minus ``man without glasses" plus ``woman without glasses" results in ``woman with glasses". 


\section{Mode Collapse Problem in GAN}\label{section_mode_collapse}

\subsection{Mode Collapse Problem}

We expect from a GAN to learn a meaningful latent space of $\b{z}$ so that every specific value of $\b{z}$ maps to a specific generated data point $\b{x}$. Also, nearby $\b{z}$ values in the latent space should be mapped to similar but a little different generations.  
The mode collapse problem \cite{metz2017unrolled}, also known as the Helvetica scenario \cite{goodfellow2016nips}, is a common problem in GAN models. 
It refers to when the generator cannot learn a perfectly meaningful latent space as was explained. Rather, it learns to map several different $\b{z}$ values to the same generated data point. 
Mode collapse usually happens in GAN when the distribution of training data, $p_\text{data}(\b{x})$, has multiple modes.

An example of mode collapse is illustrated in Fig. \ref{figure_Mode_collapse} which shows training steps of a GAN model when the training data is a mixture of Gaussians \cite{metz2017unrolled}. In different training steps, GAN learns to map all $\b{z}$ values to one of the modes of mixture. When the discriminator learns to reject generation of some mode, the generator learns to map all $\b{z}$ values to another mode.
However, it never learns to generate all modes of the mixture. We expect GAN to map some part, and not all parts, of the latent space to one of the modes so that all modes are covered by the whole latent space. 

Another statement of the mode collapse is as follows {\citep[Fig. 1]{xiao2018bourgan}}. 
Assume $p_\text{data}(\b{x})$ is multi-modal while the latent space $p_z(\b{z})$ has only one mode. Consider two points $\b{x}_1$ and $\b{x}_2$ from two modes of data whose corresponding latent noises are $\b{z}_1$ and $\b{z}_2$, respectively. According to the mean value theorem, there is a latent noise with the absolute gradient value $\|\b{x}_2 - \b{x}_1\| / \|\b{z}_2 - \b{z}_1\|$ where $\|.\|$ is a norm. As this gradient is Lipschitz continuous, when the two modes are very far resulting in a large $\|\b{x}_2 - \b{x}_1\|$, we face a problem. In this case, the latent noises between $\b{z}_1$ and $\b{z}_2$ generate data points between $\b{x}_1$ and $\b{x}_2$ which are not in the modes of data and thus are not valid. 

There exist various methods which resolve the mode collapse problem in GAN and adversarial learning. Some of them make the latent space a mixture distribution to imitate generation of the multi-modal training data. Some of them, however, have other approaches. In the following, we introduce the methods which tackle the mode collapse problem. 

\begin{figure*}[!t]
\centering
\includegraphics[width=6.5in]{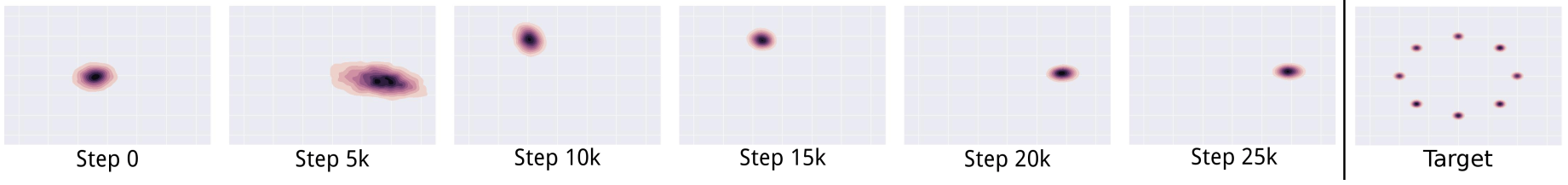}
\caption{An example of mode collapse in GAN. Image is from \cite{metz2017unrolled}.}
\label{figure_Mode_collapse}
\end{figure*}

\subsection{Minibatch GAN}

One way to resolve the mode collapse problem is mini-batch discrimination {\citep[Section 3.2]{salimans2016improved}}. In this method, the discriminator considers multiple data points in combination rather than separately. This avoids the mode collapse in generator. Suppose $\b{f}(\b{x}_i) \in \mathbb{R}^a$ is the feature vector of one of the intermediate layers, with $a$ neurons, in the discriminator for the data point $\b{x}_i$. The data point $\b{x}_i$ is either real or generated (fake). We multiply $\b{f}(\b{x}_i)$ by a tensor $\b{T} \in \mathbb{R}^{a \times b \times c}$ to obtain $\mathbb{R}^{b \times c} \ni \b{M}_i := (\b{T}^\top \b{f}(\b{x}_i))^\top$. If there are $|\mathcal{B}|$ points in the mini-batch $\mathcal{B}$, we can calculate $\{\b{M}_i\}_{i=1}^{|\mathcal{B}|}$. Let $(\b{M}_i)_{l:}$ denote the $l$-th row of $\b{M}_i$. We define \cite{salimans2016improved}:
\begin{equation}
\begin{aligned}
& \mathbb{R} \ni c_l(\b{x}_i, \b{x}_j) := \exp(-\|(\b{M}_i)_{l:} - (\b{M}_j)_{l:}\|_1), \\
&~~~~~~~~~~~~~~~~~~~~~~~~~~~~~~~~~~~~~~~~~~~~~~~~~~~~~~~~~ \forall l \in \{1, \dots, b\}, \\
& \mathbb{R} \ni o(\b{x}_i)_l := \sum_{j=1}^{|\mathcal{B}|} c_l(\b{x}_i, \b{x}_j), \quad \forall i \in \{1, \dots, |\mathcal{B}|\}, \\
& \mathbb{R}^b \ni \b{o}(\b{x}_i) := [o(\b{x}_i)_1, \dots, o(\b{x}_i)_b]^\top, \quad \forall i, \\
& \mathbb{R}^{|\mathcal{B}| \times b} \ni \b{o}(\b{X}) := [\b{o}(\b{x}_1), \dots, \b{o}(\b{x}_{|\mathcal{B}|})]^\top. 
\end{aligned}
\end{equation}
For every point $\b{x}_i$ within the mini-batch, we concatenate $\b{o}(\b{x}_i)$ with $\b{f}(\b{x}_i)$ and feed it to the next layer, rather than feeding merely $\b{f}(\b{x}_i)$. 
In other words, the additional features $\b{o}(\b{x}_i)$ are side information for better training of discriminator (which makes the generator also stronger in the game). 
This procedure, in the discriminator, is performed for both mini-batches of real and generated data. 

\subsection{Unrolled GAN}

Unrolled GAN \cite{metz2017unrolled} uses $\Delta$ levels of unrolling of discriminator when updating the generator. 
The alternating optimization in unrolled GAN is:
\begin{align}
& D^{(k+1)} := D^{(k)} + \eta \frac{\partial }{\partial D} \Big(V(D,G^{(k)})\Big), \label{equation_unrolled_GAN_alternating_opt_D} \\
& 
\left\{
    \begin{array}{ll}
        D_0 := D^{(k+1)}, \\
        D_{(\delta+1)} := D_{(\delta)} + \eta^{(\delta)} \frac{\partial }{\partial D} \Big(V(D,G^{(k)})\Big), \\
        ~~~~~~~~~~~~~~~~~~~~~~~~~~~~~~~~~~~~\forall \delta \in \{0, \dots, \Delta-1\}, \\
        V_\Delta(D^{(k+1)},G) := V(D_{(\Delta)},G), 
    \end{array}
\right. \label{equation_unrolled_GAN_alternating_opt_D_unrolling}
\\
& G^{(k+1)} := G^{(k)} - \eta \frac{\partial }{\partial G} \Big(V_\Delta(D^{(k+1)},G)\Big), \label{equation_unrolled_GAN_alternating_opt_G}
\end{align}
where Eq. (\ref{equation_unrolled_GAN_alternating_opt_D_unrolling}) unrolls the parameters of discriminator for $\Delta$ times before using it for updating the generator. 
The loss function $V(D_{(\Delta)},G)$ is called the surrogate loss. 
As was mentioned in Section \ref{section_GAN_optimization_and_loss}, the original GAN can update the discriminator itself for several time before updating the generator \cite{goodfellow2014generative}. Note that Eq. (\ref{equation_unrolled_GAN_alternating_opt_D_unrolling}) is different from updating the discriminator for several times, as done in GAN, because it does not update the discriminator but it is used in updating the generator in Eq. (\ref{equation_unrolled_GAN_alternating_opt_G}).

The gradient for generator in Eq. (\ref{equation_unrolled_GAN_alternating_opt_G}) can be simplified as:
\begin{align*}
&\frac{\partial }{\partial G} \Big(V_\Delta(D^{(k+1)},G)\Big) \overset{(\ref{equation_unrolled_GAN_alternating_opt_D_unrolling})}{=} \\
&\frac{\partial }{\partial G} \Big(V_\Delta(D_{(\Delta)},G)\Big) + \frac{\partial }{\partial D_{(\Delta)}} \Big(V_\Delta(D_{(\Delta)},G)\Big) \frac{\partial D_{(\Delta)}}{\partial G}.
\end{align*}
The gradient for generator in the original GAN has only the first term. The second term in the above equation captures the information of changes of discriminator w.r.t. changes in the generator. This reduces the problem of mode collapse which exists in GAN. 

\subsection{Bourgain GAN (BourGAN)}\label{section_BourGAN}

Bourgain GAN (BourGAN) \cite{xiao2018bourgan} learns a mixture distribution \cite{ghojogh2019fitting} in the latent space to resolve the problem of mode collapse and learn to generate multi-modal data. 
If the size of dataset, $n$, is large, it samples $m$ points from dataset $\{\b{x}_i\}_{i=1}^n$ where $m \ll n$. Let $\{\widetilde{\b{x}}_i\}_{i=1}^m$ denote the sampled data.
The idea of BourGAN is based on the Bourgain embedding \cite{bourgain1985lipschitz} which enables embedding a dataset of size $m$ into a $\mathcal{O}(\log^2(m))$-dimensional $\ell_2$ norm embedding space with high probability. 
An improved Bourgain embedding is as follows. 

\begin{theorem}[{\citep[Corollary 2]{xiao2018bourgan}}]
For a dataset $\{\widetilde{\b{x}}_i\}_{i=1}^m$ in a space with norm $\|.\|$, there exists a mapping $f$ from the data space to a $\mathcal{O}(\log(m))$-dimensional embedding space which preserves the local distances:
\begin{align*}
\|\widetilde{\b{x}}_i - \widetilde{\b{x}}_j\| \leq \|f(\widetilde{\b{x}}_i) - f(\widetilde{\b{x}}_j)\| \leq \alpha\, \|\widetilde{\b{x}}_i - \widetilde{\b{x}}_j\|,
\end{align*}
where $\alpha \leq \mathcal{O}(\log(m))$. 
This embedding is achieved by Bourgain embedding \cite{bourgain1985lipschitz} followed by random projection \cite{johnson1984extensions,ghojogh2021johnson}. This combined embedding can be found in {\citep[Appendix A]{xiao2018bourgan}}. 
\end{theorem}
This embedding requires computing pairwise distances.
The sampling of $m$ points from the $n$ data points is to make this embedding feasible. 
This sampling does not affect the characteristics of pairwise distances in data if $m$ is sufficiently large {\citep[Theorem 4]{xiao2018bourgan}}. 
We embed all sampled points to obtain $\{f(\widetilde{\b{x}}_i)\}_{i=1}^m$.
The distance characteristics of data is preserved by this embedding {\citep[Theorem 5]{xiao2018bourgan}}. 

As the next step in BourGAN, we sample from the embedding points uniformly, i.e., $\b{\mu} \sim \{f(\widetilde{\b{x}}_i)\}_{i=1}^m$. 
Then, we sample the latent noise from a multivariate Gaussian distribution with mean $\b{\mu}$, i.e., $\b{z} \sim \mathcal{N}(\b{\mu}, 0.01 \b{I})$. This procedure models a multi-modal mixture of Gaussians in the latent space of GAN. Note that it does not imply a mixture of $m$ modes because if several $f(\widetilde{\b{x}}_i)$'s are close to each other, they can be considered as one mode. The multi-modality of the latent space is preserved to be the same as the multi-modality of data {\citep[Eq. 6]{xiao2018bourgan}}.

The loss function of BourGAN regularizes the cost of generator so that the generator preserves the distances of generated points compared to the distances of corresponding noises. In this way, the multi-modality of latent space will also appear in the distribution of generated data $p_g(\b{x})$. 
\begin{equation}\label{equation_BourGAN_loss}
\begin{aligned}
&\max_{D}\,\,\,\, V_\text{BourGAN}(D) := \,V(D,G), \\
&\min_{G}\,\,\,\, V_\text{BourGAN}(G) := \, V(D,G) \\
&~~~~~~~~~~~~~~~~~~ + \lambda\, \mathbb{E}_{\b{z}_i, \b{z}_j \sim p_z(\b{z})}\Big[\big(\log(\|G(\b{z}_i) - G(\b{z}_j)\|) \\
&~~~~~~~~~~~~~~~~~~~~~~~~~~~~~~~~~~~~~ - \log(\|\b{z}_i - \b{z}_j\|)\big)^2\Big].
\end{aligned}
\end{equation}
where $V(D,G)$ is defined in Eq. (\ref{equation_GAN_loss}), $\lambda >0$ is the regularization parameter, and $\|.\|$ is some norm such as $\ell_2$ norm. 

\subsection{Mixture GAN (MGAN)}\label{section_MGAN}


Mixture GAN (MGAN) \cite{hoang2018mgan} overcomes the mode collapse issue by assuming that the distribution of latent space, from which noise is sampled, is a mixture distribution \cite{ghojogh2019fitting} rather than a single distribution. 
It also enlarges the divergence of latent distributions in the mixture so that each of them covers a different mode for generation of data. 
The structure of MGAN is shown in Fig. \ref{figure_MGAN}. It has $k$ generators $\{G_j\}_{j=1}^k$, a discriminator $D$, and a classifier $C$. 
In terms of having a classifier, it is similar to triple GAN \cite{li2017triple} (see Section \ref{section_triple_GAN}).
Every generator $G_j$ takes care of the $j$-th latent distribution in the mixture. 
The difference of MGAN from BourGAN (see Section \ref{section_BourGAN}) is that MGAN associates a generator to every mode while BourGAN has one generator with a mixture latent distribution.

Let $\pi_j$ be the mixing probability for the $j$-th component in the mixture. 
We denote the distribution of generated data from the mixture latent distribution and the $j$-th component in the mixture by $p_g(\b{x})$ and $p_{g_j}(\b{x})$, respectively. 
The discriminator tries to judge whether a data point is real, $\b{x} \sim p_\text{data}(\b{x})$, or generated from one of the modes in the mixture, i.e., $\b{x} = G_j(\b{z})$. 
At every iteration of the alternating optimization, the index $j$ of the selected $G_j$ for feeding to discriminator $D$ is sampled from a multinomial distribution Mult($\b{\pi}$) where $\b{\pi} := [\pi_1, \dots, \pi_k]^\top$. 
The classifier $C$ classifies which one of the $k$ generators has generated the generated data. 
The loss function of MGAN, as a multi-player minimax game, is:
\begin{equation}\label{equation_MGAN_loss}
\begin{aligned}
&\min_{\{G_j\}_{j=1}^k, C} \max_{D}\,\,\,\, V(D, C, G_1, \dots, G_k) := \\
&\mathbb{E}_{\b{x} \sim p_\text{data}(\b{x})}\Big[\log\!\big(D(\b{x})\big)\Big] + \mathbb{E}_{\b{x} \sim p_g(\b{x})}\Big[\log\!\big(1 - D(\b{x})\big)\Big] \\
&- \lambda \sum_{j=1}^k \pi_j\, \mathbb{E}_{\b{x} \sim p_{g_j}(\b{x})}\Big[\log\!\big(C_j(\b{x})\big)\Big],
\end{aligned}
\end{equation}
where $\lambda > 0$ is the regularization parameter and $C_j(\b{x})$ is the probability that $\b{x}$ is generated by $G_j$. The last layer of the classifier $C$ has $k$ neurons with softmax activation function where $C_j$ is the output of $j$-th neuron in classifier after the activation function. The last term in the loss function maximizes the entropy of classification so that, by competition of generators and classifier, the generated data from various generators become separated gradually. In this way, generators will cover different modes of data, resolving the mode collapse problem.

\begin{figure}[!t]
\centering
\includegraphics[width=3.2in]{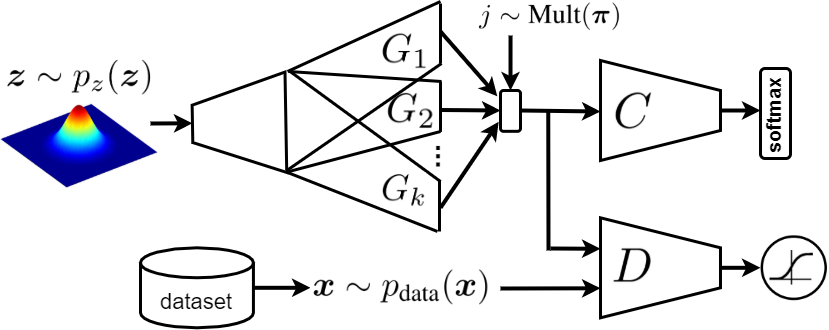}
\caption{The structure of MGAN.}
\label{figure_MGAN}
\end{figure}

\begin{theorem}[{\citep[Theorem 2]{hoang2018mgan}}]
After convergence (i.e., the Nash equilibrium) of MGAN, we have:
\begin{equation}
\begin{aligned}
&p_{g^*}(\b{x}) = p_\text{data}(\b{x}), \\
&D^*(\b{x}) = \frac{p_\text{data}(\b{x})}{p_\text{data}(\b{x}) + p_g(\b{x})}, \\
&C_j^*(\b{x}) = \frac{\pi_j\, p_{g^*_j}(\b{x})}{\sum_{l=1}^k \pi_l\, p_{g^*_l}(\b{x})}, \\
&G^* := \arg \min_G \big( 2\, \text{JSD}(p_\text{data} \| p_g) \\
&~~~~~~~~~~~~~~~~~~~~~~~~~~~~~ - \lambda\, \text{JSD}_\pi(p_{G_1}, \dots, p_{G_k}) \big),
\end{aligned}
\end{equation}
where:
\begin{equation}
\begin{aligned}
&\text{JSD}_\pi(p_{G_1}, \dots, p_{G_k}) := \\
&\sum_{j=1}^k \pi_j \mathbb{E}_{\b{x} \sim p_{g_j}}[\log (\frac{\pi_j\, p_{g_j}(\b{x})}{\sum_{l=1}^k \pi_l\, p_{g_l}(\b{x})})] - \sum_{j=1}^k \pi_j\, \log(\pi_j). 
\end{aligned}
\end{equation}
\end{theorem}
The above theorem means that the JSD between the mixture distribution $p_g$ and the data distribution $p_\text{data}$ is minimized; however, the JSD between the the components of mixture is maximized so that the components capture various modes of data. 

The below theorem shows that if the distribution of data is actually a mixture distribution itself, the optimal generation distribution at the Nash equilibrium becomes exactly that mixture. 
\begin{theorem}[{\citep[Theorem 3]{hoang2018mgan}}]
If the data distribution is a mixture $p_\text{data} = \sum_{j=1}^k \pi_j\, q_j(\b{x})$ where the components $q_j(\b{\b{x}})$'s are well-separated, the optimal generation distribution in MGAN is:
\begin{equation}
\begin{aligned}
& p_{g^*_j}(\b{x}) = q_j(\b{x}), \quad \forall j = 1, \dots, k, \\
& p_{g^*}(\b{x}) = p_\text{data} = \sum_{j=1}^k \pi_j\, q_j(\b{x}).
\end{aligned}
\end{equation}
\end{theorem}

\subsection{Dual Discriminator GAN (D2GAN)}\label{section_D2GAN}

It is empirically observed \cite{theis2016note,huszar2015not,goodfellow2016nips} that the JSD used in GAN (see Eq. (\ref{equation_GAN_loss_JSD})) has the same effect as reverse KL divergence $\text{KL}(p_g(\b{x}) \| p_\text{data}(\b{x}))$. This reverse KL divergence has the problem of mode collapse because it covers a single mode very well but cannot cover multiple modes well. 
That is while the KL divergence $\text{KL}(p_\text{data}(\b{x}) \| p_g(\b{x}))$ can cover multiple modes and does not have the mode collapse problem; however, it can include potentially undesirable samples \cite{nguyen2017dual}.
Dual Discriminator GAN (D2GAN) \cite{nguyen2017dual} combines the advantages of both KL divergence and reverse KL divergence by having both in its formulation. Therefore, it does not face a mode collapse while it prevents undesirable samples. 

In D2GAN, we have two discriminators $D_1$ and $D_2$ and one generator $G$. The discriminators do not share their weights. In contrast to the original GAN, the outputs of discriminators are non-negative rather than being in range $[0,1]$. 
The discriminator $D_1$ gives high and low scores to real and generated (fake) data, respectively. Conversely, the discriminator $D_2$ gives high and low scores to generated (fake) and real data, respectively. The generator $G$ tries to fool both discriminators. 
D2GAN plays a three-player game whole loss function is:
\begin{equation}\label{equation_D2GAN_loss}
\begin{aligned}
&\min_G \max_{D_1, D_2}\,\,\,\, V(D_1, D_2, G) := \\
&\alpha \mathbb{E}_{\b{x} \sim p_\text{data}(\b{x})}\Big[\log\!\big(D_1(\b{x})\big)\Big] + \mathbb{E}_{\b{z} \sim p_z(\b{z})}\Big[\!-\! D_1\big(G(\b{z})\big)\Big] \\
&+ \mathbb{E}_{\b{x} \sim p_\text{data}(\b{x})}\Big[\!-\! D_2(\b{x})\Big] + \beta \mathbb{E}_{\b{z} \sim p_z(\b{z})}\Big[\log\!\big(D_2\big(G(\b{z})\big)\big)\Big],
\end{aligned}
\end{equation}
where $\alpha, \beta \in (0,1]$ are hyperparameters. 
Alternating optimization \cite{ghojogh2021kkt} between $D_1$, $D_2$, and $G$ solves the problem. 

\begin{theorem}[{\citep[Proposition 1 and Theorem 2]{nguyen2017dual}}]
After convergence (i.e., Nash equilibrium) of D2GAN, we have:
\begin{align}
& D_1^*(\b{x}) = \frac{\alpha p_\text{data}(\b{x})}{p_g(\b{x})}, \quad D_2^*(\b{x}) = \frac{\beta p_g(\b{x})}{p_\text{data}(\b{x})}, 
\end{align}
The loss function at the optimal discriminators is:
\begin{align}
&V(D^*_1, D^*_2, ,G) := \alpha (\log(\alpha)-1) + \beta (\log(\beta)-1) \nonumber \\
&~~~~~~~~~~ + \alpha \text{KL}(p_\text{data}(\b{x}) \| p_g(\b{x})) + \beta \text{KL}(p_g(\b{x}) \| p_\text{data}(\b{x})). \label{equation_D2GAN_V_alpha_beta} \\
&\!\implies\!\!\! V(D^*_1, D^*_2, ,G^*) := \alpha (\log(\alpha)-1) + \beta (\log(\beta)-1). \nonumber
\end{align}
Therefore,
\begin{align}
& p_{g^*}(\b{x}) = p_\text{data}(\b{x}), \quad D_1^*(\b{x}) = \alpha, \quad D_2^*(\b{x}) = \beta.
\end{align}
\end{theorem}
According to Eq. (\ref{equation_D2GAN_V_alpha_beta}), the parameters $\alpha$ and $\beta$ are for the KL divergence $\text{KL}(p_\text{data}(\b{x}) \| p_g(\b{x}))$ and the reverse KL divergence $\text{KL}(p_g(\b{x}) \| p_\text{data}(\b{x}))$, respectively. 
Therefore, increasing $\alpha$ results in generation several modes, resolving the mode collapsing issue, but may include some undesirable samples. Increasing $\beta$ results in generation of a single mode but might miss several modes. A balance should be kept between the parameters $\alpha$ and $\beta$. 

\subsection{Wasserstein GAN (WGAN)}\label{section_WGAN}

Wasserstein GAN (WGAN) was proposed in \cite{arjovsky2017wasserstein}, developed from \cite{arjovsky2017towards}.
The Wasserstein-1 or Earth-Mover distance between two distributions $p_\text{data}(\b{x})$ and $p_g(\b{x})$ is defined as:
\begin{align}
&W(p_\text{data}(\b{x}), p_g(\b{x})) \nonumber\\
&~~~~~~~~ := \inf_{\gamma \in \Pi(p_\text{data}(\b{x}), p_g(\b{x}))} \mathbb{E}_{(\b{x}_i,\b{x}_j) \sim \gamma} [\|\b{x}_i - \b{x}_j\|], 
\end{align}
where $\Pi(p_\text{data}(\b{x}), p_g(\b{x}))$ is the set of all joint distributions whose marginals are $p_\text{data}(\b{x})$ and $p_g(\b{x})$. 
By the Kantorovich-Rubinstein duality \cite{villani2009optimal}, the Wasserstein-1 distance is equivalent to:
\begin{align}
&W(p_\text{data}(\b{x}), p_g(\b{x})) \nonumber\\
&= \sup_{\|D\|_L \leq 1} \big(\mathbb{E}_{\b{x} \sim p_\text{data}(\b{x})} [D(\b{x})] - \mathbb{E}_{\b{x} \sim p_g(\b{x})} [D(\b{x})]\big), \label{equation_WGAN_D}
\end{align}
where $\|D\|_L \leq 1$ is the 1-Lipschitz functions $D: \mathcal{X} \rightarrow \mathbb{R}$. 
The gradient of the Wasserstein-1 distance w.r.t. the parameters of generator $G$ is {\citep[Theorem 3]{arjovsky2017wasserstein}}:
\begin{align}\label{equation_WGAN_G}
\frac{\partial W(p_\text{data}(\b{x}), p_g(\b{x}))}{\partial G} = -\mathbb{E}_{\b{z} \sim p_z(\b{z})}\Big[\frac{\partial D(G(\b{\b{z}}))}{\partial G}\Big].
\end{align}
The function $D(.)$ plays the role of discriminator in WGAN. 
In an alternating optimization, we maximize the loss in Eq. (\ref{equation_WGAN_D}) for the discriminator and minimize in a gradient descent step by the gradient of Eq. (\ref{equation_WGAN_G}) for generator. 
In other words, the loss function of WGAN is:
\begin{align}
\min_G \max_{\|D\|_L \leq 1} \mathbb{E}_{\b{x} \sim p_\text{data}(\b{x})} [D(\b{x})] - \mathbb{E}_{\b{x} \sim p_g(\b{x})} [D(\b{x})],
\end{align}
where $\b{x} \sim p_g(\b{x})$ is generated from the generator, i.e., $\b{x} = G(\b{z})$ in which $\b{z}$ is the latent noise.
The weights of discriminator are clipped to $[-0.01, 0.01]$. 
The constraint $\|D\|_L \leq 1$ can be implemented by regularization with a gradient penalty \cite{gulrajani2017improved}:
\begin{equation}\label{equation_WGAN_loss}
\begin{aligned}
\min_G \max_{\|D\|_L \leq 1} &\mathbb{E}_{\b{x} \sim p_\text{data}(\b{x})} [D(\b{x})] - \mathbb{E}_{\b{x} \sim p_g(\b{x})} [D(\b{x})] \\
& -\lambda ( \nabla_{\widehat{\b{x}}} D(\widehat{\b{x}})\|_2 - 1 )^2,
\end{aligned}
\end{equation}
where $\lambda>0$ is the regularization parameter and $\widehat{\b{x}}$ is uniformly sampled as an interpolation between the real data and the generated data:
\begin{align}\label{equation_WGAN_x_hat}
\widehat{\b{x}} := \rho\, \b{x} + (1 - \rho) G(\b{z}),
\end{align}
in which $\rho \sim U(0,1)$.
Experiments show that WGAN resolves the mode collapse problem \cite{arjovsky2017wasserstein}. 

\section{Maximum Likelihood Estimation in GAN}

In the following, we introduce the methods which relate GAN and Maximum Likelihood Estimation (MLE).

\subsection{Comparison of MLE and GAN}

GAN is related to Noise-Contrastive Estimation (NCE) \cite{gutmann2010noise} and MLE, in the sense that they all optimize a distinguishing game value function \cite{goodfellow2015distinguishability}:
\begin{equation}
\begin{aligned}
& V(p_c, p_g) := \mathbb{E}_{\b{x} \sim p_\text{data}}[\log(p_c(y=1|\b{x}))] \\
&~~~~~~~~~~~~~~~~~ + \mathbb{E}_{\b{x} \sim p_g}[\log(p_c(y=0|\b{x}))],
\end{aligned}
\end{equation}
where $p_g$ and $p_\text{data}$ denote the distributions of generated and real data, respectively, and $p_c(y|\b{x})$ is the output probability of a classifier which judges whether a data point $\b{x}$ is generated or real. 
Let $\b{\theta}$ denote the parameters of $p_g(\b{x})$ distribution. If $f(\b{x}) := \log(p_c(y=0|\b{x}))$, GAN performs the following optimization for its generator \cite{goodfellow2015distinguishability}:
\begin{align}
&\min_{\b{\theta}}\, \mathbb{E}_{\b{x} \sim p_g}[f(\b{x})] \label{equation_MLE_GAN_min_E_f}\\
&\implies \frac{\partial }{\partial \b{\theta}} \mathbb{E}_{\b{x} \sim p_g}[f(\b{x})] = \frac{\partial }{\partial \b{\theta}} \int f(\b{x}) p_g(\b{x}) d\b{x} \nonumber\\
&= \int f(\b{x}) \frac{\partial }{\partial \b{\theta}} p_g(\b{x}) d\b{x} \nonumber\\
&= \int f(\b{x}) p_g(\b{x}) \frac{1}{p_g(\b{x})} \frac{\partial }{\partial \b{\theta}} p_g(\b{x}) d\b{x} \nonumber\\
&= \int f(\b{x}) p_g(\b{x}) \frac{\partial }{\partial \b{\theta}} \log(p_g(\b{x})) d\b{x}. \label{equation_GAN_G_MLE}
\end{align}
On the other hand, MLE has the following optimization:
\begin{align}
&\max_{\b{\theta}}\, \mathbb{E}_{\b{x} \sim p_g}[p_\text{data}(\b{x})] \nonumber \\
&\implies\!\! \frac{\partial }{\partial \b{\theta}} \mathbb{E}_{\b{x} \sim p_g}[p_\text{data}(\b{x})] = \int p_\text{data}(\b{x}) \frac{\partial }{\partial \b{\theta}} \log(p_g(\b{x})) d\b{x}. \label{equation_MLE}
\end{align}
Eqs. (\ref{equation_GAN_G_MLE}) and (\ref{equation_MLE}) are for minimization in GAN and maximization in MLE, respectively. By their comparison, we can have MLE in GAN if we set:
\begin{align}\label{equation_MLE_GAN_f}
f(\b{x}) = - \frac{p_\text{data}(\b{x})}{p_g(\b{x})}.
\end{align}
The discriminator of GAN is modeled as a classifier for $\b{x}$ being real and not generated (fake); hence:
\begin{align}\label{equation_MLE_GAN_D_sigma}
D(\b{x}) = p_c(y=1|\b{x}) = \sigma(D'(\b{x})), 
\end{align}
where $\sigma(.)$ is the sigmoid activation function and $D'(\b{x})$ denotes the discriminator network except the sigmoid function at its last layer.

\begin{theorem}[\cite{goodfellow2015distinguishability,goodfellow2016nips}]
The loss function for the generator of GAN can be stated as any of the following loss functions:
\begin{align}
& \min_G \mathbb{E}_{\b{z} \sim p_z(\b{z})}\Big[\log\!\Big(1 - D\big(G(\b{z})\big)\Big)\Big], \label{equation_MLE_GAN_loss_G_1} \\
& \min_G -\mathbb{E}_{\b{z} \sim p_z(\b{z})}\Big[\log\!\Big(D\big(G(\b{z})\big)\Big)\Big], \label{equation_MLE_GAN_loss_G_2} \\
& \min_G -\mathbb{E}_{\b{z} \sim p_z(\b{z})}\Big[\sigma^{-1}\Big(D\big(G(\b{z})\big)\Big)\Big], \label{equation_MLE_GAN_loss_G_3}
\end{align}
where $\sigma(.)$ is the sigmoid function.
\end{theorem}
\begin{proof}
Eq. (\ref{equation_MLE_GAN_loss_G_1}) is the generator part of loss function (\ref{equation_GAN_loss}). The generator wants to fool the discriminator so it wants $D(G(\b{z}))$ to be close to one (see Eq. (\ref{equation_GAN_D})). Hence, rather than minimizing $\log(1 - D(G(\b{z})))$ in Eq. (\ref{equation_MLE_GAN_loss_G_1}), we can maximize $\log(D(G(\b{z})))$, or minimize its negation, in Eq. (\ref{equation_MLE_GAN_loss_G_2}) \cite{goodfellow2016nips}. 
The proof of Eq. (\ref{equation_MLE_GAN_loss_G_3}) is as follows \cite{goodfellow2015distinguishability}. 
Assume the discriminator is optimal for a given generator; hence, according to Eq. (\ref{equation_GAN_D_optimum}), we have:
\begin{align*}
p_c(y=1|\b{x}) &\overset{(\ref{equation_GAN_D_optimum})}{=} \frac{p_\text{data}(\b{x})}{p_\text{data}(\b{x}) + p_g(\b{x})} = \frac{1}{1 + \frac{p_g(\b{x})}{p_\text{data}(\b{x})}} \\
&\overset{(\ref{equation_MLE_GAN_D_sigma})}{=} \sigma(D'(\b{x})) = \frac{1}{1 + \exp(-D'(\b{x}))} 
\end{align*}
\begin{align*}
&\implies \frac{p_g(\b{x})}{p_\text{data}(\b{x})} = \exp(-D'(\b{x})) \\
&\overset{(\ref{equation_MLE_GAN_f})}{\implies} f(\b{x}) = -\exp(D'(\b{x})) \overset{(\ref{equation_MLE_GAN_D_sigma})}{=} -\exp\big(\sigma^{-1}(D(\b{x}))\big).
\end{align*}
Hence, for the generated data $\b{x} = G(\b{z})$, from the latent noise sample $\b{z} \sim p_z(\b{z})$, the Eq. (\ref{equation_MLE_GAN_min_E_f}) becomes Eq. (\ref{equation_MLE_GAN_loss_G_3}). Q.E.D.
\end{proof}

\subsection{f-GAN}

f-GAN \cite{nowozin2016f} uses f-divergence in the formulation of GAN. The f-GAN computes divergence between two distributions $p(\b{x})$ and $q(\b{x})$ by \cite{liese2006divergences}:
\begin{align}
D_f(P \| Q) := \int q(\b{x}) f\big(\frac{p(\b{x})}{q(\b{x})}\big) d\b{x}, 
\end{align}
where the convex so-called generator function $f: \mathbb{R}_+ \rightarrow \mathbb{R}$ satisfies $f(1)=0$. 
Two special cases of f-GAN are KL-divergence and JL-divergence. 
We denote the space of data by $\mathcal{X}$.

\begin{lemma}[{\citep[Lemma 1]{nguyen2010estimating}}]
A lower-bound on the f-divergence is as follows:
\begin{align}\label{equation_f_divergence_lower_bound}
D_f(P \| Q) \geq \sup_{T \in \mathcal{T}}(\mathbb{E}_{\b{x} \sim p(\b{x})}[T(\b{x})] - \mathbb{E}_{\b{x} \sim q(\b{x})}[f^*(T(\b{x}))]),
\end{align}
where $f^*$ is the convex conjugate of $f$ and $\mathcal{T}: \mathcal{X} \rightarrow \mathbb{R}$ is an arbitrary class of functions.
\end{lemma}
\begin{proof}
The convex conjugate of function $f$ is defined as \cite{ghojogh2021kkt}:
\begin{align}
&f^*(t) := \sup_{u \in \text{dom}(f)} (ut - f(u)) \nonumber\\
&\implies f(u) := \sup_{t \in \text{dom}(f^*)} (tu - f^*(t)). \label{equation_f_convex_conjugate}
\end{align}
We have:
\begin{align*}
&D_f(P \| Q) \overset{(\ref{equation_f_convex_conjugate})}{=} \int q(\b{x}) \sup_{t \in \text{dom}(f^*)} \big(t\, \frac{p(\b{x})}{q(\b{x})} - f^*(t)\big)\, d\b{x} \\
&\overset{(a)}{\geq} \sup_{T \in \mathcal{T}} \big(\int p(\b{x}) T(\b{x}) d\b{x} - \int q(\b{x}) f^*(T(\b{x})) d\b{x}\big) \\
&= \sup_{T \in \mathcal{T}}(\mathbb{E}_{\b{x} \sim p(\b{x})}[T(\b{x})] - \mathbb{E}_{\b{x} \sim q(\b{x})}[f^*(T(\b{x}))]),
\end{align*}
where $(a)$ is because the summation of maximums is greater than or equal to the maximum of summations. Q.E.D.
\end{proof}

Variational Divergence Minimization (VDM) \cite{nowozin2016f} optimizes the f-divergence by optimizing the bound in Eq. (\ref{equation_f_divergence_lower_bound}). In this sense, it is similar to variational inference \cite{ghojogh2021factor}. 
Suppose $p(\b{x}) = p_\text{data}(\b{x})$ and $q(\b{x}) = p_g(\b{x})$ in Eq. (\ref{equation_f_divergence_lower_bound}), where $p_\text{data}(\b{x})$ and $p_g(\b{x})$ are the distributions of real and generated data, respectively. Let $T(\b{x}) = o_f(V(\b{x}))$ where $V: \mathcal{X} \rightarrow \mathbb{R}$ is the mapping of network from its input to one output neuron (before activation) and $o_f: \mathbb{R} \rightarrow \text{dom}(f^*)$ is the output of activation function. 
VDM can be used for optimization of various f-divergences. Its loss function is:
\begin{equation}\label{equation_VDM_loss}
\begin{aligned}
\min_G \max_V\,\,\,\, &\mathbb{E}_{\b{x} \sim p_\text{data}(\b{x})}\Big[o_f(V(\b{x}))\Big] \\
&~~~~~~~~~~~ + \mathbb{E}_{\b{x} \sim p_g(\b{x})}\Big[\!-\!f^*\big(o_f(V(\b{x}))\big)\Big].
\end{aligned}
\end{equation}
The reader can refer to {\citep[Table 2]{nowozin2016f}} for a complete list of expressions for $o_f(v)$ and $f^*$ in different special cases of f-divergence.
A special case of VDM is f-GAN in which $o_f(v) = -\log(1 + \exp(-v))$, $f^*(t) = -\log(1-\exp(t))$ and the discriminator is the sigmoid function of $V(\b{x})$, i.e., $D(\b{x}) = 1 / (1 + \exp(-V(\b{x})))$. 
Hence, in f-GAN, we have:
\begin{align*}
& o_f(V(\b{x})) = -\log(1 + \exp(-V(\b{x}))) = \log (D(\b{x})), \\
& f^*\big(o_f(V(\b{x}))\big) = -\log(1-\exp(\log (D(\b{x})))) \\
&~~~~~~~~~~~~~~~~~~~~~~~~ = -\log(1 - D(\b{x})).
\end{align*}
Putting these in Eq. (\ref{equation_VDM_loss}) gives the loss of f-GAN:
\begin{equation}\label{equation_f_GAN_loss}
\begin{aligned}
\min_G \max_D\,\,\,\, &\mathbb{E}_{\b{x} \sim p_\text{data}(\b{x})}\Big[\log\!\big(D(\b{x})\big)\Big] \\
&+ \mathbb{E}_{\b{x} \sim p_g(\b{x})}\Big[\log\!\big(1 - D\big(\b{x}\big)\big)\Big].
\end{aligned}
\end{equation}

\subsection{Adversarial Variational Bayes (AVB)}

Adversarial Variational Bayes (AVB) \cite{mescheder2017adversarial} combines the ideas of variational and adversarial training. Variational inference \cite{ghojogh2021factor} maximizes an evidence lower bound defined as:
\begin{align}\label{equation_AVB_ELBO_optimization}
\max_{\theta} \max_{\phi}\, \mathbb{E}_{p(\b{x})} \mathbb{E}_{q_{\phi}(\b{z}|\b{x})} \big[&\log(p(\b{z})) - \log(q_{\phi}(\b{z}|\b{x})) \nonumber\\
&+ \log(p_{\theta}(\b{x}|\b{z}))\big],
\end{align}
where $\theta$ and $\phi$ are parameters corresponding to $p_{\theta}(\b{x}|\b{z})$ and $q_{\phi}(\b{z}|\b{x})$, respectively. 
On the other hand, adversarial learning uses a discriminator in training.
AVB uses a discriminator $D(\b{x},\b{z})$ with one output neuron having a sigmoid activation function in variational inference. The loss function for the discriminator is:
\begin{align}
\max_D\,\, &\mathbb{E}_{p(\b{x})} \mathbb{E}_{q_{\phi}(\b{z}|\b{x})} \big[\log(D(\b{x},\b{z}))\big] \nonumber\\
&+ \mathbb{E}_{p(\b{x})} \mathbb{E}_{p_{\b{z}}(\b{z})} \big[ \log(1 - D(\b{x},\b{z}))\big],
\end{align}
whose solution is $D^*(\b{x}, \b{z}) = -\log(p(\b{z})) + \log(q_{\phi}(\b{z}|\b{x}))$ {\citep[Proposition 1]{mescheder2017adversarial}}. Therefore, Eq. (\ref{equation_AVB_ELBO_optimization}) becomes:
\begin{align}\label{equation_AVB_ELBO_optimization_2}
\max_{\theta} \max_{\phi}\, \mathbb{E}_{p(\b{x})} \mathbb{E}_{q_{\phi}(\b{z}|\b{x})} \big[&-D^*(\b{x}, \b{z}) + \log(p_{\theta}(\b{x}|\b{z}))\big],
\end{align}
which is optimized by backpropagation, after the reparameterization trick \cite{ghojogh2021factor}.

\subsection{Bayesian GAN (BGAN)}

Bayesian GAN (BGAN) \cite{saatci2017bayesian} models GAN using Bayesian analysis. Let $G$ and $D$ denote the parameters of generator and discriminator, respectively, $\b{x}$ be the real data, $\b{z}$ be the noise sample, and $b$ be the mini-batch size.
\begin{align*}
& p(G|\b{z}, D) \propto \Big(\prod_{i=1}^b D(G(\b{z}_i))\Big)\, p(G), \\
& p(D|\b{z}, \b{x}) \propto \Big(\prod_{i=1}^b D(\b{x}_i)\Big) \Big(\prod_{i=1}^b \big(1-D(G(\b{z}_i))\big)\Big)\, p(D).
\end{align*}
We can marginalize these distributions:
\begin{align}
&p(G|D) = \int p(G,\b{z}|D)\, d\b{z} = \int p(G|\b{z},D)\, p(\b{z}|D)\, d\b{z} \nonumber\\
&~~~~~\overset{(a)}{=} \int p(G|\b{z},D)\, p_z(\b{z})\, d\b{z} \overset{(b)}{\approx} \frac{1}{b} \sum_{i=1}^b p(G|\b{z}_i), \label{equation_BGAN_p_G_D}
\end{align}
where $\b{z}_i \sim p_z(\b{z})$, $(a)$ is because the noise $\b{z}$ is independent of the discriminator $D$, and $(b)$ is because of the Monte Carlo approximation \cite{ghojogh2020sampling}. 
Similarly, we have:
\begin{align}\label{equation_BGAN_p_D_G}
&p(D|G) = \frac{1}{b} \sum_{i=1}^b p(D|\b{z}_i, \b{x}_i).
\end{align}
Sampling from the distributions in Eqs. (\ref{equation_BGAN_p_G_D}) and (\ref{equation_BGAN_p_D_G}) will converge to the joint distribution of generator and discriminator, based on Gibbs sampling \cite{ghojogh2020sampling}.
Stochastic Gradient Hamiltonian Monte Carlo (SGHMC) is a technique for training a neural network using the posteriors. The discriminator and generator of BGAN are trained alternatively using this technique and the posteriors in Eqs. (\ref{equation_BGAN_p_G_D}) and (\ref{equation_BGAN_p_D_G}).
Note that another GAN model with variational inference and Bayesian analysis is the variational Bayesian GAN \cite{chien2019variational}.


\section{Other Variants of GAN}

\subsection{Feature Matching in GAN}

During training, the layers of discriminator $D$ are trained to have discriminative features between the real and generated data. Therefore, for better training of the generator $G$ and fooling the discriminator by it, we can use the features of an intermediate layer of discriminator {\citep[Section 3.1]{salimans2016improved}}. 
We train the generator to match the expected values of the intermediate features for inputs of real and generated data. Hence, the optimization of generator can be:
\begin{equation}\label{equation_featureMatching_GAN_loss}
\begin{aligned}
\min_G\,\,\,\, & \big\|\mathbb{E}_{\b{x} \sim p_\text{data}(\b{x})} [\b{f}(\b{x})] - \mathbb{E}_{\b{z} \sim p_z(\b{z})}\big[\b{f}\big(G(\b{z})\big)\big] \big\|_2^2,
\end{aligned}
\end{equation}
where $\b{f}(\b{x})$ and $\b{f}(G(\b{z}))$ are the features of an intermediate layer of discriminator for inputs $\b{x}$ (real data) and $G(\b{z})$ (generated data), respectively. 
The discriminator is trained as in the original GAN, i.e., maximization in Eq. (\ref{equation_GAN_loss}). 

\subsection{InfoGAN}

Information maximizing GAN (InfoGAN), proposed in \cite{chen2016infogan}, is an information-theoretic approach to GAN. 
It maximizes the mutual information between latent variables and generated data. 
In InfoGAN, we have two sets of latent variables, i.e., $\b{z}$ and $\b{c}$. The generator gets these two latent variables as input and outputs $G(\b{z}, \b{c})$. 
The optimization problem in InfoGAN is a regularized problem as:
\begin{equation}\label{equation_InfoGAN_loss}
\begin{aligned}
\min_G \max_D\,\,\,\, &V_I(D,G) := V(D,G) - \lambda I(\b{c}; G(\b{z}, \b{c})),
\end{aligned}
\end{equation}
where $V(D,G)$ is defined in Eq. (\ref{equation_GAN_loss}), the $\lambda > 0$ is the regularization parameter and $I(.;.)$ is the mutual information defined as $I(\b{c}; G(\b{z}, \b{c})) := H(\b{c}) - H(\b{c} | G(\b{z}, \b{c}))$ in which $H(.)$ is the entropy. 
Note that the added regularization term depends only on $G$ and not $D$. The generator maximizes the mutual information $I(\b{c}; G(\b{z}, \b{c}))$.

Computing this mutual information is difficult in practice. 
The mutual information can be simplified as the following by introducing an auxiliary distribution $Q(\b{c}| \b{x})$.
\begin{align*}
& I(\b{c}; G(\b{z}, \b{c})) := H(\b{c}) - H(\b{c} | G(\b{z}, \b{c})) \\
&\overset{(a)}{=} H(\b{c}) - \big(\!-\mathbb{E}_{\b{x} \sim G(\b{z}, \b{c})}[\log P(\b{c}| \b{x})]\big) \\
&= H(\b{c}) + \mathbb{E}_{\b{x} \sim G(\b{z}, \b{c})}\big[\mathbb{E}_{\b{c}' \sim P(\b{c}| \b{x})} [\log P(\b{c}'| \b{x})]\big] \\
&\overset{(b)}{=} H(\b{c}) + \mathbb{E}_{\b{x} \sim G(\b{z}, \b{c})}\big[\text{KL}(P(.|\b{x}) \| Q(.|\b{x})) \\
&~~~~~~~~~~~~~~~+ \mathbb{E}_{\b{c}' \sim P(\b{c}| \b{x})} [\log Q(\b{c}'| \b{x})]\big] \\
&\overset{(c)}{\geq} H(\b{c}) + \mathbb{E}_{\b{x} \sim G(\b{z}, \b{c})}\big[\mathbb{E}_{\b{c}' \sim P(\b{c}| \b{x})} [\log Q(\b{c}'| \b{x})]\big] \\
&\overset{(d)}{=} H(\b{c}) + \mathbb{E}_{\b{c} \sim P(\b{c}),\, \b{x} \sim G(\b{z}, \b{c})}[\log Q(\b{c}'| \b{x})] \overset{(e)}{=} L_I(G,Q),
\end{align*}
where $(a)$ is because of definition of entropy, $(b)$ is because of the definition of KL divergence, $(c)$ is because the KL divergence is non-negative, $(d)$ is because $\mathbb{E}_{\b{x} \sim X, \b{y} \sim Y|X}[f(\b{x}, \b{y})] = \mathbb{E}_{\b{x} \sim X, \b{y} \sim Y|X, \b{x}' \sim X|Y}[f(\b{x}', \b{y})]$ (see {\citep[Lemma 5.1]{chen2016infogan}}), and $(e)$ is because we define $L_I(G,Q)$ as that expression. Hence, $L_I(G,Q)$ is a lower-bound for $I(\b{c}; G(\b{z}, \b{c}))$. Using this lower-bound in Eq. (\ref{equation_InfoGAN_loss}) gives:
\begin{equation}\label{equation_InfoGAN_loss_2}
\begin{aligned}
\min_{G,Q} \max_D\,\,\,\, &V_I(D,G) := V(D,G) - \lambda L_I(G,Q),
\end{aligned}
\end{equation}
where:
\begin{align*}
L_I(G,Q) := H(\b{c}) + \mathbb{E}_{\b{c} \sim P(\b{c}),\, \b{x} \sim G(\b{z}, \b{c})}[\log Q(\b{c}'| \b{x})],
\end{align*}
can be calculated by Monte Carlo approximation \cite{ghojogh2020sampling}.

\subsection{Generative Recurrent Adversarial Network (GRAN)}

Generative Recurrent Adversarial Network (GRAN) \cite{im2016generating} has been inspired by the Deep Recurrent Attentive Writer (DRAW) \cite{gregor2015draw}. DRAW uses variational inference for drawing images gradually on canvas by passing time. GRAN does the same but using adversarial learning. Therefore, it is a combination of GAN and recurrent networks. 
In GRAN, the generator $G$ has a recurrent feedback loop whose inputs are a sequence of noise samples $\{\b{z}_t\}_{t=1}^T$. The recurrent loop of generator generates a sequence of drawings on canvas, i.e., $\{\Delta C_1, \Delta C_2, \dots, \Delta C_T\}$. 
Every recurrent loop, at time $t \in \{1, \dots, T\}$, is like an autoencoder with encoder $f(.)$ and decoder $g(.)$. The coding layer between the encoder and decoder gives the concatenation of latent coding $\b{h}_{z,t}$ and canvas coding $\b{h}_{c,t}$. This coding concatenation is fed to the decoder $f$ to result the canvas drawing $\Delta C_t$. 
In every recurrent loop, at time $t \in \{1, \dots, T\}$, we have:
\begin{equation}
\begin{aligned}
&\b{z}_t \sim p_z(\b{z}), \\
&\b{h}_{c,t} := g(\Delta C_{t-1}), \\
&\b{h}_{z,t} := \text{tanh}(\b{W} \b{z}_t + \b{b}), \\
&\Delta C_t := f([\b{h}_{z,t}^\top, \b{h}_{c,t}^\top]^\top),
\end{aligned}
\end{equation}
where $\b{W}$ and $\b{b}$ are the layer weights and bias weights for the latent variable $\b{z}_t$. 
We use DCGAN \cite{radford2016unsupervised} (see Section \ref{section_DCGAN}) for the encoder $f$ and decoder $g$ at every recurrent loop, where the canvas drawings $\{\Delta C_t\}_{t=1}^T$ are generated. 
The total canvas drawing is the summation of drawings at the time slots. We use a logistic function $\sigma(.)$ to scale the drawing to $(0,1)$ for the sake of pixel visualization:
\begin{align*}
C = \sigma\Big(\sum_{t=1}^T \Delta C_t\Big). 
\end{align*}

\subsection{Least Squares GAN (LSGAN)}

The GAN loss function has a problem. In the discriminator, the gradient vanishes for the generated data points which fall on the correct side of decision boundary but are still different from the real data. Least Squares GAN (LSGAN) \cite{mao2017least,mao2019effectiveness} resolves this issue by using least squares cost in the adversarial loss function.
For the discriminator $D$ of LSGAN, we use two scalar labels $a$ and $b$ for generated (fake) and real data points. 
For the generator $G$ of LSGAN, we use the scalar label $c$ which the generator wants the discriminator to believe for in classification. 
The loss functions in LSGAN are:
\begin{equation}\label{equation_LSGAN_loss}
\begin{aligned}
\min_{D}\,\,\,\, V_\text{LSGAN}(D) := &\,\frac{1}{2} \mathbb{E}_{\b{x} \sim p_\text{data}(\b{x})}\big[(D(\b{x}) - b)^2\big] \\
&+ \frac{1}{2} \mathbb{E}_{\b{z} \sim p_z(\b{z})}\big[(D(G(\b{z})) - a)^2\big], \\
\min_{G}\,\,\,\, V_\text{LSGAN}(G) := &\, \frac{1}{2} \mathbb{E}_{\b{z} \sim p_z(\b{z})}\big[(D(G(\b{z})) - c)^2\big].
\end{aligned}
\end{equation}

\begin{lemma}[{\citep[Proposition 1]{mao2019effectiveness}}]
For a fixed generator $G$, the optimal discriminator in LSGAN is:
\begin{align}\label{equation_LSGAN_D_optimum}
D^*(\b{x}) = \frac{b p_\text{data}(\b{x}) + a p_g(\b{x})}{p_\text{data}(\b{x}) + p_g(\b{x})},
\end{align}
where $p_\text{data}(\b{x})$ is the probability distribution of real dataset evaluated at point $\b{x}$ and $p_g(\b{x})$ is the probability distribution of output of generator evaluated at point $\b{x}$.
\end{lemma}
\begin{proof}
\begin{align*}
&V_\text{LSGAN}(D) \overset{(\ref{equation_LSGAN_loss})}{=} \frac{1}{2} \mathbb{E}_{\b{x} \sim p_\text{data}(\b{x})}\big[(D(\b{x}) - b)^2\big] \\
&~~~~~~~~~~~~~~~~~~~~~~~~ + \frac{1}{2} \mathbb{E}_{\b{z} \sim p_z(\b{z})}\big[(D(G(\b{z})) - a)^2\big] \\
&= \int_{\b{x}} \frac{1}{2} p_\text{data}(\b{x}) (D(\b{x}) - b)^2 d\b{x} \\
&~~~~~~~~~~~~~~~~~~~~~~~~+ \int_{\b{x}} \frac{1}{2} p_z(\b{z}) (D(G(\b{z})) - a)^2 d\b{x} \\
&\overset{(\ref{equation_pg_pz_relation})}{=} \int_{\b{x}} \frac{1}{2} \Big( p_\text{data}(\b{x}) (D(\b{x}) - b)^2 + p_g(\b{x}) (D(\b{x}) - a)^2 \Big) d\b{x}
\end{align*}
For optimization in Eq. (\ref{equation_LSGAN_loss}), taking derivative w.r.t. $D(\b{x})$ gives:
\begin{align*}
&\frac{\partial V_\text{LSGAN}(D)}{\partial D(\b{x})} \\
&\overset{(a)}{=} \frac{\partial }{\partial D(\b{x})} \Big( \frac{1}{2} \big( p_\text{data}(\b{x}) (D(\b{x}) - b)^2 \\
&~~~~~~~~~~~~~~~~~~~~~~~~ + p_g(\b{x}) (D(\b{x}) - a)^2 \big) \Big) \\
&= p_\text{data}(\b{x}) (D(\b{x}) - b) + p_g(\b{x}) (D(\b{x}) - a) \overset{\text{set}}{=} 0 \\
&\implies D(\b{x}) = \frac{b p_\text{data}(\b{x}) + a p_g(\b{x})}{p_\text{data}(\b{x}) + p_g(\b{x})},
\end{align*}
where $(a)$ is because taking derivative w.r.t. $D(\b{x})$ considers a specific $\b{x}$ and hence it removes the integral (summation). Q.E.D.
\end{proof}

\begin{theorem}[{\citep[Theorem 1]{mao2019effectiveness}}]
Optimization of LSGAN is equivalent to minimizing the Pearson $\chi^2$ divergence between $p_\text{data}(\b{x}) + p_g(\b{x})$ and $2p_g(\b{x})$, if we have:
\begin{equation}\label{equation_LSGAN_a_b_c_conditions}
\begin{aligned}
& b-c=1, \quad b-a=2.
\end{aligned}
\end{equation}
\end{theorem}
\begin{proof}
\begin{align*}
&2V_\text{LSGAN}(G) \overset{(\ref{equation_LSGAN_loss})}{=} \mathbb{E}_{\b{x} \sim p_z(\b{z})}\big[(D^*(G(\b{z})) - c)^2\big] \\
&\overset{(a)}{=} \mathbb{E}_{\b{x} \sim p_\text{data}(\b{x})}\big[(D^*(\b{x}) - c)^2\big] \\
&~~~~~~~~~~~~~~~~~~~~~~ + \mathbb{E}_{\b{z} \sim p_z(\b{z})}\big[(D^*(G(\b{z})) - c)^2\big] \\
&\overset{(\ref{equation_pg_pz_relation})}{=} \mathbb{E}_{\b{x} \sim p_\text{data}(\b{x})}\big[(D^*(\b{x}) - c)^2\big] \\
&~~~~~~~~~~~~~~~~~~~~~~ + \mathbb{E}_{\b{x} \sim p_g(\b{x})}\big[(D^*(\b{x}) - c)^2\big] \\
&\overset{(\ref{equation_LSGAN_D_optimum})}{=} \mathbb{E}_{\b{x} \sim p_\text{data}(\b{x})}\Big[(\frac{b p_\text{data}(\b{x}) + a p_g(\b{x})}{p_\text{data}(\b{x}) + p_g(\b{x})} - c)^2\Big] \\
&~~~~~~~~~~~~~~ + \mathbb{E}_{\b{x} \sim p_g(\b{x})}\Big[(\frac{b p_\text{data}(\b{x}) + a p_g(\b{x})}{p_\text{data}(\b{x}) + p_g(\b{x})} - c)^2\Big] 
\end{align*}
\begin{align*}
&\overset{(b)}{=} \int_{\b{x}} p_\text{data}(\b{x}) \Big(\frac{(b-c) p_\text{data}(\b{x}) + (a-c) p_g(\b{x})}{p_\text{data}(\b{x}) + p_g(\b{x})}\Big)^2 d\b{x} \\
&~~~~~~~~~~ + \int_{\b{x}} p_g(\b{x}) \Big(\frac{(b-c) p_\text{data}(\b{x}) + (a-c) p_g(\b{x})}{p_\text{data}(\b{x}) + p_g(\b{x})}\Big)^2 d\b{x} \\
&\overset{(c)}{=} \int_{\b{x}} \frac{\big((b-c) p_\text{data}(\b{x}) + (a-c) p_g(\b{x})\big)^2}{p_\text{data}(\b{x}) + p_g(\b{x})} d\b{x} \\
&= \int_{\b{x}} \frac{\big((b-c) (p_\text{data}(\b{x}) + p_g(\b{x})) - (b-a) p_g(\b{x})\big)^2}{p_\text{data}(\b{x}) + p_g(\b{x})} d\b{x} \\
&\overset{(\ref{equation_LSGAN_a_b_c_conditions})}{=} \int_{\b{x}} \frac{\big(2 p_g(\b{x}) - (p_\text{data}(\b{x}) + p_g(\b{x}))\big)^2}{p_\text{data}(\b{x}) + p_g(\b{x})} d\b{x} \\
&\overset{(d)}{=} \chi^2(p_\text{data}(\b{x}) + p_g(\b{x})\, \|\, 2p_g(\b{x})),
\end{align*}
where $(a)$ is because $\mathbb{E}_{\b{x} \sim p_\text{data}(\b{x})}\big[(D^*(\b{x}) - c)^2\big]$ is constant w.r.t. $G$, $(b)$ is because of the definition of expectation, $(c)$ is because of simplification of terms, and $(d)$ is because of definition of Pearson $\chi^2$ divergence. Q.E.D.
\end{proof}
As we saw, the labels $a$, $b$, and $c$ in LSGAN should satisfy Eq. (\ref{equation_LSGAN_a_b_c_conditions}). An options for satisfying these conditions is:
\begin{align}
& a=-1, \quad b=1, \quad c=0, 
\end{align}
which means the real and fake labels for discriminator are $+1$ and $-1$, respectively, while the generator fools the discriminator by label $0$. In other words, the generator does not take it very hard on the discriminator and sets the fake label to $0$ (some moderate value) rather than $1$. 
Another possible option for the labels is:
\begin{align}
& a=0, \quad b=c=1, 
\end{align}
which does not satisfy Eq. (\ref{equation_LSGAN_a_b_c_conditions}) but fools the discriminator completely (with more power) by the generator. 
Experiments have shown both of these options perform equally well in practice \cite{mao2017least}. 

\subsection{Energy-based GAN (EBGAN)}

In energy-based learning \cite{lecun2006tutorial}, a function is learned which maps data points to some energy values where the incorrectly labeled data points are assigned higher energy values. In unsupervised energy-based learning, higher energy is assigned to data points away from the data manifold or data cloud. 
Energy-based GAN (EBGAN) \cite{zhao2017energy} uses energy-based learning in adversarial learning. 
The loss functions in EBGAN are:
\begin{equation}\label{equation_EBGAN_loss}
\begin{aligned}
\min_{D}\,\,\,\, V_\text{EBGAN}(D) := &\,D(\b{x}) + [m - D(G(\b{z}))]_+, \\
\min_{G}\,\,\,\, V_\text{EBGAN}(G) := &\, D(G(\b{z})),
\end{aligned}
\end{equation}
where $[.]_+ := \max(.,0)$ is the standard Hinge loss and $m > 0$ is the margin. 
The discriminator minimizes the error of $D(\b{x})$ while maximizing $D(G(\b{z}))$ not to be fooled by the generator. 
The generator minimizes $D(G(\b{z}))$ to fool the discriminator. 

\begin{theorem}[{\citep[Theorem 1]{zhao2017energy}}]
Let:
\begin{align}
Q(D, G) := \int_{\b{x}, \b{z}} V_\text{EBGAN}(D)\, p_\text{data}(\b{x})\, p_z(\b{z})\, d\b{x}\, d\b{z}.
\end{align}
Optimization of EBGAN results in $p_g(\b{x}) = p_\text{data}(\b{x})$ and $Q(D^*, G^*) = m$ after convergence (i.e., Nash equilibrium). 
\end{theorem}
\begin{proof}
\begin{align}
&Q(D, G^*) \nonumber\\
&\overset{(\ref{equation_EBGAN_loss})}{=}\!\!\! \int_{\b{x}, \b{z}}\!\! \big(D(\b{x}) + [m - D(G^*(\b{z}))]_+\big)\, p_\text{data}(\b{x})\, p_z(\b{z})\, d\b{x}\, d\b{z} \nonumber\\
&= \int_{\b{x}} D(\b{x})\, p_\text{data}(\b{x})\, d\b{x} \underbrace{\int_{\b{z}} p_z(\b{z})\, d\b{z}}_{=1} \nonumber\\
&+ \underbrace{\int_{\b{x}} p_\text{data}(\b{x})\, d\b{x}}_{=1} \int_{\b{z}} [m - D(G^*(\b{z}))]_+\, p_z(\b{z})\, d\b{z} \nonumber\\
&\overset{(\ref{equation_pg_pz_relation})}{=} \int_{\b{x}} \big( p_\text{data}(\b{x})\, D(\b{x})\, + p_{g^*}(\b{x})\, [m - D(\b{x})]_+ \big)\, d\b{x}. \label{equation_EBGAN_Q}
\end{align}
The function inside the integral is $a t + b[m - t]_+$ whose minimum occurs if $a < b$. Hence, the minimum of $Q(D,G^*)$ is:
\begin{align}
&Q(D^*, G^*) = m \int_{\b{x}} \mathbb{I}(p_\text{data}(\b{x}) < p_{g^*}(\b{x}))\, p_\text{data}(\b{x}) d\b{x} \nonumber\\
&+ m \int_{\b{x}} \mathbb{I}(p_\text{data}(\b{x}) \geq p_g(\b{x}))\, p_{g^*}(\b{x}) d\b{x} \nonumber\\
&= m \int_{\b{x}} \Big( \mathbb{I}(p_\text{data}(\b{x}) < p_{g^*}(\b{x}))\, p_\text{data}(\b{x}) \nonumber\\
&+ \big( 1 - \mathbb{I}(p_\text{data}(\b{x}) < p_{g^*}(\b{x}))\big)\, p_{g^*}(\b{x}) \Big) d\b{x} \nonumber\\
&= m \underbrace{\int_{\b{x}} p_{g^*}(\b{x}) d\b{x}}_{=1} \nonumber\\
&+ m \int_{\b{x}} \mathbb{I}(p_\text{data}(\b{x}) < p_{g^*}(\b{x}))\, \big(p_\text{data}(\b{x}) - p_{g^*}(\b{x})\big) d\b{x}. \label{equation_EBGAN_Q_2}
\end{align}
As the probability of generated data $p_{g^*}(\b{x})$ is upper-bounded by the probability of data $p_\text{data}(\b{x})$, the second term is non-positive. Hence, $Q(D,G) \leq m$. On the other hand, as $p_{g^*}(\b{x}) \leq p_\text{data}(\b{x})$, we have:
\begin{align*}
\int_{\b{x}} p_{g^*}(\b{x}) D^*(\b{x}) d\b{x} \leq \int_{\b{x}} p_\text{data}(\b{x}) D^*(\b{x}) d\b{x}.
\end{align*}
Using this in Eq. (\ref{equation_EBGAN_Q}) gives:
\begin{align*}
&Q(D^*,G^*) \\
&\geq \int_{\b{x}} \big( p_{g^*}(\b{x}) D^*(\b{x}) + p_{g^*}(\b{x})\, [m - D^*(\b{x})]_+ \big)\, d\b{x} \\
&\overset{(a)}{=} \int_{\b{x}} p_{g^*}(\b{x}) D^*(\b{x}) d\b{x} + \int_{\b{x}} p_{g^*}(\b{x})\, (m - D^*(\b{x}))\, d\b{x} \\
&= m \underbrace{\int_{\b{x}} p_{g^*}(\b{x})\, d\b{x}}_{=1} = m,
\end{align*}
where $(a)$ is because $D^*(\b{x}) \leq m$ almost everywhere at the Nash equilibrium (since discriminator is trained at the convergence to not violate the margin). 
We showed that $m \leq Q(D^*,G^*) \leq m$ so $Q(D^*,G^*) = m$. From $Q(D^*,G^*) = m$ and Eq. (\ref{equation_EBGAN_Q_2}), we have $\int_{\b{x}} \mathbb{I}(p_\text{data}(\b{x}) < p_{g^*}(\b{x})) d\b{x} = 0$. As we have $p_{g^*}(\b{x}) \leq p_\text{data}(\b{x})$, this only holds when $p_\text{data}(\b{x}) = p_{g^*}(\b{x})$. Q.E.D.
\end{proof}

\subsection{Semi-supervised GAN}

In the following, we introduce the semi-supervised methods in the GAN literature. 

\subsubsection{Categorical GAN (CatGAN)}

\textbf{-- Unsupervised CatGAN:}
In Categorical GAN (CatGAN) \cite{springenberg2016unsupervised}, the discriminator classifies $c$ classes (i.e., categories) rather than a binary classification which we had in GAN's discriminator. Hence, the last layer of discriminator has $c$ neurons with softmax activation functions. Let $D_k(\b{x})$ denote the $k$-th logit, i.e., softmax output. The conditional probabilities, for the categories, are modeled as follows based on the logits of discriminator:
\begin{align}
p(y = k\, |\, \b{x}) = \frac{\exp(D_k(\b{x}))}{\sum_{k=1}^c \exp(D_k(\b{x}))},\,\, \forall k = \{1, \dots, c\}.
\end{align}
Note that the dataset is unlabeled (unsupervised) and the categories are just made by our model in the logits of discriminator. 
The discriminator wants to be certain about classification of real data into the $c$ categories; hence, it should minimize the entropy $H$ of conditional probabilities of real data which is:
\begin{align}
&\mathbb{E}_{\b{x} \sim p_\text{data}(\b{x})}[H(p(y = k\, |\, \b{x}))] \overset{(a)}{\approx} \frac{1}{n} \sum_{i=1}^n H(p(y = k\, |\, \b{x}_i)) \nonumber\\
&\overset{(b)}{=} \frac{1}{n} \sum_{i=1}^n \Big(\!\!-\sum_{k=1}^c p(y=k | \b{x}_i) \log \big(p(y=k | \b{x}_i)\big) \Big),
\end{align}
where $n$ is number of real data points, $(a)$ is because of the Monte Carlo approximation \cite{ghojogh2020sampling}, and $(b)$ is because of the definition of entropy. 
We draw $n$ noise samples, $\b{z} \sim p_z(\b{z})$, and feed to generator to generate data points $G(\b{z})$.
The discriminator wants to be uncertain about classification of generated (fake) data into the $c$ categories; hence, it should maximize its corresponding entropy:
\begin{align}
&\mathbb{E}_{\b{z} \sim p_z(\b{z})}[H(p(y = k\, |\, G(\b{z})))] \nonumber\\
&~~~~~~~~~~~~~~~~~~~~~ \approx \frac{1}{n} \sum_{i=1}^n H(p(y = k\, |\, G(\b{z}_i))). 
\end{align}
The generator, on the other hand, wants to minimize the above entropy to fool the discriminator. 
We also assume uniform prior $p(y)$ for categories so we want the discriminator and generator use all categories equally. For that, they should maximize the entropy of marginal category distributions:
\begin{align}
& H_\text{data}(p(y)) = H\Big(\frac{1}{n} \sum_{i=1}^n p(y\,|\,\b{x}_i)\Big), \\
& H_g(p(y)) = H\Big(\frac{1}{n} \sum_{i=1}^n p\big(y\,|\,G(\b{z}_i)\big)\Big).
\end{align}
Overall, according to above explanations, the loss functions in CatGAN are:
\begin{equation}
\begin{aligned}
& \max_D\,\,\, V(D) := H_\text{data}(p(y)) \\
&~~~~~~~~~~~~~~~~~ - \mathbb{E}_{\b{x} \sim p_\text{data}(\b{x})}[H(p(y = k\, |\, \b{x}))] \\
&~~~~~~~~~~~~~~~~~ + \mathbb{E}_{\b{z} \sim p_z(\b{z})}[H(p(y = k\, |\, G(\b{z})))], 
\end{aligned}
\end{equation}
\begin{equation}
\begin{aligned}
& \min_G\,\,\, V(D) := - H_g(p(y)) \\
&~~~~~~~~~~~~~~~~~ + \mathbb{E}_{\b{z} \sim p_z(\b{z})}[H(p(y = k\, |\, G(\b{z})))].
\end{aligned}
\end{equation}

\textbf{-- Semi-supervised CatGAN:}
The above loss function for CatGAN is used for an unsupervised case. 
We can extend CatGAN to semi-supervised cases \cite{springenberg2016unsupervised}. Suppose we have $n_\ell$ labeled data points in addition to the $n$ unlabeled data points. 
We set $c$ (i.e., the number of categories) equal to the number of classes of the labeled data.
We denote the labeled dataset by $\mathcal{X}_L := \{(\b{x}_i^\ell, \b{y}_i^\ell)\}_{i=1}^{n_\ell}$ where $\b{y}_i^\ell \in \mathbb{R}^c$ is the one-hot encoded label for the $i$-th labeled data point. 
The discriminator should maximize the cross-entropy of the labeled data to be able to discriminate the actual classes in addition to discrimination of the categories. 
This cross-entropy is:
\begin{align}
\text{CE}(\b{y}, p(y|\b{x})) := - \sum_{k=1}^c y_i \log(p(y = y_i | \b{x})),
\end{align}
where $p(y = y_i | \b{x})$ is the logit of discriminator for the labeled data input. We regularize this cross-entropy into the loss of discriminator:
\begin{equation}
\begin{aligned}
& \max_D\,\,\, V(D) := H_\text{data}(p(y)) \\
&- \mathbb{E}_{\b{x} \sim p_\text{data}(\b{x})}[H(p(y = k\, |\, \b{x}))] \\
&+ \mathbb{E}_{\b{z} \sim p_z(\b{z})}[H(p(y = k\, |\, G(\b{z})))] + \lambda \text{CE}(\b{y}, p(y|\b{x})),
\end{aligned}
\end{equation}
where $\lambda>0$ is the regularization parameter. 

\subsubsection{Generated Data as a New Class}

We can consider the generated data to be data with an additional label $(c+1)$. This idea has appeared in two independent papers which are {\citep[Section 5]{salimans2016improved}} and \cite{odena2016semi}. 
Here, we explain {\citep[Section 5]{salimans2016improved}}.
The discriminator $D$ classifies which class the data point $\b{x}$ has. This is in contrast to the discriminator in the original GAN which has a neuron with sigmoid activation function as its last layer. Here, the last layer of discriminator has $(c+1)$ neurons with softmax activation function where the $j$-th neuron outputs the probability for $\b{x}$ belonging to the $j$-th class. The optimization of discriminator is minimization of summation of two cross-entropy costs:
\begin{align}
\min_D\, \big(V_{\text{supervised}}(D) + V_{\text{unsupervised}}(D) \big), 
\end{align}
where:
\begin{equation}
\begin{aligned}
& V_{\text{supervised}}(D) := \\
&~~~~~~~~~~~~~~ -\mathbb{E}_{\b{x}, y \sim p_\text{data}(\b{x}, y)}\big[\log (p_d(y|\b{x}))\big], \forall y < c+1, \\
& V_{\text{unsupervised}}(D) := -\mathbb{E}_{\b{x} \sim p_\text{data}(\b{x})}\big[\log (1 - p_d(y|\b{x}))\big] \\
&~~~~~~~~~~~~~~ + \mathbb{E}_{\b{x} \sim p_g(\b{x})}\big[\log (p_d(y|\b{x}))\big], \text{ for } y = c+1,
\end{aligned}
\end{equation}
where $p_d(y|\b{x})$ is the output of softmax at the last layer of discriminator. 
With with cost, discriminator learns to classify the generated (fake) data points as a new class so the generator should try to fool it to not correctly classify it as the new class. The cost of generator is the same as in Eq. (\ref{equation_GAN_loss}).

We can subtract a general function from every class label. Hence, we can subtract the output, corresponding to the labels of generated data, from all labels to make the label of fake data zero, $\ell_{c+1}=0$. Hence, the softmax output of generated data becomes $\exp(\ell_{c+1} = 0) = 1$. Therefore, according to Eq. (\ref{equation_GAN_D_optimum}) and the fact that probabilities are obtained by softmax outputs (in the form of logits), we have:
\begin{align}
D(\b{x}) &\overset{(\ref{equation_GAN_D_optimum})}{=} \frac{\sum_{j=1}^c \exp(\ell_j(\b{x}))}{\big(\sum_{j=1}^{c} \exp(\ell_j(\b{x}))\big) + \exp(\ell_{c+1}(\b{x}))} \nonumber \\
&= \frac{\sum_{j=1}^c \exp(\ell_j(\b{x}))}{\sum_{j=1}^{c} \exp(\ell_j(\b{x})) + 1}. \label{equation_semi_supervised_GAN_newClass_D}
\end{align}

\subsection{MMD GAN}

MMD GAN \cite{li2017mmd} combines the ideas of moment matching networks \cite{li2015generative} and GAN \cite{goodfellow2014generative} by using adversarial learning in Maximum Mean Discrepancy (MMD). 
MMD \cite{gretton2006kernel} is a measure of divergence of two distributions and it uses distance in the Reproducing Kernel Hilbert Space (RKHS) to measure the difference of moments of two distributions \cite{ghojogh2021reproducing}. The MMD between two distributions $p(\b{x})$ and $q(\b{x})$ is:
\begin{align*}
&M_k(p,q) := \mathbb{E}_{\b{x}_i, \b{x}_j \sim p(\b{x})}[k(\b{x}_i, \b{x}_j)] \nonumber\\
&+ \mathbb{E}_{\b{x}_i, \b{x}_j \sim q(\b{x})}[k(\b{x}_i, \b{x}_j)] - 2\mathbb{E}_{\b{x}_i \sim p(\b{x}), \b{x}_j \sim q(\b{x})}[k(\b{x}_i, \b{x}_j)],
\end{align*}
where $k(.,.)$ is a kernel function such as the Gaussian kernel. 
If $p(\b{x}) = p_\text{data}(\b{x})$ and $q(\b{x}) = p_q(\b{x})$ are the distributions of real and generated data, respectively, we want to minimize this MMD so that the generated data distribution becomes similar to the real data distribution. 

\begin{figure*}[!t]
\centering
\includegraphics[width=5.3in]{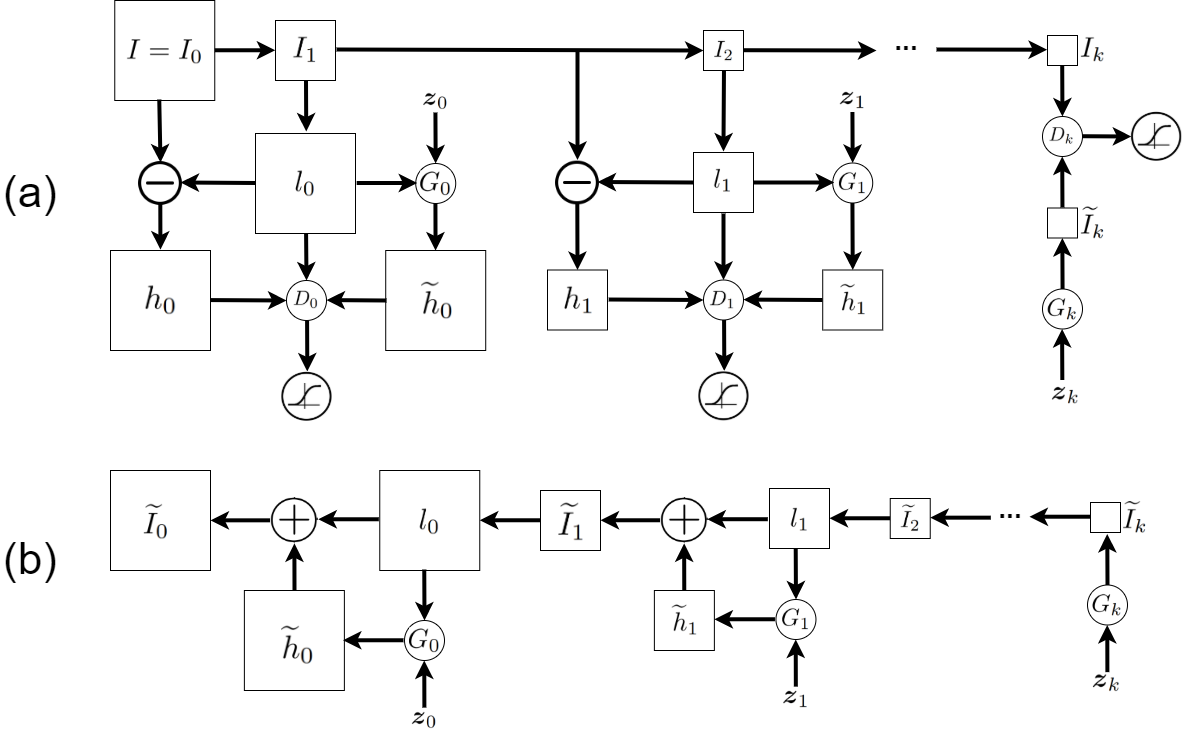}
\caption{The structure of LapGAN for (a) training and (b) test, i.e., sampling.}
\label{figure_LapGAN}
\end{figure*}

We can find the best kernel, giving the largest MMD for the worst-case scenario, from a set of valid kernel functions $\mathcal{K}$:
\begin{align*}
\min_G \max_{k \in \mathcal{K}}\,\, M_{k}(p_\text{data}(\b{x}),p_g(\b{x})).
\end{align*}
However, this optimization is difficult. 
In MMD GAN, rather than using a fixed kernel such as the Gaussian kernel, we train the kernel function by adversarial learning. We learn a function $D(.)$ to define the kernel function as:
\begin{align*}
k_D(\b{x}_i, \b{x}_j) = \exp(-\|D(\b{x}_i) - D(\b{x}_j)\|^2).
\end{align*}
We use an autoencoder for $D(.)$ with $D_e(.)$ and $D_d(.)$ as encoder and decoder, respectively. This autoencoder plays the role of discriminator in adversarial learning. The generator is denoted by $G(.)$. 
This autoencoder should reconstruct both real data, $\b{x} \in \mathcal{X}$, and generated data from latent noise, $\b{x} \in G(\b{z})$. 
The loss function of MMD GAN is:
\begin{equation}\label{equation_MMD_GAN_loss}
\begin{aligned}
&\min_{G} \max_D\,\, M_{k_D}(p_\text{data}(\b{x}),p_g(\b{x})) \\
&~~~~~~~~~~~~~~~~~~ -\lambda \mathbb{E}_{\b{x} \in \mathcal{X} \cup G(\b{z})}\big[\|\b{x} - D_d(D_c(\b{x}))\|_2^2\big], 
\end{aligned}
\end{equation}
where $\lambda>0$ is the regularization parameter. 
Both terms depend on the autoencoder $D$ while the first term depends on the generator $G$. 
Some theoretical analysis of MMD GAN can be found in \cite{mroueh2021convergence}.

\subsection{Additive GANs}

In the following, we introduce the additive GAN models which have a hierarchical or additive approach.  

\subsubsection{Laplacian GAN (LapGAN)}\label{section_LapGAN}

Laplacian GAN (LapGAN) \cite{denton2015deep} was one of the first extensions of GAN. It generates higher resolution images compared to GAN and conditional GAN. 
Inspired by the Laplacian pyramid for image \cite{burt1983laplacian}, LapGAN uses a Laplacian pyramid.
The structure of LapGAN for training is illustrated in Fig. \ref{figure_LapGAN}-a.
Let the pyramid have $k$ levels. We start with the image itself at level zero, i.e., $I = I_0$. We downsample the image to $I_1$ by a factor of two, i.e., we halve the rows and columns of image. Then, we upsample $I_1$ to $l_0$ by a factor of two, where $l_0$ is the low-pass (low-resolution) version of $I_0$. We use a conditional GAN \cite{mirza2014conditional} (see Section \ref{section_conditional_GAN}), denoted by $G_0$, which gets the noise $\b{z}_0$ as its input noise and the low-pass $l_0$ as its conditional input. The generator generates $\widetilde{h}_0$. Let $h_0 := I_0 - l_0$. We input $h_0$, $\widetilde{h}_0$, and $l_0$ to a discriminator $D_0$ whose last layer is a neuron with sigmoid activation function. The discriminator judges whether the image at this level is a real or fake (generated). This procedure is repeated for other levels until the level $(k-1)$. In each of these levels, a conditional GAN is used. In the last level $k$, a GAN (not conditional) is used which gets the noise $\b{z}_k$ as input and generates $\widetilde{I}_k$. This $\widetilde{I}_k$ and the downsampled $I_k$ are input to a discriminator $D_k$ which judges the image at that level. 

The test or sampling phase of the LapGAN is depicted in Fig. \ref{figure_LapGAN}-b. Like the training phase, all levels except the last level $k$ have conditional GANs while the last level has a GAN. At the $j$-th level, the noise $\b{z}_j$ and the low-pass image $l_j$ are fed to generator $G_j$ as its input and conditional input, respectively. The generator generates $\widetilde{h}_j$. The generated image at the $j$-th level is obtained as $\widetilde{I}_j := \widetilde{h}_j + l_j$. The generated image at the level zero, i.e. $\widetilde{I}_0$, is the generated image by the LapGAN. 


\subsubsection{Progressive GAN}

Progressive GAN \cite{karras2018progressive} starts with shallow networks for generator $G$ and discriminator $D$ and increases new layers progressively to the networks. Initially, a small convolutional layer with low spatial resolution exists in $G$ and $D$. This generates a low-resolution image. During training of GAN, we gradually add convolution layers with higher spatial resolutions to $G$ and $D$ so higher resolution images are generated. Training GAN and adding layers occur simultaneously. 

\subsection{Triple GAN}\label{section_triple_GAN}

Triple GAN \cite{li2017triple} has a discriminator $D$, a classifier $C$, and a generator $G$. In terms of having a classifier, it is similar to MGAN \cite{hoang2018mgan} (see Section \ref{section_MGAN}). 
The generator models conditional distribution of data on the label, $p_g(\b{x} | y)$, and the classifier models the opposite conditional distribution, i.e., $p_c(y | \b{x})$. The discriminator judges whether the data-label pair $(\b{x}, y)$ is real or generated (fake). 
The classifier predicts class label $y$ for the real or generated data $\b{x}$. 
Let $p_\text{data}(\b{x},y)$ denote the joint distribution of real data and labels. 
The joint distributions for data-labels in generator and classifier are $p_g(\b{x},y) = p_g(\b{x}|y) p(y)$ and $p_c(\b{x},y) = p_c(\b{x}|y) p(y)$, respectively, where $p(y)$ is the marginal distribution of labels. 
The generator gets noise $\b{z} \sim p_z(\b{z})$ and label $y$ as input and generates a data point $\b{x} = G(\b{z}, y)$, where $(G(\b{z}, y), y) \sim p_g(\b{x},y)$.
Triple GAN optimizes the loss function for a three-player game:
\begin{equation}\label{equation_Triple_GAN_loss}
\begin{aligned}
&\min_{G,C} \max_D\,\, V(D,C,G) := \mathbb{E}_{(\b{x},y) \sim p_\text{data}(\b{x},y)}\big[\log(D(\b{x},y))\big] \\
&~~~~ +\alpha \mathbb{E}_{(\b{x},y) \sim p_c(\b{x},y)}\big[\log(1 - D(\b{x},y))\big] \\
&~~~~ +(1-\alpha) \mathbb{E}_{\b{z} \sim p_z(\b{z}), y \sim p(y)}\Big[\log\!\Big(1 - D\big(G(\b{z},y), y\big)\Big)\Big] \\
&~~~~ +\mathbb{E}_{(\b{x}, y) \sim p_\text{data}(\b{x}, y)} \big[\!-\! \log (p_c(y|\b{x}))\big], 
\end{aligned}
\end{equation}
where $\alpha \in (0,1)$ is the regularization parameter ($\alpha=0.5$ is recommended). 
The last term of loss, in which $p_c(y|\b{x})$ is the predicted label by classifier, models the KL-divergence between $p_c(\b{x},y)$ and $p_\text{data}(\b{x},y)$. 
The discriminator, classifier, and generator get stronger gradually by alternating optimization \cite{ghojogh2021kkt}.

\begin{theorem}[{\citep[Lemma 3.1 and Theorem 3.3]{li2017triple}}]
The optimal discriminator of triple GAN is:
\begin{align}\label{equation_triple_GAN_optimal_D}
D^*(\b{x}, y) = \frac{p(\b{x}, y)}{p(\b{x}, y) + (1-\alpha) p_{g}(\b{x}, y) + \alpha p_{c}(\b{x}, y)}.
\end{align}
After convergence (i.e., Nash equilibrium) of triple GAN, we have:
\begin{equation}
\begin{aligned}
& p_{g^*}(\b{x}, y) = p_{c^*}(\b{x}, y) = p(\b{x}, y) \overset{(\ref{equation_triple_GAN_optimal_D})}{\implies} D^*(\b{x}, y) = 0.5.
\end{aligned}
\end{equation}
\end{theorem}

\subsection{Latent Adversarial Generator (LAG)}

Latent Adversarial Generator (LAG) \cite{berthelot2020creating} can generate high-resolution images by taking a corresponding low-resolution image as an input cue. 
In terms of getting a cue, it can be related to the conditional GAN (see Section \ref{section_conditional_GAN}).
Let $\b{z}$, $\b{x}$, and $\widetilde{\b{x}}$ denote the noise sample, the real data point, and the cue low-resolution data point, respectively. The generator $G$ takes $\widetilde{\b{x}}$ and $\b{z}$ as input and generates the high-resolution image $G(\widetilde{\b{x}}, \b{z})$. The discriminator $D$ has two parts. First, by a projection operator $\Pi$, it projects data onto a low-dimensional space, named the perceptual latent space. The operator $\Pi$ is a nonlinear neural network and gets the high and low dimensional data points as input. Then, by some other layers of network, denoted by the mapping $F(.)$, the projected data onto the perceptual latent space is mapped to a scalar after the sigmoid activation function.
Hence, the discriminator is $D(\b{x}) = F(\Pi(\b{x}, \widetilde{\b{x}}))$. 

We want to have $\Pi(G(\widetilde{\b{x}}, \b{z} = \b{0}), \widetilde{\b{x}})$ be similar to $\Pi(\b{x}, \widetilde{\b{x}})$ so we use a regularization term for it.
LAG uses WGAN (see Section \ref{section_WGAN}) whose loss is regularized. We regularize Eq. (\ref{equation_WGAN_loss}) as:
\begin{equation}\label{equation_LAG_loss}
\begin{aligned}
\min_G &\max_{\|D\|_L \leq 1} D(\b{x}, \widetilde{\b{x}}) - D(G(\widetilde{\b{x}}, \b{z}), \widetilde{\b{x}}) \\
&~~~~~~~~~~~~~ -\lambda_1 ( \nabla_{\widehat{\b{x}}} D(\widehat{\b{x}})\|_2 - 1 )^2 \\
& +\lambda_2 \|\Pi(G(\widetilde{\b{x}}, \b{z} = \b{0}), \widetilde{\b{x}}) - \Pi(\b{x}, \widetilde{\b{x}})\|_F^2,
\end{aligned}
\end{equation}
where $\lambda_1, \lambda_2 >0$ are the regularization parameters, $\|.\|_F$ is the Frobenius norm, and $\widetilde{\b{x}}$ is defined in Eq. (\ref{equation_WGAN_x_hat}).

\subsection{Ensembles of GAN Models}

In the following, we introduce some GAN models which have an ensemble of generators and/or discriminators. 
Some of them were already introduced, such as MGAN (see Section \ref{section_MGAN}) and D2GAN (see Section \ref{section_D2GAN}). Here, we explain other ensemble GAN methods. 

\subsubsection{Generative Multi-Adversarial Network (GMAN)}

Generative Multi-Adversarial Network (GMAN) \cite{durugkar2017generative} accelerates training of GAN by using several discriminators. Assume we have $n_d$ discriminators. 
The loss function of GMAN is:
\begin{align}
&\max_{D_i}\, V(D_i, G), \quad \forall i \in \{1, \dots, n_d\}, \\
&\min_G\, F\big(V(D_1, G), \dots, V(D_{n_d}, G)\big),
\end{align}
where every $V(D_i, G)$ is defined in Eq. (\ref{equation_GAN_loss}) and the function $F(.)$ can be an aggregating function such as $F(.) = \max(.)$ or $F(.) = \text{mean}(.)$. If $F(.)$ is the maximum function, generator is trained using the best discriminator at every iteration of the alternating optimization. If $F(.)$ is the mean function, an average effect of all discriminators are used for training the generator. 

\begin{figure*}[!t]
\centering
\includegraphics[width=5in]{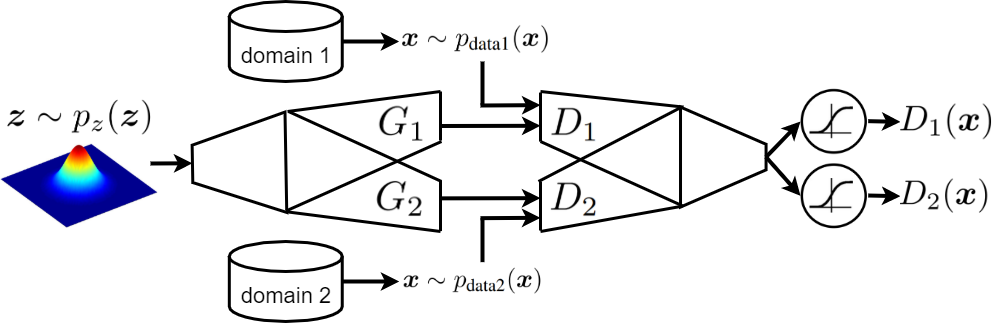}
\caption{The structure of CoGAN.}
\label{figure_CoGAN}
\end{figure*}

\subsubsection{AdaGAN: Boosting GANs}

Boosting refers to using weak models additively where every next model gives more weight to the points which were not correctly classified/regressed by the previous model \cite{ghojogh2019theory}. 
One of the most well-known boosting methods for classification and regression is AdaBoost \cite{freund1997decision}. 
AdaGAN \cite{tolstikhin2017adagan} is boosting the GAN models for generation of data points. 
Let $n$ be the number of data points.
We start with the first GAN where the weights of points are all $1/n$. 
Let the generator of the $j$-th GAN be denoted by $G_j$. We have one discriminator $D$ only as the classifier whose scalar output after sigmoid activation function is $D(\b{x})$.
For the $j$-th GAN model, we use a discriminator $D$ to discriminate between the true data and the generated data $G_{j-1}(\b{z})$ where $\b{z}$ is the latent noise. 
The weights of points are updated as:
\begin{align}\label{equation_AdaGAN_weights}
w_{i,j} := \frac{1}{n \beta_j} \big[\lambda - (1 - \beta_j) h(D(\b{x}_i))\big]_+,
\end{align}
where $w_{i,j}$ is the weight of $\b{x}_i$ for the $j$-th GAN, $[.]_+ := \max(.,0)$, $\beta_j := 1/j$ (or a fixed number in range $[0,1]$), $h(D(\b{x})) := (1 - D(\b{x})) / D(\b{x})$, and $\lambda$ is obtained by iteratively updating:
\begin{align*}
\lambda := \frac{\beta_j}{\sum_{i=1}^k (1/n)} \Big(1 + \frac{1-\beta_j}{\beta_j} \sum_{i=1}^k (1/n) h\big(D(\b{x}_i)\big)\Big),
\end{align*}
in which $k$ is the iteration of iterative updating. 
The generator of $j$-th weak GAN, $G'_t$, is trained by the weighted data points using the updated weights in Eq. (\ref{equation_AdaGAN_weights}). Finally, the $j$-th GAN is computed to be the linear combination of $G'_t$ and the previous GAN:
\begin{align*}
G_j := (1 - \beta_j) G_{j-1} + \beta_j G'_t.
\end{align*}
The proofs for the above formulas can be found in \cite{tolstikhin2017adagan}.

\subsubsection{Boosted Generative Model (BGM)}

Another similar method for boosting GAN models is the Boosted Generative Model (BGM) \cite{grover2018boosted}. We briefly introduce its idea here. 
Again, it starts with equal weights, all $1/n$, for the points. It trains the first generative model $G_1$. For the $j$-th GAN, it uses the lower bound of the f-divergence in Eq. (\ref{equation_f_divergence_lower_bound}) to estimate the next generative model based on the previous model. The formulation is inspired by the AdaBoost \cite{freund1997decision}. 

\subsection{Coupled GAN (CoGAN)}


Coupled GAN (CoGAN) \cite{liu2016coupled} is a generative model for several domains, where several data points are generated each of which has a different domain but the data points are related. For example, one domain can be image and another domain can be text where an image and a related caption can be generated. Another example is generation of two related images but from different domains, such as facial and nature images. If the tuples of corresponding data points are available, CoGAN can learn to generate corresponding and related images from different domains; otherwise, it can generate not-necessarily-related data points from the domains. 

Assume we have two domains. In this case, CoGAN has two coupled GAN structures as illustrated in Fig. \ref{figure_CoGAN}. Let $G_1$/$D_1$ and $G_2$/$D_2$ denote the generators/discriminators of the first and second GAN structures, respectively. 
In a generator, the first and last layers of network extract high-level and low-level features, respectively \cite{liu2016coupled}. Conversely, in a discriminator, the first and last layers of network extract low-level and high-level features, respectively \cite{krizhevsky2012imagenet}.
We want the GAN structures to share their high-level features but their low-level features should differ for capturing each domain's characteristics. Therefore, as shown in Fig. \ref{figure_CoGAN}, the first layers of generators and the last layers of discriminators are shared. 
Let the datasets of the first and second domains be denoted by $p_\text{data1}(\b{x})$ and $p_\text{data2}(\b{x})$, respectively. 
The loss function of CoGAN is:
\begin{equation}\label{equation_coGAN_loss}
\begin{aligned}
\min_{G_1,G_2} \max_{D_1, D_2}\,\,\,\, &V(D_1, D_2, G_1, G_2) := \\
&\mathbb{E}_{\b{x} \sim p_\text{data1}(\b{x})}\Big[\log\!\big(D(\b{x})\big)\Big] \\
&+ \mathbb{E}_{\b{z} \sim p_z(\b{z})}\Big[\log\!\Big(1 - D_1\big(G_1(\b{z})\big)\Big)\Big], \\
&+\mathbb{E}_{\b{x} \sim p_\text{data2}(\b{x})}\Big[\log\!\big(D(\b{x})\big)\Big] \\
&+ \mathbb{E}_{\b{z} \sim p_z(\b{z})}\Big[\log\!\Big(1 - D_2\big(G_2(\b{z})\big)\Big)\Big],
\end{aligned}
\end{equation}
subject to the fact that some layers of the generators and some layers of discriminators are shared, as shown in Fig. \ref{figure_CoGAN}.
Note that, although the paper \cite{liu2016coupled} has focused on coupling two GAN structures, the CoGAN can be easily extended to any number of structures and thus any number of domains. 

\subsection{Inverse GAN Models}\label{section_inverse_GAN_models}

We can invert generation of data points in GAN. This refers to generating a latent noise sample $\b{z}$ from some data point $\b{x}$. This latent noise is corresponding to the point $\b{x}$ in the sense that if it is fed to the generator, $\b{x}$ is generated. Some existing methods for inverse in GAN are adversarial autoencoder, BiGAN, ALI, and inverse technique. 
The adversarial autoencoder will be introduced later in Section \ref{section_AAE}. The other methods are explained in the following. 

\subsubsection{Bidirectional GAN (BiGAN)}

In GAN, the generator gets a latent noise $\b{z}$ and generates data point $\b{x}$. However, the inverse of this process, i.e. outputting a latent variable from the data point $\b{x}$, does not exist in GAN. 
Bidirectional GAN (BiGAN) \cite{donahue2017adversarial} is a version of GAN which also includes this inverse. Its structure is depicted in Fig. \ref{figure_BiGAN}.
In BiGAN, the generator $G$ gets the noise $\b{z}$ as input and generates $G(\b{z})$. The encoder $E$, as the inverse of $G$, gets $\b{x}$ as input and outputs $E(\b{x})$. 
Recall that the discriminator of GAN gets the data $\b{x}$ and the generated data $G(\b{z})$ as input (see Fig. \ref{figure_GAN}). 
However, the discriminator of BiGAN gets all $G(\b{z})$, $\b{z}$, $E(\b{x})$, and $\b{x}$ as input and judges whether the generated data $G(\b{z})$ is real or generated (fake). 
It assigns label one to each pair $(\b{x}, E(\b{x}))$ and label zero to each pair $(\b{z}, G(\b{z}))$. 
The loss function of BiGAN is:
\begin{equation}\label{equation_BiGAN_loss}
\begin{aligned}
\min_{G,E} &\max_D\,\,\,\, V(D,G,E) := \\
&\mathbb{E}_{\b{x} \sim p_\text{data}(\b{x})}\Big[\mathbb{E}_{\b{z} \sim p_E(.|\b{x})}\big[\log(D(\b{x},\b{z}))\big]\Big] \\
&+\mathbb{E}_{\b{z} \sim p_z(\b{z})}\Big[\mathbb{E}_{\b{x} \sim p_G(.|\b{z})}\big[\log(1 - D(\b{x},\b{z}))\big]\Big] \\
&=\mathbb{E}_{\b{x} \sim p_\text{data}(\b{x})}\Big[\log\!\Big(D\big(\b{x},E(\b{x})\big)\Big)\Big] \\
&+ \mathbb{E}_{\b{z} \sim p_z(\b{z})}\Big[\log\!\Big(1 - D\big(G(\b{z}), \b{z}\big)\Big)\Big].
\end{aligned}
\end{equation}
We use alternating optimization \cite{ghojogh2021kkt} by alternating between optimizing for $D$, $G$, and $E$. 

\begin{theorem}[{\citep[Theorem 2]{donahue2017adversarial}}]
After convergence (i.e., Nash equilibrium) of BiGAN, the optimal encoder $E$ and generator $G$ are inverse of each other:
\begin{align}
E^* = (G^*)^{-1}, \,\, G^*(E^*(\b{x})) = \b{x}, \,\, E^*(G^*(\b{z})) = \b{z}.
\end{align}
\end{theorem}

\begin{figure}[!t]
\centering
\includegraphics[width=3.2in]{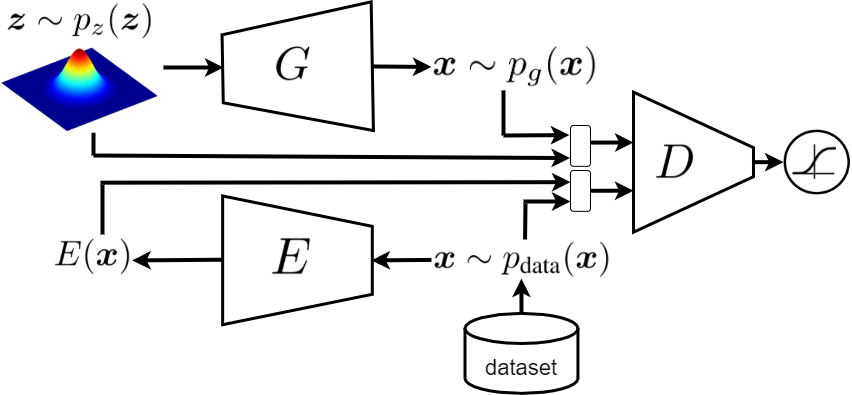}
\caption{The structure of BiGAN.}
\label{figure_BiGAN}
\end{figure}

\subsubsection{Adversarially Learned Inference (ALI)}

Adversarially Learned Inference (ALI) \cite{dumoulin2017adversarially} is one of the methods for having inverse in GAN. The generator $G$ of ALI is an autoencoder whose encoder $G_x(\b{z})$ and decoder $G_z(\b{x})$ are called the generator network and the inference network, respectively. 
The generator network $G_x(\b{z})$ maps latent noise sample $\b{z}$ to a generated data point $\widetilde{\b{x}} := G_x(\b{z})$. The inference network $G_z(\b{x})$ maps a data point $\b{x}$ to its corresponding latent noise sample $\widetilde{\b{z}} := G_z(\b{x})$. 
The discriminator $D(\b{x},\b{z})$ tries to distinguish the pairs $(\widetilde{\b{x}}, \b{z})$ and $(\b{x}, \widetilde{\b{z}})$, obtained from the generator and inference networks, respectively. 
The loss function of ALI is:
\begin{equation}
\begin{aligned}
\min_G \max_D\,\, &\mathbb{E}_{\b{x} \sim p_\text{data}(\b{x})}\big[\log(D(\b{x}, G_z(\b{x})))\big] \\
&+ \mathbb{E}_{\b{x} \sim p_\text{data}(\b{x})}\big[\log(1 - D(G_x(\b{z}), \b{z}))\big].
\end{aligned}
\end{equation}

\subsubsection{The Inversion Technique}

Another approach for having inverse in GAN is the inversion technique \cite{creswell2018inverting}. For this, after training a GAN model, we find a noise sample which results in the generated data point:
\begin{align}
\max_{\b{z}}\,\, \mathbb{E}\big[\log(G(\b{z}))\big] + \lambda \log(p_z(\b{z})),
\end{align}
where $p_z(\b{z})$ is the desired prior distribution of latent space (e.g., $\mathcal{N}(\b{0}, \b{I})$) and $\lambda>0$ is the regularization parameter. This optimization can be performed using gradient descent. 

\subsection{Self-Attention GAN (SAGAN)}

Attention mechanism \cite{vaswani2017attention} is weighting the features of data in a way that machine attends to the more important features by giving them larger weights \cite{ghojogh2020attention}. 
The weights are calculated by measuring the similarity of features with respect to each other using inner product. 
In self-attention, the similarities of features of every data point with other features of the same data point are calculated. These inner produces are implemented within the convolutional layers of network. 
Self-Attention GAN (SAGAN) \cite{zhang2019self} uses self-attention mechanism in the networks of both generator and discriminator. 
For the mathematical details of attention mechanism and SAGAN, refer to \cite{ghojogh2020attention} and \cite{zhang2019self}, respectively.

\subsection{Few-shot GAN Models}

In the following, we introduce the GAN models which learn from few number of training data points. 

\subsubsection{Transfer Learning in GAN}

Consider a GAN $(G_s, D_s)$ which is already trained on some data in a source domain. Few-shot GAN \cite{ojha2021few} can do transfer learning where the trained GAN on the source domain also generates images from another target domain. In this method, we have an adapted generator $G_{s \rightarrow t}$ which is aimed to generate data points from the target domain. As the target domain has few data points in few-shot learning, it is prone to overfitting \cite{ghojogh2019theory}. Hence, we try to preserve the pairwise similarities before and after adaptation. For this, we draw a mini-batch of $(b+1)$ noise samples $\{\b{z}_i\}_{i=1}^{b+1}$ from the latent space. We feed these to the generators $G_s$ and $G_{s \rightarrow t}$. At the $\ell$-th layer, we calculate:
\begin{align*}
& y_{s,i}^{\ell} := \text{softmax}(\text{sim}(G_s^\ell(\b{z}_i), G_s^\ell(\b{z}_j))), \\
& y_{s \rightarrow t,i}^{\ell} := \text{softmax}(\text{sim}(G_{s \rightarrow t}^\ell(\b{z}_i), G_{s \rightarrow t}^\ell(\b{z}_j))),
\end{align*}
for all $i \neq j, i,j \in \{1, \dots, b+1\}$ where sim(.) denotes the cosine similarity. We want the adapted generator to have similar distributions across layers; hence we define the loss: 
\begin{align*}
V(G_{s \rightarrow t}, G_s) := \mathbb{E}_{\b{z}_i \sim p_z(\b{z})}\Big[\sum_\ell \sum_i \text{KL}(y_{s \rightarrow t,i}^{\ell} \| y_{s,i}^{\ell})\Big],
\end{align*}
where KL(.) denotes the KL-divergence. 

We then sample $k$ number of random noises and call them the anchor points $Z_\text{anchor}$. This anchor space is a subset of the whole latent space $Z$. 
We have two discriminators which are $D_\text{image}$ for judging the whole image and $D_\text{patch}$ for judging an image patch. Let:
\begin{align*}
&V(D_\text{image}, D_\text{patch}, G_{s \rightarrow t}) := \\
&\mathbb{E}_{\b{x} \sim \mathcal{D}_t} \big[\mathbb{E}_{\b{z} \sim Z_\text{anchor}}[D_\text{image}(G_{s \rightarrow t}(\b{z})) - D_\text{image}(\b{x}_\text{image})] \\
&+\mathbb{E}_{\b{z} \sim p_z(\b{z})}[D_\text{patch}(G_{s \rightarrow t}(\b{z})) - D_\text{patch}(\b{x}_\text{patch})] \big],
\end{align*}
where $\mathcal{D}_t$ denotes the target domain. 
The overall loos function is:
\begin{equation}\label{equation_transferLearning_GAN_loss}
\begin{aligned}
\min_{G_{s \rightarrow t}} &\max_{D_\text{image}, D_\text{patch}}\,\,\,\, \\
&~~~ V(D_\text{image}, D_\text{patch}, G_{s \rightarrow t}) + \lambda V(G_{s \rightarrow t}, G_s), 
\end{aligned}
\end{equation}
where $\lambda > 0$ is the regularization parameter. 
In this loss, the first term gives freedom to the structure of patches in the image and the second term takes care of transfer learning. 

\subsubsection{GAN with Single Image (SinGAN)}

GAN with Single Image (SinGAN) \cite{shaham2019singan} learns to generate images by being trained on one image only. 
It generates images which are all related texture-wise to the training image. 
It learns the distributions of patches within the image in different scales and uses multi-scale adversarial learning. In the sens of using multiple scales in a Laplacian pyramid, it is similar to the LapGAN \cite{denton2015deep} (see Section \ref{section_LapGAN}). Assume we have $(k+1)$ levels $\{0, \dots, k\}$ in the Laplacian pyramid where the level $0$ is the image itself and the image is downsampled in other levels. At every $j$-th level, we have a GAN $(G_j, D_j)$. Training is from the $k$-th to the $0$-th level. If $\b{z}_j$ is the latent noise at level $j$, the generations are:
\begin{align*}
& \b{x}_k = G_k(\b{z}_k), \\
& \b{x}_j = G_j(\b{z}_j, \b{x}'_{j+1}), \quad \forall j < k, 
\end{align*}
where $\b{x}'_{j+1}$ is the upsampled version of the generated image $\b{x}_{j+1}$. 
The GANs are trained sequentially and the previously trained GANs are kept fixed while training the next GAN. The loss function is regularized by a reconstruction error to make the model generate better images. 

\subsection{Training Triplet Network with GAN}

A Siamese network \cite{bromley1993signature} is a network composed of multiple networks sharing their weights. If the number of networks is three, the Siamese network is a triplet network. Adversarial learning can be used for training a triplet network \cite{zieba2017training}. 
Consider triplets $(\b{x}_a, \b{x}_p, \b{x}_n)$ where $\b{x}_a$ is the anchor point, $\b{x}_p$ is the positive point having the same class as anchor, and $\b{x}_n$ is the negative point having a different class from anchor. 
For this, the loss function can be:
\begin{equation}\label{equation_triplet_GAN_loss}
\begin{aligned}
\min_{\theta}\,\,\, &-\log\big(\frac{\exp(\|\b{x}_a - \b{x}_p\|_2^2)}{\exp(\|\b{x}_a - \b{x}_p\|_2^2) + \exp(\|\b{x}_a - \b{x}_n\|_2^2)}\big) \\
&- V(D,G),
\end{aligned}
\end{equation}
where $\theta$ is the weights of network, the first term is the Neighborhood Component Analysis (NCA) \cite{goldberger2004neighbourhood}, and the second term is the adversarial loss function. Paper \cite{zieba2017training} uses Eq. (\ref{equation_semi_supervised_GAN_newClass_D}) for the discriminator $D$. 

\section{Sampling and Interpolation in GAN}

After training a GAN, we can generate new data points by sampling noise from the latent space and feeding it to the generator. 
There may exist two problems in sampling from the latent space \cite{white2016sampling}. First, we should avoid sampling from the locations in the latent space which are highly unlikely. Secondly, as the latent space is usually high dimensional, there often exist some dead-zone locations in the latent space which are not trained during the training \cite{makhzani2015adversarial}. 
In the following, we introduce some techniques for sampling and interpolation in the latent space. Note that these techniques can also be used for other generative models such as variational autoencoder \cite{kingma2014auto,ghojogh2021factor}. 

\begin{figure*}[!t]
\centering
\includegraphics[width=5.5in]{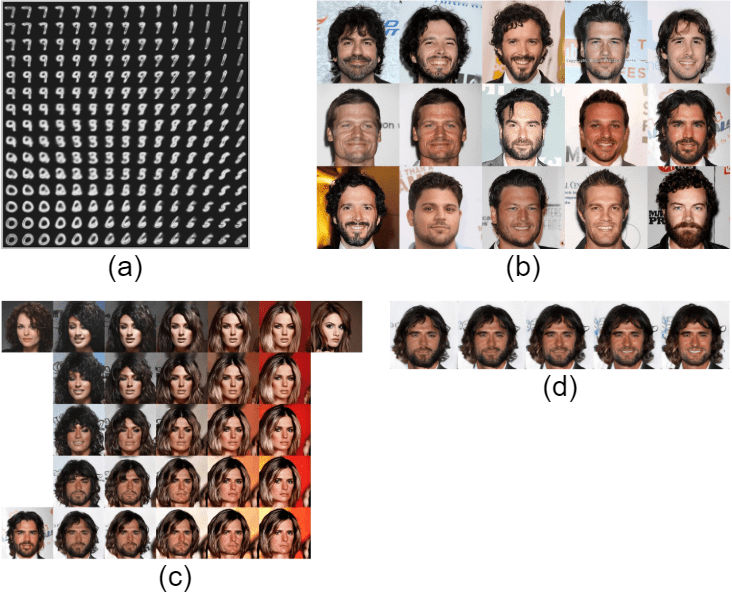}
\caption{(a) Interpolation in the latent space of VAE trained on MNIST data (image is from \url{https://blog.keras.io/building-autoencoders-in-keras.html}), (b) MINE for VAE trained on the CelebA dataset \cite{liu2015deep}, (c) J-diagram by interpolation in the latent space of a GAN trained on the CelebA dataset, and (d) traversal along the smile vector for a GAN trained on the CelebA dataset (Image for (b), (c), and (d) are from \cite{white2016sampling}).}
\label{figure_interpolation}
\end{figure*}

\subsection{Interpolation in the Latent Space}

For showing that the GAN model has not memorized the training data and the latent space is meaningful for the trained GAN, we can traverse different locations in the latent space and see what data points are generated from the sampled noises. Traversing different locations in the latent space with some step is usually called interpolation in the latent space. 
A problem with linear interpolation, which has fixed step size, is that we traverse some highly unlikely priors. This can result in strange generated data points. 
Therefore, rather than the linear interpolation, we can use spherical linear interpolation \cite{white2016sampling}, called \textit{slerp}, to traverse a path on a $p$-dimensional hypersphere in the $p$-dimensional latent space. Assume we want to sample noises between locations $\b{z}_1$ and $\b{z}_2$ in the latent space. The interpolated locations are obtained as \cite{shoemake1985animating}:
\begin{align}
\text{slerp}(\b{z}_1, \b{z}_2, \mu) := \frac{\sin ((1-\mu) \theta)}{\sin (\theta)} \b{z}_1 + \frac{\sin (\mu \theta)}{\sin (\theta)} \b{z}_2,
\end{align}
where $\mu$ is swept in range $[0,1]$ and $\theta := \cos^{-1}(\b{z}_1^\top \b{z}_2)$. 

We can have generated data points from the sampled noises by interpolation in the latent space. 
If we do interpolation across two perpendicular axes in the latent space, we can put the generations in a two dimensional table
An example for linear interpolation is shown in Fig. \ref{figure_interpolation}-a. Interpolation shows how the latent space is covering generation of various data points and what the shared features are between data points. 

\subsection{Manifold Interpolated Neighbor Embedding}

Rather than reporting the generated data points from the sampled latent vectors in interpolation, we can find the nearest neighbor of the generated point among the training data points. The nearest neighbors for the generated points are then shown in a two dimensional grid. This is called the \textit{Manifold Interpolated Neighbor Embedding} (MINE) \cite{white2016sampling}. An example grid for MINE is shown in Fig. \ref{figure_interpolation}-b. 

\subsection{Analogy and J-Diagram}

We can have vector arithmetic in the latent space (see Section \ref{section_vector_arithmetic_DCGAN}). 
The vector arithmetic shows analogy relation between vectors. Let $\b{a}$, $\b{b}$, $\b{c}$, and $\b{d}$ be the latent vectors associated with four generated data points by the generator. 
We want to find the vector $\b{d}$ to satisfy the analogy relation:
\begin{align}
\b{a}:\b{b}::\b{c}:\b{d} \implies (\b{b} - \b{a}) = (\b{d} - \b{c}).
\end{align}
In the natural language processing models, a famous analogy relation is ``man : king :: woman : queen" \cite{mikolov2013distributed}. 
\textit{J-diagram} \cite{white2016sampling} is a J-shape diagram whose top left corner, top right corner, bottom left corner, and bottom right corner are the generated images for the source vector $\b{a}$, analogy target vector $\b{b}$, analogy target vector $\b{c}$, and the result vector $\b{d}$, respectively. The other images inside the diagram are obtained by linear or slerp interpolation between these vectors. 
This diagram shows how an image is obtained from another by changing its features. 
An example J-diagram, for a GAN trained on the CelebA dataset \cite{liu2015deep}, is shown in Fig. \ref{figure_interpolation}-c.
As can be seen, moving along an axis changes some specific features of generated images. In this figure, the vertical axis takes care of gender and the horizontal axis is responsible for hair color, hair type, and facial pose. 

\subsection{Attribute Vector}

We can obtain attribute vectors for an embedding space as follows \cite{white2016sampling}. For example, a smile vector \cite{larsen2016autoencoding} can be obtained by subtracting the latent vector for a neutral face from the latent vector for the smiling face of the same person. The resulted vector can be considered as the latent vector for smiling. Other attribute vectors can be obtained similarly. An attribute vector can be used to change a neutral image to an image having that attribute. For example, we can add the smiling latent vector, denoted by $\b{z}_s \in \mathbb{R}^p$, to the latent vector of a (neutral) face, denoted by $\b{z}_n \in \mathbb{R}^p$, to obtain a new latent vector which results in generation of a smiling face of that person, after being fed to the generator. Let $\eta \in \mathbb{R}$ be the weight for smiling. The vector $\b{z}_n + \eta \b{z}_n$ is the latent vector for face with different levels of smiling. A negative $\eta$ makes a smiling face neutral. 
An example of traversal along the smile vector is shown in Fig. \ref{figure_interpolation}-d. 

\subsection{Evaluation of Generated Images}

\begin{remark}[The Inception score {\citep[Section 4]{salimans2016improved}}]
A score, named the Inception score, can be used to assess the quality of generated images by GAN models. For this, we feed the generated images $\b{x}$ to the Inception network \cite{szegedy2016rethinking} which outputs predicted labels $p(y|\b{x})$ where $y$ is the label. On one hand, we desire this conditional label distribution to have low entropy. 
On the other hand, we want the generator to generate various images; hence, the marginal $p(y) = \int p(y | \b{x}) d\b{z}$ for $\b{x} = G(\b{z})$ should be large. The Inception score combines these two as:
\begin{align}
\text{Inception score} = \exp\!\Big(\mathbb{E}_{\b{x}} \Big[\text{KL}\big(p(y|\b{x}) \| p(y)\big)\Big]\Big).
\end{align}
The higher this score, the more quality the generated image has. It has been observed that this score is very similar to human's evaluation of the generated images \cite{salimans2016improved}. 
\end{remark}

Note that there exists another method for quantitative analysis of GAN results \cite{wu2017quantitative} which is based on the annealed importance sampling \cite{neal2001annealed}. 

\section{Applications of GAN}

We already saw that GAN can be used for data generation for any data type such as image. 
In the following, we introduce some other applications of GAN. 

\subsection{Image-to-Image Translation by GAN}

There exist some methods, based on GAN, for image-to-image translation where an image is generated corresponding to an input image. The correspondence can be any relation in different applications. In the following, we introduce these methods. 

\subsubsection{PatchGAN}

PatchGAN \cite{isola2017image} uses conditional GAN \cite{mirza2014conditional} (see Section \ref{section_conditional_GAN}) with a regularized loss function. It uses $\ell_1$ norm between data and generated data for regularization because $\ell_1$ norm encourages less blurring compared to $\ell_2$ norm. The loss is:
\begin{equation}\label{equation_PatchGAN_loss}
\begin{aligned}
\min_G \max_D\,\,\,\, &V'_C(D,G) + \lambda\, \mathbb{E}_{\b{x}, \b{z}, \b{y}}\big[\|\b{x} - G(\b{z},\b{y})\|_1\big],
\end{aligned}
\end{equation}
where $\lambda > 0$ is the regularization parameter and $V'_C(D,G)$ is a slightly modified version of Eq. (\ref{equation_conditional_GAN_loss}):
\begin{equation}\label{equation_PatchGAN_loss_V_c}
\begin{aligned}
V'_C(D,G) := &\,\mathbb{E}_{\b{x}, \b{y}}\Big[\log\!\big(D(\b{x}, \b{y})\big)\Big] \\
&+ \mathbb{E}_{\b{z}, \b{y}}\Big[\log\!\Big(1 - D\big(G(\b{z}, \b{y}), \b{y}\big)\Big)\Big],
\end{aligned}
\end{equation}
in which $\b{x}$ is the data, $\b{y}$ is the label of data, and $\b{z} \sim p_z(\b{z})$ is the noise. 
The generator $G$ takes the noise $\b{z}$ and label $\b{y}$ as input and generates data denoted by $G(\b{z}, \b{y})$. The discriminator takes the data point $\b{x}$ and its label $\b{y}$ as input. It judges whether the data point $\b{x}$ is real or generated. 

For the generator $G$, PatchGAN uses skips or connections between every layer $\ell$ and layer $(L-\ell)$ where $L$ is the number of layers. This is inspired by the structure of U-Net \cite{ronneberger2015u}. 
Moreover, the $\ell_1$ norm, used in Eq. (\ref{equation_PatchGAN_loss}), takes care of the low-frequency features of generated image \cite{isola2017image}. Therefore, the discriminator should take care of the high-frequency features. For this, the discriminator $D$ classifies the image patch-wise rather than the whole image. Every patch is judged to be whether it is real or generated (fake). We average the judgments of patches to have model averaging for classifying the whole image. This patch-wise classification of an image models the image as a Markov random field because it assumes that every patch of pixels is independent of other patches. 

The PatchGAN has been used for image-to-image translation $I_1 \mapsto I_2$, i.e., translating image $I_1$ to image $I_2$.
For this, we use $\b{x} = I_2$, $\b{y} = I_1$, and noise $\b{z} \sim p_z(\b{z})$ in Eqs. (\ref{equation_PatchGAN_loss}) and (\ref{equation_PatchGAN_loss_V_c}). In other words, the image $I_1$ is used as the label in conditional GAN, while the image $I_2$ is the data point. 
The generator takes $I_1$ and noise as the input, then generates a generated $I_2$. The discriminator takes $I_1$ and $I_2$ as input and judges whether $I_2$ is a real translation of $I_1$ or a generated translation. The generator and discriminator make each other stronger gradually. For training PatchGAN, we need a dataset with pairs of $(I_1, I_2)$ images. 
Some results of PatchGAN are shown in Figs. \ref{figure_Image_to_image}-a to \ref{figure_Image_to_image}-d.

\begin{figure}[!t]
\centering
\includegraphics[width=3.2in]{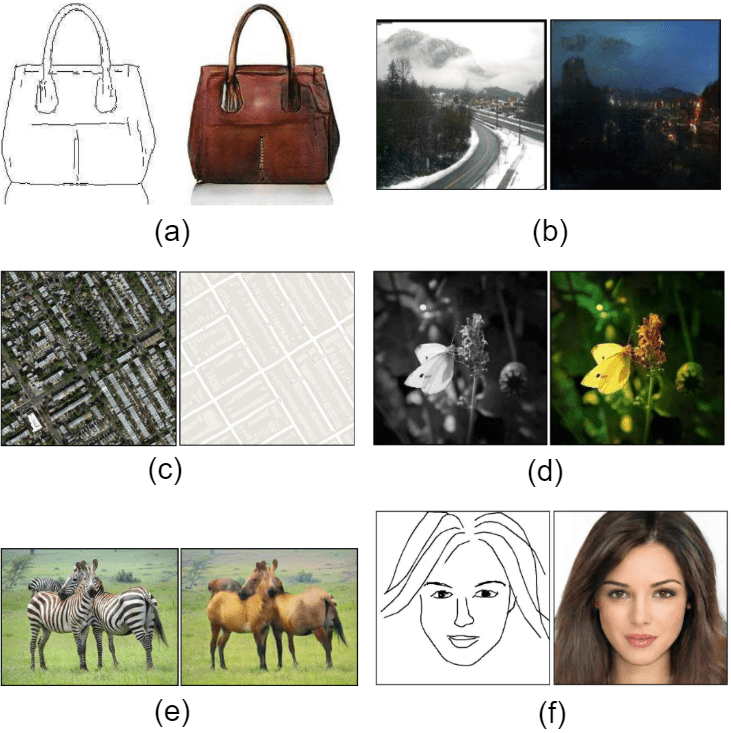}
\caption{Image-to-image translation: (a) coloring a sketch, (b) changing daylight to night darkness in image, (c) changing an aerial image to a map, (d) coloring a black-and-white image, (e) transforming zebra to horse and vice versa, and (f) generating a facial image from a facial sketch. 
Transformations in (a), (b), (c), and (d) are by PatchGAN whose credits are for \cite{isola2017image}.
Transformations in (e) and (f) are by CycleGAN (credit: \cite{zhu2017unpaired}) and DeepFaceDrawing (credit: \cite{chen2020deepfacedrawing}), respectively. 
}
\label{figure_Image_to_image}
\end{figure}

\subsubsection{CycleGAN}


CycleGAN (Cycle-Consistent Generative Adversarial Networks) \cite{zhu2017unpaired} is a method for image-to-image translation without the need to pairs of training images (in contrast to PatchGAN which needs pairs of images). Let the two domains of image translation be $X$ and $Y$. In cycleGAN, we have two generators $G:X \rightarrow Y$ and $F:Y \rightarrow X$. 
Two discriminators also exist; one is $D_X$ for judging images in $X$ and $F(Y)$ and the other is $D_Y$ for judging images in $Y$ and $G(X)$.
Hence, we have two GAN losses:
\begin{equation*}
\begin{aligned}
&V(D_Y,G,X,Y) := \,\mathbb{E}_{\b{y} \sim p_\text{data}(\b{y})}\Big[\log\!\big(D_Y(\b{y})\big)\Big] \\
&~~~~~~~~~~~~~~~~~~ + \mathbb{E}_{\b{x} \sim p_\text{data}(\b{x})}\Big[\log\!\Big(1 - D_Y\big(G(\b{x})\big)\Big)\Big], \\
&V(D_X,F,X,Y) := \,\mathbb{E}_{\b{x} \sim p_\text{data}(\b{x})}\Big[\log\!\big(D_X(\b{x})\big)\Big] \\
&~~~~~~~~~~~~~~~~~~ + \mathbb{E}_{\b{y} \sim p_\text{data}(\b{y})}\Big[\log\!\Big(1 - D_X\big(F(\b{y})\big)\Big)\Big].
\end{aligned}
\end{equation*}
We also define the following cycle consistency loss to have $F(G(\b{x})) \approx \b{x}$ and $G(F(\b{y})) \approx \b{y}$:
\begin{equation*}
\begin{aligned}
&V_\text{cyc}(G,F) := \,\mathbb{E}_{\b{x} \sim p_\text{data}(\b{x})}\Big[\|F(G(\b{x})) - \b{x}\|_1\Big] \\
&~~~~~~~~~~~~~~~~~~ + \mathbb{E}_{\b{y} \sim p_\text{data}(\b{y})}\Big[\|G(F(\b{y})) - \b{y}\|_1\Big].
\end{aligned}
\end{equation*}
The overall loss function of CycleGAN is:
\begin{equation}
\begin{aligned}
\min_{G,F} \max_{D_X, D_Y}\,\,\,\, &V(D_Y,G,X,Y) + V(D_X,F,X,Y) \\
&+ \lambda V_\text{cyc}(G,F),
\end{aligned}
\end{equation}
where $\lambda>0$ is the regularization parameter. 
A result of CycleGAN is shown in Fig. \ref{figure_Image_to_image}-e.

\subsubsection{Deep Face Drawing}

DeepFaceDrawing \cite{chen2020deepfacedrawing} generates high-quality facial images from input sketches of faces. 
For training data, automatic sketches have been created using the Canny edge detection \cite{canny1986computational}. 
DeepFaceDrawing has three modules. The first one is the component embedding module which takes different facial patches as input and learns embedding vectors for them. Then, these vectors are fed to the feature mapping module which transform the embedding vectors to 2D facial features patches. These feature patches are then fed to the image synthesis module which is a conditional GAN (see Section \ref{section_conditional_GAN}), generating facial images from the feature patches. 
A result of DeepFaceDrawing is shown in Fig. \ref{figure_Image_to_image}-f.

\subsubsection{Simulated GAN (SimGAN)}

Simulated GAN (SimGAN) \cite{shrivastava2017learning} is an unsupervised method for transforming simulated images to real-world images while preserving the annotation information of images, such as image landmarks and pose of image. 
This transformation is performed by a refiner $R(.)$.
Let $\b{y}_j$'s, $\b{x}_i$'s, and $\widetilde{\b{x}}_i$'s denote the training unlabeled real-world images, the training simulated images, and the transformation of the training simulated images to real world, i.e., $\widetilde{\b{x}}_i = R(\b{x}_i)$.
In SimGAN, we train a discriminator $D$ by minimizing the loss:
\begin{align*}
\min_D\, -\sum_i \log(D(\widetilde{\b{x}}_i)) -\sum_j \log(1 - D(\b{y}_j)),
\end{align*}
so $D$ generates labels close to one and zero for the real-world and simulated images, respectively. 
After the discriminator is trained, we use it in the loss function of refiner. The refiner acts like the generator in GAN so it tries to confuse the discriminator; hence, the loss of refiner is:
\begin{align*}
\min_R\, -\sum_i \log(1 - D(R(\b{x}_i))) + \lambda\, \|\psi(R(\b{x}_i)) - \psi(\b{x}_i)\|_1, 
\end{align*}
where $\lambda>0$ is the regularization parameter, $\psi(.)$ is a mapping from the pixel space to a feature space, and the second term tries to minimize the reconstruction error in the feature space. 

\subsubsection{Interactive GAN (iGAN)}


Interactive GAN (iGAN) \cite{zhu2016generative} allows users to edit the image interactively while the edited image remains realistic. 
In iGAN, we first project the image onto the manifold of image. The manifold of image is the manifold of latent noise in GAN. This projection is done by finding the closest latent noise which can generate the image:
\begin{align*}
\b{z}^* := \arg \min_{\b{z}} \|G(\b{z}) - \b{x}\|_2^2.
\end{align*}
In this sense, this projection is similar to the approach of inverse GAN models (see Section \ref{section_inverse_GAN_models}). 
Then, we edit the projected image, i.e., $\b{z}^*$, by different brushing and editing tools. Then, we add back the geometric and color changes to re-obtain the image, but edited this time. 

\subsection{Text-to-Image Generation}


There exist several methods for text-to-image generation where an image is generated from some descriptive caption. Some of these methods are \cite{reed2016learning,reed2016generative,zhang2017stackgan,reed2017generating,nguyen2017plug,zhang2018stackganPlusPlus}. 
Here, we introduce Stacked GAN (StackGAN) \cite{zhang2017stackgan} for text-to-image generation.

In StackGAN, we first generate embedding of texts by a pre-trained autoencoder. Let the text and the embedding of text be denoted by $\b{t}$ and $\b{\phi}_t$, respectively.
We have a stack of two stages of adversarial learning where the first stage generates a low-resolution image by drawing merely the shapes and colors. The loss function of the first stage is:
\begin{align}
& \min_D\, \mathbb{E}_{(\b{x}, \b{t}) \sim p_\text{data}(\b{x}, \b{t})}[\log(D(\b{x}, \b{\phi}_t))] \nonumber\\
&~~~~~~~ + \mathbb{E}_{\b{t} \sim p_\text{data}(\b{t}), \b{z} \sim p_z(\b{z})}[\log(1 - D(G(\b{z}), \b{\phi}_t))], \label{equation_loss_stackGAN_D}\\
& \min_G\, \mathbb{E}_{\b{t} \sim p_\text{data}(\b{t}), \b{z} \sim p_z(\b{z})}[\log(1 - D(G(\b{z}), \b{\phi}_t))] \nonumber\\
&~~~~~~~ + \lambda\, \text{KL}(q_z(\b{z}) \| p_z(\b{z})), \label{equation_loss_stackGAN_G}
\end{align}
where $q_z(\b{z})$ is the distribution of the latent code from the encoder of an autoencoder and $p_z(\b{z})$ is the prior on the latent noise. 
The next stage takes the low-resolution generated image from the first stage, denoted by $\b{s}$, as well as the text embedding as input and generates a high-resolution image. The adversarial loss of the second stage is the same as Eqs. (\ref{equation_loss_stackGAN_D}) and (\ref{equation_loss_stackGAN_G}) but it has $G(\b{s})$ rather than $G(\b{z})$ because the low-resolution image is fed to its generator. 
Some results of StackGAN are shown in Fig. \ref{figure_Text_to_image}.
An improved version of StackGAN is StackGAN++ \cite{zhang2018stackganPlusPlus}.

\begin{figure}[!t]
\centering
\includegraphics[width=2.3in]{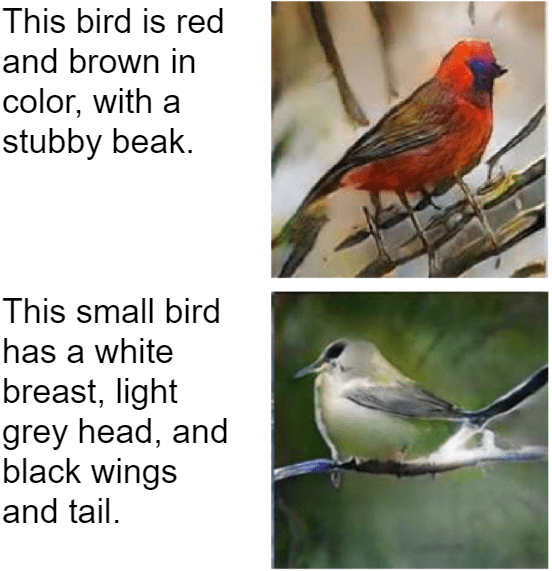}
\caption{Text-to-image translation by StackGAN. Images are from \cite{zhang2017stackgan}.
}
\label{figure_Text_to_image}
\end{figure}

\subsection{Mixing Image Characteristics}

\subsubsection{FineGAN}\label{section_FineGAN}

FineGAN \cite{singh2019finegan} is an unsupervised GAN model which disentangles the features of the generated image to background, shape, and color/texture. For this, we have three separate latent noise samples, i.e., the background code $\b{b}$, the parent code $\b{p}$, and the child code $\b{c}$, responsible for the background, shape, and color/texture, respectively. 
We assume we have $n_b$, $n_p$, and $n_c$ unknown categories (classes) for the background, shape, and color/texture, respectively, which will be learned by the FineGAN. The priors for the latent codes are categorical distribution where the probability of every class is $1/n_b$, $1/n_p$, and $1/n_c$, respectively. As every shape of some object may have several various textures in different images, we take $n_p < n_c$. 

FineGAN generates an image hierarchically. It starts with generating the background. For training data, we use a pre-trained detector to detect the background patches. 
We also use a continuous latent code $\b{z}_b$ which controls the background details within every category of background. 
The generator $G_b$ takes both $\b{b}$ and $\b{z}_b$ as input and $D_b$ is the discriminator for judging the generated background. 
We also use another discriminator $D'_b$ which is a binary classifier to two classes of foreground and background. This discriminator is pre-trained by cross entropy on the background and foreground patches. 
The loss of the background stage is:
\begin{equation}
\begin{aligned}
& \min_{G_b} \max_{D_b}\, \mathbb{E}_{\b{x}}[\log(D_b(\b{x}))] \\
&~~~~~~~~~~~~~~~~ + \mathbb{E}_{\b{b}, \b{z}_b}[\log(1 - D_b(G_b(b,\b{z}_b)))] \\
&~~~~~~~~~~~~~~~~ + \lambda\, \mathbb{E}_{\b{b}, \b{z}_b}[\log(1 - D'_b(G_b(b,\b{z}_b)))],
\end{aligned}
\end{equation}
where $\lambda>0$ is the regularization parameter. 

In the parent stage, we have two generators $G_{p,m}$ and $G_{p,f}$ generating the mask and initial texture of the object, respectively. 
A network $G_p$ takes the categorical $\b{p}$ and continuous $\b{z}_p$ as input and outputs $\b{z}'_p$ which is the input code for $G_{p,m}$ and $G_{p,f}$.
The $\b{z}_p$ controls the initial texture. 
The two generations of $G_{p,m}$ and $G_{p,f}$ are glued together to obtain the shape of object with some initial texture, which we denote by $\b{x}_p$. Then, we stitch it to the background obtained before. If the discriminator of this stage is $D_p$, the loss of this stage maximizes the mutual information between $\b{p}$ and $\b{x}_p$ as:
\begin{equation}
\begin{aligned}
& \max_{D_p, G_{p,m}, G_{p,f}}\, \mathbb{E}_{\b{p}, \b{z}_p}[\log(D_p(\b{p} | \b{x}_p))].
\end{aligned}
\end{equation}

In the child stage, we have two generators $G_{c,m}$ and $G_{c,f}$ generating the mask and color/texture of the object, respectively. 
A network $G_c$ takes $\b{c}$ and $\b{z}'_p$ as input and outputs $\b{z}'_c$ which is the input code for $G_{c,m}$ and $G_{c,f}$. 
The two generations of $G_{c,m}$ and $G_{c,f}$ are glued together to obtain the shape of object with color/texture, which we denote by $\b{x}_c$. Then, we stitch it to $\b{x}_p$, obtained before, to have the final generated image $\b{x}_f$.
The loss of the background stage is:
\begin{equation}
\begin{aligned}
& \min_{G_c} \max_{D_c}\, \mathbb{E}_{\b{x}}[\log(D_b(\b{x}))] + \mathbb{E}_{\b{c},\b{p}, \b{z}_p}[\log(1 - D_c(\b{x}_f))] \\
&~~~~~~~~~~ + \max_{D_c, G_{c,m}, G_{c,f}}\, \mathbb{E}_{\b{b}, \b{z}_b}[\log(1 - D'_c(\b{c}|\b{x}_c))],
\end{aligned}
\end{equation}
where the first two terms are for adversarial learning and the last term is for maximizing the mutual information. 
Some results of FineGAN are illustrated in Fig. \ref{figure_Mixing_image_characteristics}-a. 

\begin{figure}[!t]
\centering
\includegraphics[width=3.2in]{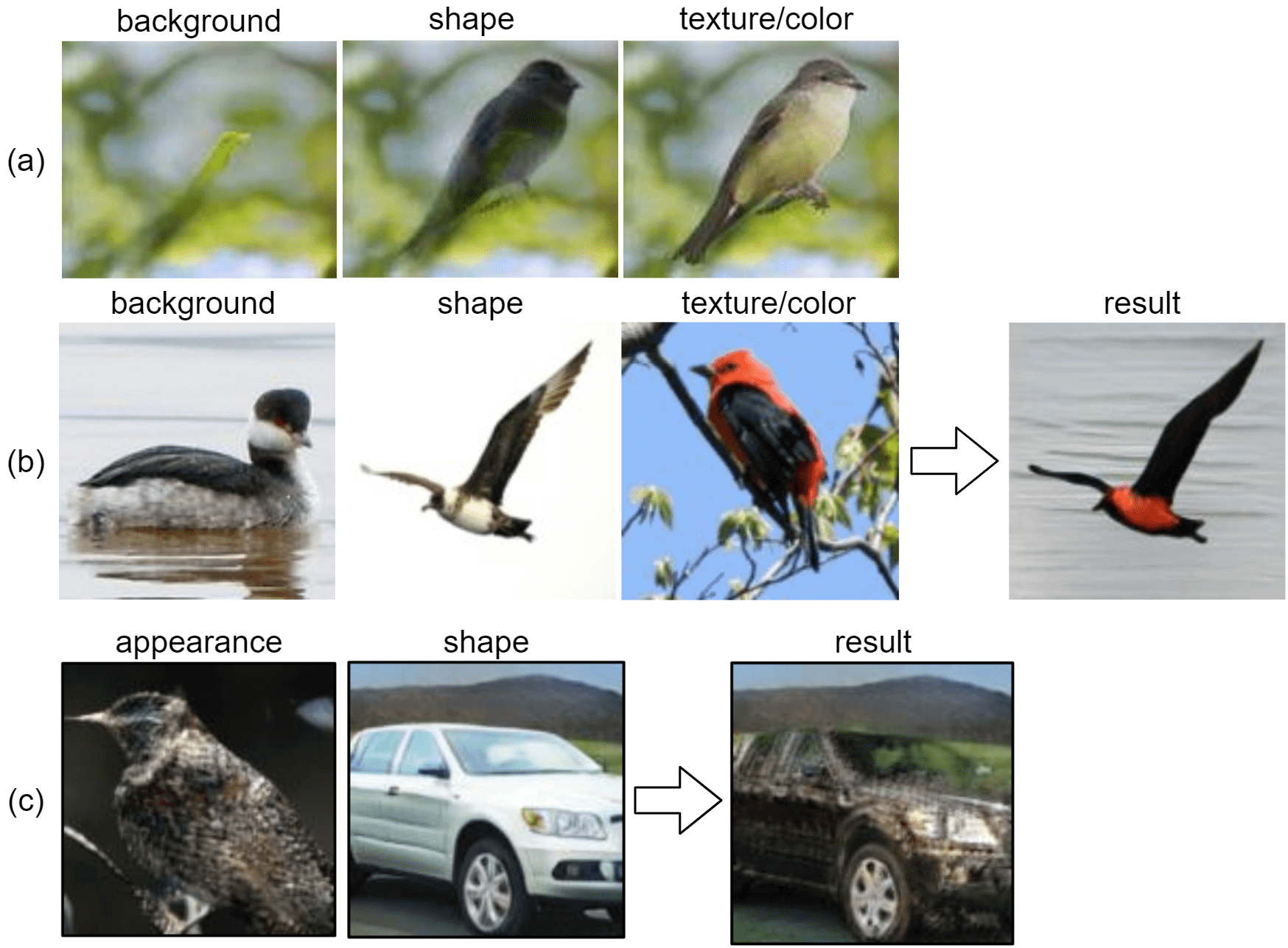}
\caption{Mixing image characteristics using GAN: (a) Generating an image with background, shape, and color characteristics by FineGAN, (b) generating an image by borrowing its characteristics from three images using MixNMatch, and (c) generating an image by borrowing its characteristics from different domains using improved MixNMatch. Images are from \cite{singh2019finegan}, \cite{li2020mixnmatch}, and \cite{ojha2021generating}, respectively.}
\label{figure_Mixing_image_characteristics}
\end{figure}

\subsubsection{MixNMatch}

MixNMatch \cite{li2020mixnmatch} is built upon FineGAN introduced in Section \ref{section_FineGAN}. It gives the user the opportunity to choose the background, shape, and color/texture from three pictures and it generates an image with the chosen characteristics. For this, we need an encoder network $E(\b{x})$ which gets three images for their background, shape, and color/texture characteristics and outputs the three latent codes $\b{b}$, $\b{p}$, and $\b{c}$. These codes are then fed to FineGAN. 

In MixNMatch, we use the idea of inverse in GAN (see Section \ref{section_inverse_GAN_models}) to have the input of the encoder and FineGAN networks. 
The input/output pair of encoder is $(\b{x} \sim p_\text{data}(\b{x}), \widetilde{\b{y}} = E(\b{x}))$ where $\widetilde{\b{y}}$ is the codes $\b{b}$, $\b{p}$, and $\b{c}$. 
The output/input pair of the FineGAN is $(\widetilde{\b{x}} = G(\b{y}), \b{y} \sim p_\text{code}(\b{y}))$ where $G(.)$ denotes the FineGAN and $p_\text{code}(\b{y})$ is the prior distribution of the latent codes $\b{b}$, $\b{p}$, and $\b{c}$. 
We have a discriminator $D$ which takes an image-code pair and judges whether it is the pair of encoder or the FineGAN. The loss of MixNMatch is:
\begin{equation}
\begin{aligned}
& \min_{G,E} \max_{D}\, \mathbb{E}_{\b{x} \sim p_\text{data}(\b{x})}\big[ \mathbb{E}_{\widetilde{\b{y}} = E(\b{x})} [\log(D(\b{x}, \widetilde{\b{y}}))] \big] \\
&~~~~~~~~~ + \mathbb{E}_{\b{y} \sim p_\text{code}(\b{y})}\big[ \mathbb{E}_{\widetilde{\b{x}} = G(\b{y})} [\log(1 - D(\widetilde{\b{x}}, \b{y}))] \big].
\end{aligned}
\end{equation}

MixNMatch lets the user choose image characteristics from the same domain and the generated image is from that domain. An example result of MixNMatch is shown in Fig. \ref{figure_Mixing_image_characteristics}-b.
Recently, an improved version of MixNMatch \cite{ojha2021generating} can take the characteristics from multiple domains and generate a new image having those characteristics. An example result of this version is also shown in Fig. \ref{figure_Mixing_image_characteristics}-c.

\subsection{Other Applications}

There are some other applications for GAN. One of the applications is inpainting some lost parts of image with GAN \cite{pathak2016context}. GAN learns to inpaint the lost part based on the available pixels in the image. 
A medical application of GAN is generating histopathology images which can give insight into cancer diagnosis from pathology whole slide images \cite{levine2020synthesis}. 
GAN has also been used for NLP \cite{li2018text,wang2019deep}, speech processing \cite{pascual2017segan,sriram2018robust}, network embedding \cite{dai2018adversarial}, logic \cite{nagisetty2021xai}, and sketch retrieval \cite{creswell2016adversarial}.  

\section{Autoencoders Based on Adversarial Learning}

Previously, variational Bayes was used in an autoencoder setting to have variational autoencoder \cite{kingma2014auto,ghojogh2021factor}. Likewise, adversarial learning can be used in an autoencoder setting \cite{makhzani2018unsupervised}. Several adversarial-based autoencoders exist which we introduce in the following. 

\subsection{Adversarial Autoencoder (AAE)}\label{section_AAE}

\subsubsection{Unsupervised AAE}

Adversarial Autoencoder (AAE) was proposed in \cite{makhzani2015adversarial}.
In contrast to variational autoencoder \cite{kingma2014auto,ghojogh2021factor} which uses KL divergence and evidence lower bound, AAE uses adversarial learning for imposing a specific distribution on the latent variable in its coding layer. 
The structure of AAE is depicted in Fig. \ref{figure_AAE_unsupervised}.
Each of the blocks $B_1$, $B_2$, and $B_3$ in this figure has several network layers with nonlinear activation functions.
AAE has an encoder (i.e., block $B_1$) and a decoder (i.e., block $B_2$). The input of encoder is a real data point $\b{x} \in \mathbb{R}^d$ and the output of decoder is the reconstructed data point $\widehat{\b{x}} \in \mathbb{R}^d$. One of the low-dimensional middle layers is the latent (or code) layer, denoted by $\b{z} \in \mathbb{R}^p$, where $p \ll d$. The encoder and decoder model conditional distributions $p(\b{z}|\b{x})$ and $p(\b{x}|\b{z})$, respectively. Let the distribution of the latent variable in the autoencoder be denoted by $q(\b{z})$. This is the posterior distribution of latent variable. The blocks $B_1$ and $B_3$ are the generator $G$ and discriminator $D$ of adversarial network, respectively. We also have a prior distribution, denoted by $p(\b{z})$, on the latent variable which is chosen by the user. This prior distribution can be a $p$-dimensional normal distribution $\mathcal{N}(\b{0}, \b{I})$, for example. 
The encoder of the autoencoder (i.e., block $B_1$) is the generator $G$ which generates the latent variable from the posterior distribution:
\begin{align}\label{equation_AAE_G}
G(\b{x}) = \b{z} \sim q(\b{z}).
\end{align}
The discriminator $D$ (i.e., block $B_3$) has a single output neuron with sigmoid activation function. It classifies the latent variable $\b{z}$ to be a real latent variable from the prior distribution $p(\b{z})$ or a generated latent variable by the encoder of autoencoder:
\begin{align}\label{equation_AAE_D}
D(\b{z}) := 
\left\{
    \begin{array}{ll}
        1 & \mbox{if } \b{z} \text{ is real, i.e., } \b{z} \sim p(\b{z}) \\
        0 & \mbox{if } \b{z} \text{ is generated, i.e., } \b{z} \sim q(\b{z}).
    \end{array}
\right.
\end{align}

As was explained, the block $B_1$ is shared between the autoencoder and the adversarial network. 
This adversarial learning makes both autoencoder and adversarial network stronger gradually because the autoencoder tries to generate the latent variable which is very similar to the real latent variable from the prior distribution. In this way, it tries to fool the discriminator. The discriminator, on the other hand, tries to become stronger not to be fooled by the encoder of autoencoder. 

\begin{figure}[!t]
\centering
\includegraphics[width=3.2in]{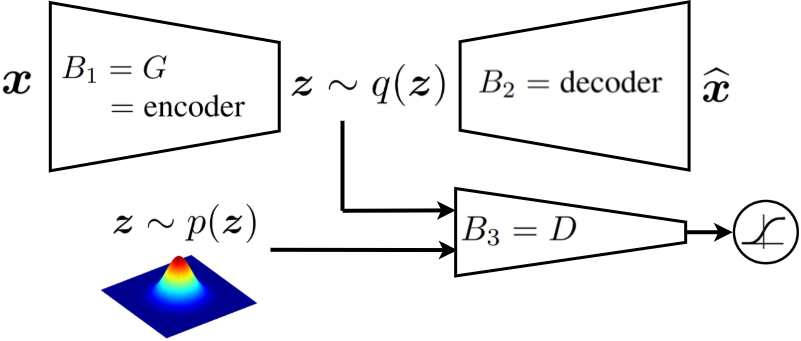}
\caption{The structure of unsupervised AAE.}
\label{figure_AAE_unsupervised}
\end{figure}

\begin{figure*}[!t]
\centering
\includegraphics[width=6in]{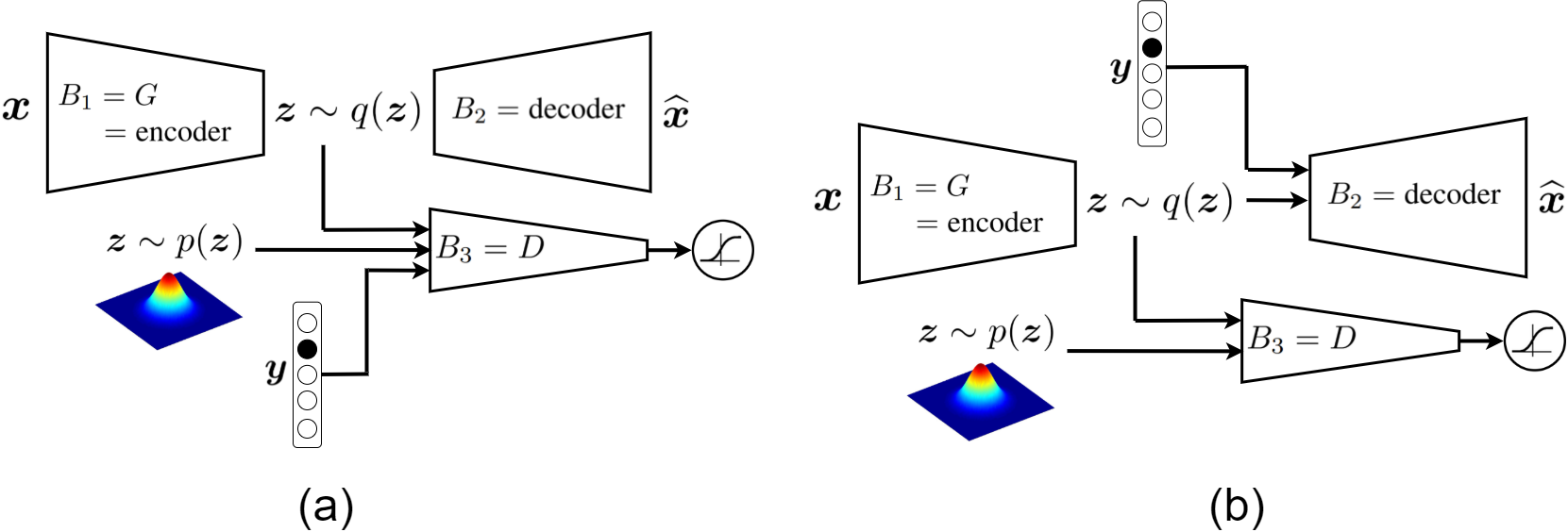}
\caption{Two structures for supervised AAE.}
\label{figure_AAE_supervised}
\end{figure*}

In AAE, we have alternating optimization \cite{ghojogh2021kkt} where reconstruction and regularization phases are repeated iteratively. 
In the reconstruction phase, the mean squared error is minimized between the data $\b{x}$ and the reconstructed data $\widehat{\b{x}}$. In the regularization phase, the discriminator and generator are updated using the GAN approach. For each of these updates, we use stochastic gradient descent \cite{ghojogh2021kkt} with backpropagation. 
Overall, the two phases are performed as:
\begin{align}
&B'_1, B_2^{(k+1)} := \arg \min_{B_1, B_2} \|\widehat{\b{x}} - \b{x}\|_2^2, \label{equation_AAE_alternating_opt_MSE} \\
&\left\{
    \begin{array}{ll}
        B_3^{(k+1)} := B_3^{(k)} - \eta^{(k)} \frac{\partial }{\partial B_3} \Big(V(B_3,B'_1)\Big), \\ 
        B_1^{(k+1)} := B'_1 - \eta^{(k)} \frac{\partial }{\partial B_1} \Big(V(B_3^{(k+1)},B_1)\Big), 
    \end{array}
\right. \label{equation_AAE_alternating_opt_B3_B1}
\end{align}
where $B_1 = G$ and $B_3 = D$ (see Fig. \ref{figure_AAE_unsupervised}). Eq. (\ref{equation_AAE_alternating_opt_MSE}) is the reconstruction phase and Eq. (\ref{equation_AAE_alternating_opt_B3_B1}) is the regularization phase. 

\subsubsection{Sampling the Latent Variable}

There are several approaches for sampling the latent variable $\b{z}$ from the coding layer of autoencoder with posterior $q(\b{z})$. In the following, we explain these approaches \cite{makhzani2015adversarial}:
\begin{itemize}
\item Deterministic approach: the latent variable is the output of encoder directly, i.e., $\b{z}_i = B_1(\b{x}_i)$. The stochasticity in $q(\b{z})$ is in the distribution of dataset, $p_\text{data}(\b{x})$.
\item Gaussian posterior: this approach is similar to what we have in variational autoencoder \cite{kingma2014auto,ghojogh2021factor}. The encoder outputs the mean $\b{\mu}$ and covariance $\b{\Sigma}$ and the latent variable is sampled from the Gaussian distribution, i.e., $\b{z}_i \sim \mathcal{N}(\b{\mu}(\b{x}_i), \b{\Sigma}(\b{x}_i))$. The stochasticity in $q(\b{z})$ is in both $p_\text{data}(\b{x})$ and the Gaussian distribution as output of encoder.
\item Universal approximator posterior: we concatenate the data point $\b{x}$ and some noise $\b{\eta}$, with a fixed distribution such as Gaussian, as input to the encoder. Hence, the latent variable is $\b{z}_i = B_1(\b{x}_i, \b{\eta}_i)$ where $\b{\eta}_i \sim \mathcal{N}(\b{0}, \b{I})$. The stochasticity in $q(\b{z})$ is in both $p_\text{data}(\b{x})$ and the noise $\b{\eta}$.
\end{itemize}

\subsubsection{Supervised AAE}

We have two variants for supervised AAE \cite{makhzani2015adversarial} where the class labels are utilized.  These two structures are illustrated in Fig. \ref{figure_AAE_supervised}. 
Let $c$ denote the number of classes. 
In the first variant, we feed the one-hot encoded label $\b{y} \in \mathbb{R}^c$ to the discriminator, i.e. block $B_3$, by concatenating it to the the latent variable $\b{z}$. In this way, the discriminator learns the label of point $\b{x}$ as well as discrimination of real and generated latent variables. This makes the generator or the encoder to generate the latent variables corresponding to the label of point for fooling the discriminator. 

In the second variant of supervised AAE, the one-hot encoded label $\b{y}$ is fed to the decoder, i.e. block $B_2$, by concatenating it to the latent variable $\b{z}$. In this way, the decoder learns to reconstruct the data point by using the label of point. This also makes the encoder, which is also the generator, generate the latent variable $\b{z}$ based on the label of point. Hence, the discriminator also gets stronger for competing the generator, in adversarial learning. 
Note that the two variants can also be combined, i.e., we can feed the one-hot encoded label can be fed to both the discriminator and the decoder. 


\subsubsection{Semi-supervised AAE}\label{section_semi_supervised_AAE}

Consider a partially labeled dataset. The labeled part of data has $c$ number of classes. AAE can be used for semi-supervised learning with partially labeled dataset. The structure for semi-supervised AAE is depicted in Fig. \ref{figure_AAE_semi_supervised}. 
This structure includes an autoencoder (blocks $B_1$ and $B_2$), an adversarial learning for generating latent variable (blocks $B_1$ and $B_3$), and an adversarial learning for generating class labels (blocks $B_1$ and $B_4$).
The encoder generates both label $\b{y} \in \mathbb{R}^c$ and latent variable $\b{z} \in \mathbb{R}^p$. The last layer of encoder for label has softmax activation function to output a $c$-dimensional vector whose entries sum to one (behaving as probability). The last layer of encoder for latent variable has linear activation function.

It has three phases which are reconstruction, regularization, and semi-supervised classification. 
In the reconstruction phase, we minimize the reconstruction error. The regularization phase trains the discriminator and generator for generating the latent variable $\b{z}$. The semi-supervised classification phase generates the one-hot encoded class label $\b{y}$ for the point $\b{x}$. If the point $\b{x}$ has a label, we use its label for training $B_1$ and $B_4$. However, if the point $\b{x}$ does not have any label, we randomly sample a label $\b{y} \in \mathbb{R}^c$ from a categorical distribution, i.e., $\b{y} \sim \text{Cat}(\b{y})$. This categorical distribution gives a one-hot encoded vector where the prior probability of every class is estimated by the proportion of class's population to the total number of labeled points. 
An iteration of the alternating optimization for semi-supervised learning is:
\begin{align}
& B'_1, B^{(k+1)}_2 := \arg \min_{B_1, B_2} \|\widehat{\b{x}} - \b{x}\|_2^2, \label{equation_semisupervised_AAE_alternating_opt_MSE} \\
&\left\{
    \begin{array}{ll}
        B_3^{(k+1)} := B_3^{(k)} - \eta^{(k)} \frac{\partial }{\partial B_3} \Big(V_z(B_3,B'_1)\Big), \\ 
        B''_1 := B'_1 - \eta^{(k)} \frac{\partial }{\partial B_1} \Big(V_z(B_3^{(k+1)},B_1)\Big), 
    \end{array}
\right. \label{equation_AAE_alternating_opt_B3_B1_prime_prime} \\
&\left\{
    \begin{array}{ll}
        B_4^{(k+1)} := B_4^{(k)} - \eta^{(k)} \frac{\partial }{\partial B_4} \Big(V_y(B_4,B''_1)\Big), \\ 
        B_1^{(k+1)} := B''_1 - \eta^{(k)} \frac{\partial }{\partial B_1} \Big(V_y(B_4^{(k+1)},B_1)\Big),
    \end{array}
\right. \label{equation_AAE_alternating_opt_B4_B1} 
\end{align}
where $V_z(D,G)$ and $V_y(D,G)$ are the loss functions defined in Eq. (\ref{equation_GAN_loss}) in which the generated variables are the latent variable $\b{z}$ and the one-hot encoded label $\b{y}$, respectively. 

\begin{figure}[!t]
\centering
\includegraphics[width=3.2in]{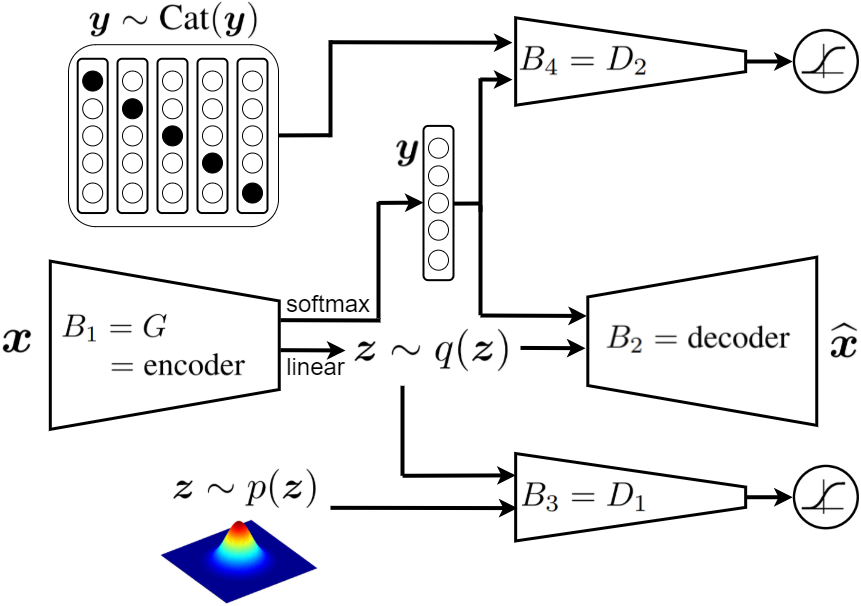}
\caption{The structure of semi-supervised AAE.}
\label{figure_AAE_semi_supervised}
\end{figure}

\subsubsection{Unsupervised Clustering with AAE}

We can use the structure of Fig. \ref{figure_AAE_semi_supervised} for clustering but rather than the classes, we assume we have $c$ number of clusters. We do not have a partially labeled part of dataset. All points are unlabeled and the cluster indices are sampled randomly by the categorical distribution. The cluster labels and the latent code are both trained in the three phases which were explained in Section \ref{section_semi_supervised_AAE}.

\subsubsection{Dimensionality Reduction with AAE}

The AAE can be used for dimensionality reduction and representation learning. The structure of AAE for this purpose is depicted in Fig. \ref{figure_AAE_dim_reduction}.
The encoder generates both label $\b{y} \in \mathbb{R}^c$ and latent variable $\b{z} \in \mathbb{R}^p$ where $p \ll d$.
Everything is similar to what we had before but a network layer $\b{W} \in \mathbb{R}^{c \times p}$ is added after the generated label by the encoder. The low-dimensional representation $\widetilde{\b{x}} \in \mathbb{R}^p$ is obtained as:
\begin{align}
\mathbb{R}^p \ni \widetilde{\b{x}} = \b{W}^\top \b{y} + \b{z},
\end{align}
where $\b{z}$ is the latent variable generated by the encoder. 
The three phases explained in Section \ref{section_semi_supervised_AAE} trains the AAE for dimensionality reduction. 

\begin{figure}[!t]
\centering
\includegraphics[width=3.3in]{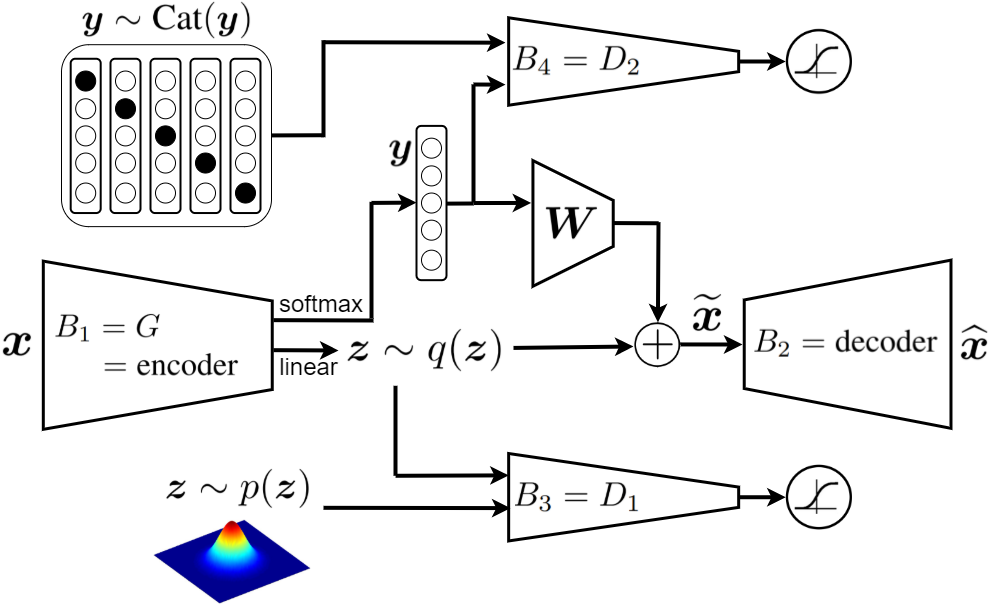}
\caption{The structure of AAE for dimensionality reduction.}
\label{figure_AAE_dim_reduction}
\end{figure}

\subsection{PixelGAN Autoencoder}


\begin{figure}[!t]
\centering
\includegraphics[width=3.3in]{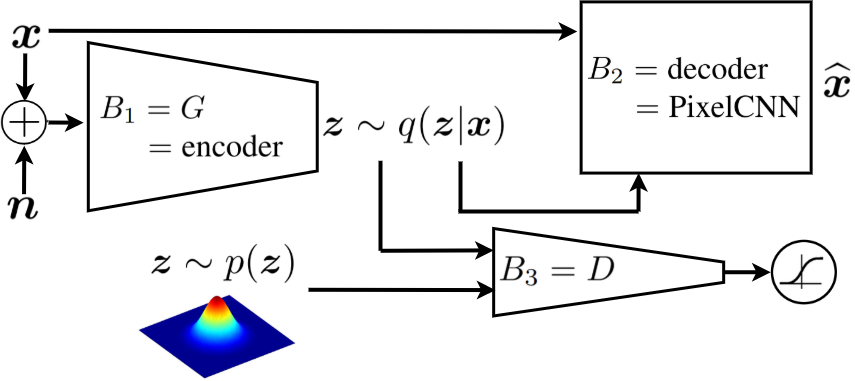}
\caption{The structure of PixelGAN.}
\label{figure_PixelGAN}
\end{figure}

In variational inference \cite{ghojogh2021factor}, the Evidence Lower Bound (ELBO) can be restated as \cite{hoffman2016elbo}:
\begin{equation}
\begin{aligned}
&\mathbb{E}_{\b{x} \sim p_\text{data}(\b{x})} [\log(p(\b{x}))] > \\
&~~~~~~~~~~~~~ -\mathbb{E}_{\b{x} \sim p_\text{data}(\b{x})} [\mathbb{E}_{q(\b{z}|\b{x})}[-\log(p(\b{x}|\b{z}))]] \\
&~~~~~~~~~~~~~ -\text{KL}(q(\b{z}) \| p(\b{z})) - \mathbb{I}(\b{z}; \b{x}),
\end{aligned}
\end{equation}
where $\mathbb{I}(.;.)$ denotes the mutual information. 
The first and second terms in this lower bound are the reconstruction error and the marginal KL divergence on the latent space. 
The PixelGAN autoencoder \cite{makhzani2017pixelgan} uses this lower bound but ignores its third term which is the mutual information because optimization of that term makes $\b{z}$ be independent of $\b{x}$. 
The reconstruction error is minimized in a reconstruction phase of training and the KL divergence part is taken care of by an adversarial learning. 

The structure of PixelGAN is shown in Fig. \ref{figure_PixelGAN}. The block $B_1$ is the encoder which gets the data point $\b{x}$ added with some noise $\b{n}$ as input and outputs the latent code $\b{z} \sim q(\b{z}|\b{x})$. The block $B_2$ is the decoder which is a PixelCNN network \cite{oord2016conditional} from which PixelGAN has borrowed its name. This decoder outputs the reconstructed data $\widehat{\b{x}}$. The generated latent code $\b{z}$ is used as the adaptive biases of layers in the PixelCNN. 
The blocks $B_1$ and $B_3$ are the generator and discriminator of adversarial learning, respectively, where we try to make the distribution of the generated latent code $\b{z}$ similar to some prior distribution $p(\b{z})$. 
In summary, blocks $B_1$ and $B_2$ are used for the reconstruction phase and blocks $B_1$ and $B_3$ are used for the adversarial learning phase. 

\subsection{Implicit Autoencoder (IAE)}


\begin{figure*}[!t]
\centering
\includegraphics[width=5in]{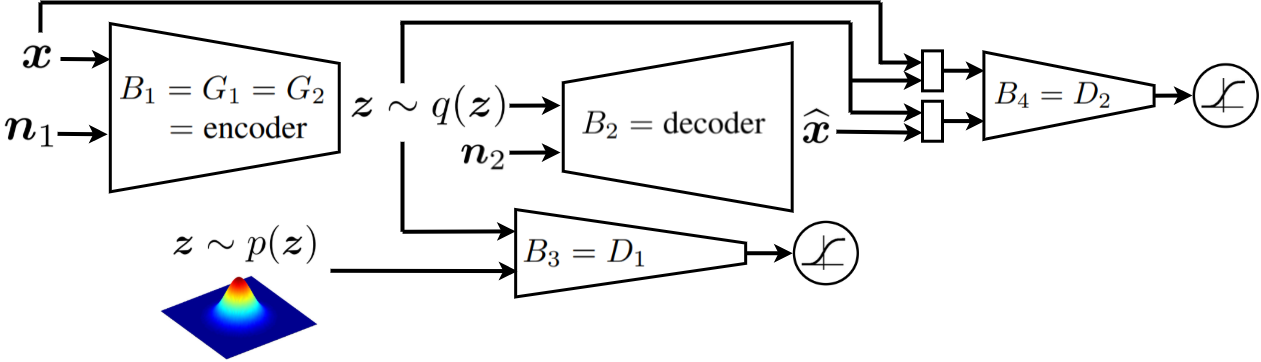}
\caption{The structure of IAE.}
\label{figure_IAE}
\end{figure*}

In variational inference \cite{ghojogh2021factor}, the Evidence Lower Bound (ELBO) can be restated as \cite{makhzani2018implicit}:
\begin{equation}\label{equation_ELBO_IAE}
\begin{aligned}
\mathbb{E}_{\b{x} \sim p_\text{data}(\b{x})} [\log(p(\b{x}))] \geq &-\text{KL}(q(\b{x}, \b{z})\|q(\widehat{\b{x}}, \b{z})) \\
&-\text{KL}(q(\b{z})\|p(\b{z})) - H_\text{data}(\b{x}),
\end{aligned}
\end{equation}
where $H_\text{data}(\b{x})$ is the entropy of data, $\widehat{\b{x}}$ is the reconstructed data, and $\b{z}$ is some latent factor. The proof is straightforward and can be found in {\citep[Appendix A]{makhzani2018implicit}}.
The first and second terms are the reconstruction and regularization terms, respectively. 
The Implicit Autoencoder (IAE) \cite{makhzani2018implicit} implements the above distributions in Eq. (\ref{equation_ELBO_IAE}), implicitly using networks. The structure of IAE is shown in Fig. \ref{figure_IAE}. The block $B_1$ is the encoder which takes data $\b{x}$ and noise $\b{n}_1$ as input and outputs the latent code $\b{z} \sim q(\b{z})$. The block $B_2$ takes the generated latent code $\b{z}$ as well as some noise $\b{n}_2$ and outputs the reconstructed data $\widehat{\b{x}}$. The blocks $B_1$ and $B_3$ are the generator $G_1$ and discriminator $D_1$ of the first adversarial learning used for making the distribution of latent code $\b{z}$ similar to some prior distribution $p(\b{z})$. 
The blocks $B_1$ and $B_4$ are the generator $G_2$ and discriminator $D_2$ of the second adversarial learning used for making the distribution of reconstructed data $\widehat{\b{x}}$ similar to data $\b{x}$. The inputs of $B_4$ are the pairs $(\b{x}, \b{z})$ and $(\widehat{\b{x}}, \b{z})$ to model the distributions $q(\b{x}, \b{z})$ and $q(\widehat{\b{x}}, \b{z})$ in Eq. (\ref{equation_ELBO_IAE}). 
In summary, three phases of training are performed which are the reconstruction phase and the two adversarial learning phases. 


\section{Conclusion}\label{section_conclusion}

This was a tutorial and survey paper on GAN, adversarial learning, adversarial autoencoder, and their variants. We covered various aspects and theories of the methods as well as applications of GAN. 




\bibliography{References}
\bibliographystyle{icml2016}

\end{document}